\documentclass[letterpaper]{article}

\usepackage[preprint,nonatbib]{neurips_2020}

\newif\ifviz
\usepackage[useregional]{datetime2}
\usepackage{amsmath}
\usepackage{amsthm}
\usepackage{amssymb}
\usepackage{algorithm, algorithmic}
\usepackage[utf8]{inputenc} 
\usepackage[T1]{fontenc}    
\usepackage{hyperref}       
\usepackage{booktabs}       
\usepackage{amsfonts}       
\usepackage{nicefrac}       
\usepackage{microtype}      
\usepackage{bm}
\usepackage{graphicx}
\usepackage{wrapfig}
\usepackage{tikz}
\usepackage{multirow}

\usepackage{mdframed}

\usepackage{xcolor}
\usepackage{color}
\usepackage{stmaryrd}
\usepackage{etoolbox}

\newenvironment{nstabbing}
  {\setlength{\topsep}{0pt}%
   \setlength{\partopsep}{0pt}%
   \tabbing}
  {\endtabbing}
\newtheorem{thm}{Theorem}
\newtheorem{theorem}[thm]{Theorem}
\newtheorem{lemma}[thm]{Lemma}
\newtheorem{corollary}[thm]{Corollary}
\newtheorem{proposition}[thm]{Proposition}
\newtheorem{definition}[thm]{Definition}

\title{Online Multitask Learning with Long-Term Memory}

\author{%
	Mark Herbster, Stephen Pasteris, Lisa Tse \\
	Department of Computer Science\\
	University College London\\
	London \\
	United Kingdom\\
	\texttt{(m.herbster|s.pasteris|l.tse)@cs.ucl.ac.uk} \\
}

\newcommand{\norm}[1]{\left\lVert#1\right\rVert}

\newcommand{\ceil}[1]{{\lceil{#1}\rceil}}

\newcommand{\sign}{\operatorname{sign}}

\newcommand{\argmin}{\operatornamewithlimits{argmin}}
\newcommand{\into}{\rightarrow}
\newcommand{\bs}[1]{\boldsymbol{#1}}

\newcommand{\con}[2]{\left[#1;#2\right]}

\newcommand{\cl}[1]{\bs{q}^{#1}}

\bmdefine{\mat}{\tilde{U}}

\newcommand{\la}{\bs{L}}

\newcommand{\one}{\bs{1}}

\newcommand{\wem}[1]{\bm{\tilde{W}}^{#1}}
\newcommand{\id}{\bs{I}}

\newcommand{\di}[1]{\eta^t}
\newcommand{\pdi}[1]{\bar{\eta}^t}

\newcommand{\q}[1]{#1}

\newcommand{\Pm}{m}
\newcommand{\m}{\q{\Pm}}

\newcommand{\mar}{\gamma}
\newcommand{\marh}{\gamma}

\newcommand{\Exp}{\mathbb{E}}

\newcommand{\nmset}{\mathbb{M}}
\newcommand{\uset}{\mathbb{U}}

\newcommand{\trace}[1]{\operatorname{tr}({#1})}
\newcommand{\trc}{\operatorname{tr}}
\newcommand{\tr}[1]{\operatorname{tr}\left(#1\right)}

\newcommand{\lr}{\eta}

\newcommand{\y}[1]{y_{#1}}

\newcommand{\dotp}[1]{{\langle{#1}\rangle}}
\newcommand{\vn}[1]{\norm{#1}}
\newcommand{\tnorm}[1]{\|{#1}\|_1}
\newcommand{\maxnorm}[1]{{\|{#1}\|_{\text{max}}}}
\newcommand{\maxnormsqr}[1]{{\|{#1}\|^2_{\text{max}}}}

\newcommand{\gm}{\Re^{m\times n}}

\newcommand{\gmd}{\Re^{m \times d}}
\newcommand{\gnd}{\Re^{n \times d}}
\newcommand{\sm}{\{-1,1\}^{m\times n}}

\newcommand{\ld}{p}
\newcommand{\bikl}{\mathbb{B}_{k,\ell}^{m,n}}
\newcommand{\biml}{\mathbb{B}_{\ld,\dis}^{\ld,T}}

\newcommand{\trans}{{\scriptscriptstyle \top}}
\newcommand{\N}{\mathbb{N}}

\newcommand{\bone}{{\bm 1}}
\newcommand{\bzero}{{\bm 0}}

\newcommand{\bw}{{\bm w}}

\newcommand{\ba}{\bm{a}}

\newcommand{\bc}{\bm{c}}
\newcommand{\be}{\bm{e}}

\newcommand{\bh}{\bm{h}}

\newcommand{\bp}{\bm{p}}
\newcommand{\bq}{\bm{q}}

\newcommand{\bu}{\bm{u}}
\newcommand{\bvv}{\bm{v}}
\newcommand{\bmv}{\bm{v}}
\newcommand{\bxx}{\bm{x}}

\newcommand{\bz}{\bm{z}}

\newcommand{\bA}{\bm{A}}

\newcommand{\bC}{\bm{C}}
\newcommand{\bD}{\bm{D}}

\newcommand{\bH}{\bm{H}}
\newcommand{\bI}{\bm{I}}
\newcommand{\bK}{\bm{K}}
\newcommand{\bLb}{\bm{L^{\circ}}}
\newcommand{\bL}{\bm{L}}
\newcommand{\bM}{\bm{M}}
\newcommand{\bN}{\bm{N}}
\newcommand{\bP}{\bm{P}}
\newcommand{\bQ}{\bm{Q}}
\newcommand{\bR}{\bm{R}}
\newcommand{\bS}{\bm{S}}
\newcommand{\bU}{\bm{U}}

\newcommand{\bW}{\bm{W}}

\newcommand{\bX}{\bm{X}}
\newcommand{\bY}{\bm{Y}}

\newcommand{\cA}{\mathcal{A}}

\newcommand{\cE}{\mathcal{E}}

\newcommand{\cG}{\mathcal{G}}
\newcommand{\cH}{\mathcal{H}}
\newcommand{\cI}{\mathcal{I}}
\newcommand{\cJ}{\mathcal{J}}

\newcommand{\cM}{\mathcal{M}}
\newcommand{\cN}{\mathcal{N}}
\newcommand{\cO}{\mathcal{O}}
\newcommand{\cOT}{\mathcal{\tilde{O}}}

\newcommand{\cT}{\mathcal{T}}

\newcommand{\cV}{\mathcal{V}}
\newcommand{\cX}{\mathcal{X}}

\newcommand{\onehalf}{\frac{1}{2}}

\newcommand{\PP}{\bP}
\newcommand{\QQ}{\bQ}
\newcommand{\nP}{\hat{\bP}} 
\newcommand{\nQ}{\hat{\bQ}} 

\newcommand{\R}{\bR} 
\newcommand{\RN}{\bM} 
\newcommand{\CN}{\bN} 
\newcommand{\RNfunc}{\mathcal{\bM^+}} 
\newcommand{\CNfunc}{\mathcal{\bN^+}} 
\newcommand{\RRN}{\mathcal{R}_{\RN}} 
\newcommand{\RCN}{\mathcal{R}_{\CN}} 
\newcommand{\IRRN}{\mathcal{R}_{\cM}} 
\newcommand{\IRCN}{\mathcal{R}_{\cN}} 
\newcommand{\RAD}{\mathcal{R}} 
\newcommand{\U}{\bU}
\newcommand{\Uset}{{\{-1,1\}^{m \times n}}}

\newcommand{\Ustar}{\U^*}
\newcommand{\Hstar}{\bH^*}

\newcommand{\RNM}{\mathcal{N}}
\newcommand{\BEM}{\mathcal{B}}
\newcommand{\SPDM}{\bS_{++}}

\bmdefine{\mat}{\tilde{U}}
\newcommand{\XT}{\bm{\tilde{X}}^t}

\newcommand{\scp}{\mathcal{\widehat{\mathcal{D}}}}
\newcommand{\upD}{\widehat{\mathcal{D}}}

\newcommand{\RKD}{\bar{\bK}}
\newcommand{\CKD}{\bar{\bP}}
\newcommand{\qD}{\mathcal{D}^{\marh}_{\RN,\CN}}
\newcommand{\qDA}{\mathcal{D}^{\marh'}_{\RN,\CN}}
\newcommand{\qDH}{\mathcal{D}^{1/\sqrt{m}}_{\RKD,\CKD}}
\newcommand{\DITF}{\mathcal{D}^{\star}_{\RN,\CN}}

\newcommand{\lzo}{\mathcal{L}_{01}}
\newcommand{\lhi}{\mathcal{L}_{\text{hi}}}

\newcommand{\swi}{k}

\newcommand{\dis}{m} 
\newcommand{\disj}{m} 
\newcommand{\nL}{s}
\newcommand{\li}{\ell}
\newcommand{\liV}{\bm{\ell}}
\newcommand{\gti}{\tau}
\newcommand{\gltF}{\bm{\sigma}}

\newcommand{\E}{\mathbb{E}}
\newcommand{\bset}{\{-1,1\}}
\newcommand{\hfin}{\cH_{\operatorname{fin}}}
\newcommand{\hK}{\cH_{K}}
\newcommand{\hKbin}{\cH^{\!(\bxx)}_{K}}
\newcommand{\gbh}{\bh^*}
\newcommand{\MW}{{\textsc{MW}}}
\newcommand{\OGD}{$\text{\textsc{OGD}}_K$}
\newcommand{\IMCSI}{{\textsc{IMCSI}}}
\newcommand{\CI}{C}
\newcommand{\CIH}{\hat{C}}

\newcommand{\KRH}{\hat{X}_K^2}
\newcommand{\PTK}{P}
\newcommand{\PTKT}{\tilde{P}}
\newcommand{\PTKTT}{\tilde{P}^{\tv,T^1,\ldots,T^{\nL}}}
\newcommand{\KPH}{{\hat{X}_P^2}}
\newcommand{\bPTK}{{\bm P}^{\gti}}
\newcommand{\Htau}{H^{\tau}}
\newcommand{\htau}{h^{\tau}}
\newcommand{\Hit}{H^{i}_t}
\newcommand{\Hits}{H^{i}_{t+1}}
\newcommand{\hit}{h^{i}_t}
\newcommand{\hits}{h^{i}_{t+1}}
\newcommand{\xtau}{x^{\gti}}
\newcommand{\xit}{x_t^i}
\newcommand{\xtup}{x^{\upsilon}}
\newcommand{\tmit}{{\smash{{}^i_t}}}
\newcommand{\tmits}{{\smash{{}^{\ \ i \ }_{t+1}}}}
\newcommand{\XTau}{\bm{\tilde{X}}^\tau}
\newcommand{\lv}{{\bm \ell}}
\newcommand{\ltau}{{\ell^{\tau}}}
\newcommand{\lups}{{\ell^{\upsilon}}}
\newcommand{\tasks}{{\nL}}
\newcommand{\tv}{{\bm \ell}}
\newcommand{\cXfin}{\cX^{\text{fin}}}
\newcommand{\yhl}{\yh^i_t}
\newcommand{\yl}{y^i_t}
\newcommand{\yit}{\yl}
\newcommand{\ytau}{y^{\tau}}
\newcommand{\yhtau}{{\hat{y}^{\tau}}}
\newcommand{\ybtau}{\bar{y}^{\tau}}

\newcommand{\yb}{\bar{y}}
\newcommand{\ybt}{\bar{y}_t}
\newcommand{\yht}{\hat{y}_t}

\newcommand{\yh}{\hat{y}}
\newcommand{\Yrv}{Y_t}
\newcommand{\Yrvtau}{Y^\tau}
\newcommand{\bawit}{{\bm w}^i_t}
\newcommand{\bawits}{{\bm w}^i_{t+1}}

\newcommand{\ldim}{\text{Ldim}}

\newcommand{\msize}{\gamma}

\newcommand{\sip}{\mu} 
\newcommand{\ww}{\theta}
\newcommand{\nsw}[1]{S_{#1}}
\newcommand{\gtu}{T}
\newcommand{\sw}{\phi}
\newcommand{\nex}{n}

\newcommand{\com}{\gbh}

\newcommand{\siz}[1]{|\dis(#1)|}

\newcommand{\ntr}{T}
\newcommand{\nta}{\nL}

\newcommand{\bal}[1]{\bs{v}^{#1}}
\newcommand{\al}[2]{v^{#1}_{#2}}
\newcommand{\simp}{\Delta_{\nex}}
\newcommand{\lset}{[0,1]^{\nex}}

\newcommand{\blo}[1]{\bs{c}^{#1}}
\newcommand{\lo}[2]{c^{#1}_{#2}}
\newcommand{\lot}[1]{C^{#1}}
\newcommand{\na}[1]{[#1]}

\newcommand{\bpe}{\bs{\pi}}
\newcommand{\pe}[1]{\pi_{#1}}
\newcommand{\baw}[1]{\bs{\alpha}^{#1}}
\newcommand{\aw}[2]{\alpha^{#1}_{#2}}

\newcommand{\bwaw}{\bs{\delta}}
\newcommand{\waw}[1]{\delta_{#1}}

\newcommand{\saw}{\beta}
\newcommand{\bnaw}{\bs{\epsilon}}
\newcommand{\naw}[1]{\epsilon_{#1}}

\newcommand{\str}{\mathcal{Z}}
\newcommand{\ltu}[1]{T^{#1}}
\newcommand{\ti}{t}
\newcommand{\trm}[2]{\lambda_{#1,#2}}
\newcommand{\trn}[1]{\lambda_{#1}}
\newcommand{\crs}{\mathcal{C}}
\newcommand{\crw}[1]{\rho(#1)}
\newcommand{\trw}[1]{\bar{\rho}(#1)}
\newcommand{\ci}{\omega}
\newcommand{\hye}{\mathcal{E}}
\newcommand{\ei}{i}
\newcommand{\ran}{R_{\com}}
\newcommand{\rcr}[1]{\xi_{#1}}
\newcommand{\idf}[1]{\left[#1\right]}
\newcommand{\ufu}{u}
\newcommand{\upd}[1]{u(#1)}
\newcommand{\was}[1]{\mathcal{W}_{#1}}
\newcommand{\liv}[1]{\ell^{#1}}

\newcommand{\gltf}[1]{\sigma(#1)}
\newcommand{\pv}[1]{p_{#1}}
\newcommand{\nxt}[2]{\bar{\sigma}(#1,#2)}
\newcommand{\qf}[2]{q_{#1,#2}}
\newcommand{\benl}{\bs{f}}
\newcommand{\enl}[1]{f_{#1}}
\newcommand{\strt}[1]{\overline{\str}_{#1}}
\newcommand{\crst}[1]{\overline{\crs}_{#1}}
\newcommand{\rtr}[2]{[#1~|~#2]}
\newcommand{\nc}[1]{b_{#1}}
\newcommand{\bsrt}[1]{\bs{g}^{#1}}
\newcommand{\srt}[2]{g^{#1}_{#2}}

\newcommand{\sbl}[3]{\enl{#3}^{#1,#2}}
\newcommand{\bsbl}[2]{\benl^{#1,#2}}

\newcommand{\pit}[2]{{\pi}^{#1}_{#2}}
\newcommand{\alt}[4]{\alpha^{#1,#2,#3}_{#4}}
\newcommand{\tq}{\hat{q}}
\newcommand{\btt}[1]{\beta^{#1}}
\newcommand{\dlt}[2]{\delta^{#1}_{#2}}
\newcommand{\ept}[2]{\epsilon^{#1}_{#2}}
\newcommand{\zet}[1]{\zeta_{#1}}

\newcommand{\cop}[1]{\ufu^{#1}}

\newcommand{\hi}{k}
\newcommand{\rel}[2]{d(#1,#2)}

\newcommand{\zed}{A}

\newcommand{\bpit}[1]{\bpe^{#1}}

\newcommand{\sh}[2]{h^{#1}_{#2}}

\newcommand{\novikoff}{Novikoff62}

\newcommand{\mcside}{HPT19}
\newcommand{\wmref}{LW94}
\newcommand{\MWM}{AHK12}
\newcommand{\hedge}{FS97}
\newcommand{\specials}{FSSW97}
\newcommand{\switchgraphmem}{ourJMLR15}
\newcommand{\msktoregret}{BPS09}
\newcommand{\littleopt}{litt88}
\newcommand{\fixedshare}{herbster1998tracking}
\newcommand{\MPP}{bousquet2003tracking}
\newcommand{\bayessleep}{adamskiy2012putting}
\newcommand{\freezing}{koolen2010freezing}
\newcommand{\tracklin}{herbster2001tracking}
\newcommand{\coinflipping}{OP16}
\newcommand{\agg}{AS90}
\newcommand{\RKHSaron}{a-trk-50}
\newcommand{\OLbook}{nicolobook}
\newcommand{\OCObook}{Shalev-Shwartz2011}
\newcommand{\maxnormrank}{LRSTS10} 

\newcommand{\MJH}[1]{{\color{purple}{{\rm\bfseries MJH:[}{\sffamily #1}{\rm\bfseries ]~}}}}
\newcommand{\LT}[1]{{\color{purple}{{\rm\bfseries LT:[}{\sffamily #1}{\rm\bfseries ]~}}}}
\newcommand{\HC}[1]{{}}

\renewcommand{\crs}{\overline{\mathcal{C}}}
\renewcommand{\str}{\overline{\mathcal{X}}}
\renewcommand{\alt}[4]{w^{#1,#2,#3}_{#4}}
\newcommand{\balt}[3]{\bs{w}^{#1,#2,#3}}
\newcommand{\talt}[2]{\bs{w}^{#1}_{#2}}
\renewcommand{\ci}{\alpha}
\renewcommand{\baw}[1]{\bs{w}^{#1}}
\renewcommand{\aw}[2]{w^{#1}_{#2}}
\renewcommand{\pv}[1]{p^{#1}}
\renewcommand{\crst}[1]{\mathcal{C}^{#1}}
\renewcommand{\strt}[1]{\mathcal{X}^{#1}}
\renewcommand{\nc}[1]{b^{#1}}
\renewcommand{\crw}[1]{\bar{\rho}(#1)}
\renewcommand{\trw}[1]{{\rho}(#1)}
\renewcommand{\ran}{\dis(\com)}
\renewcommand{\qf}[2]{q^{#1}_{#2}}
\renewcommand{\was}[1]{\mathcal{W}^{#1}}
\renewcommand{\nsw}[1]{\swi(#1)}
\renewcommand{\ei}{e}
\renewcommand{\li}{i}
\renewcommand{\com}{\bs{z}^*}
\renewcommand{\sh}[2]{{z}^{#1}_{#2}}
\newcommand{\ehb}{\kappa}
\newcommand{\tbc}[2]{C^{#2}_{#1}}
\newcommand{\blc}[2]{\bs{c}^{#2}_{#1}}
\newcommand{\lc}[3]{c^{#2}_{#1,#3}}
\renewcommand{\hi}{h}

\newcommand{\arv}[1]{\bar{v}^{#1}}
\newcommand{\arc}[1]{\bar{c}^{#1}}
\newcommand{\barv}[2]{\bar{v}^{#1}(#2)}
\newcommand{\barc}[2]{\bar{c}^{#1}(#2)}
\newcommand{\inprod}[2]{\langle#1,#2\rangle}

\renewcommand{\str}{\overline{\mathcal{T}}}
\renewcommand{\strt}[1]{\mathcal{T}^{#1}}
\renewcommand{\crst}[1]{\mathcal{A}^{#1}}
\renewcommand{\crs}{\overline{\mathcal{A}}}

\newcommand{\sht}{\mathcal{T}}
\newcommand{\sha}{\mathcal{A}}

\newcommand{\cset}[1]{S_{#1}}
\renewcommand{\was}[1]{\mathcal{W}_{#1}}
\newcommand{\bine}[1]{H(#1)}
\newcommand{\Bbine}[1]{H\!\left(#1\right)}

\allowdisplaybreaks

\newtoggle{shrink}
\toggletrue{shrink}
\begin{document}

\maketitle
\begin{abstract}
We introduce a novel online multitask setting.  In this setting each task is partitioned into a sequence of segments that is unknown to the learner.  Associated with each segment is a hypothesis from some hypothesis class.  We give algorithms that are designed to exploit the scenario where there are many such segments but significantly fewer associated hypotheses.  We prove regret bounds that hold for any segmentation of the tasks and any association of hypotheses to the segments. In the single-task setting this is equivalent to {\em switching with long-term memory} in the sense of~\cite{\MPP}.   
We provide an algorithm that predicts on each trial in time linear in the number of hypotheses when the hypothesis class is finite.  We also consider infinite hypothesis classes from reproducing kernel Hilbert spaces for which we give an algorithm whose per trial time complexity is cubic in the number of cumulative trials.  
In the single-task special case this is the first example of an efficient regret-bounded switching algorithm with long-term memory for a non-parametric hypothesis class.  
\end{abstract}
\section{Introduction}\label{sec:intro}
We consider a model of 
online prediction in a non-stationary environment with multiple interrelated tasks.  Associated with each task is an asynchronous data stream.   As an example, consider a scenario where a team of drones may need to decontaminate an area of toxic waste.  In this example, the tasks correspond to drones.  Each drone is receiving a data stream from its sensors.  The data streams are non-stationary but interdependent as the drones are travelling within a common site.  
At any point in time, a drone receives an instance $x$ and is required to predict its label $y$. 
The aim is to minimize mispredictions. As is standard in regret-bounded learning we have no statistical assumptions on the data-generation process.  Instead, we aim to predict well relative to some hypothesis class of predictors.
Unlike a standard regret model, where we aim to predict well in comparison to a single hypothesis, we instead aim to predict well relative to a completely unknown sequence of hypotheses in each task's data stream, as illustrated by the ``{\em coloring}'' in Figure~\ref{fig:model}.  Each {\em mode} (color) corresponds to a distinct hypothesis from the hypothesis class. A {\em switch} is said to have occurred whenever we move between modes temporally within the same task.

Thus in task~1, there are three modes and four switches. We are particularly motivated by the case that a mode once present will  possibly recur multiple times even within different tasks, i.e., $\text{``modes''} \ll \text{``switches.''}$  
We will give algorithms and regret bounds  
for finite hypothesis classes (the ``experts'' model~\cite{\agg,\wmref,\OLbook}) and for infinite non-parametric Reproducing Kernel Hilbert Space (RKHS)~\cite{\RKHSaron} hypothesis classes. 
\begin{figure}[ht]
\centering
\begin{minipage}[h]{0.48\linewidth}
\vspace{.13in}
\includegraphics[width=.98\linewidth]{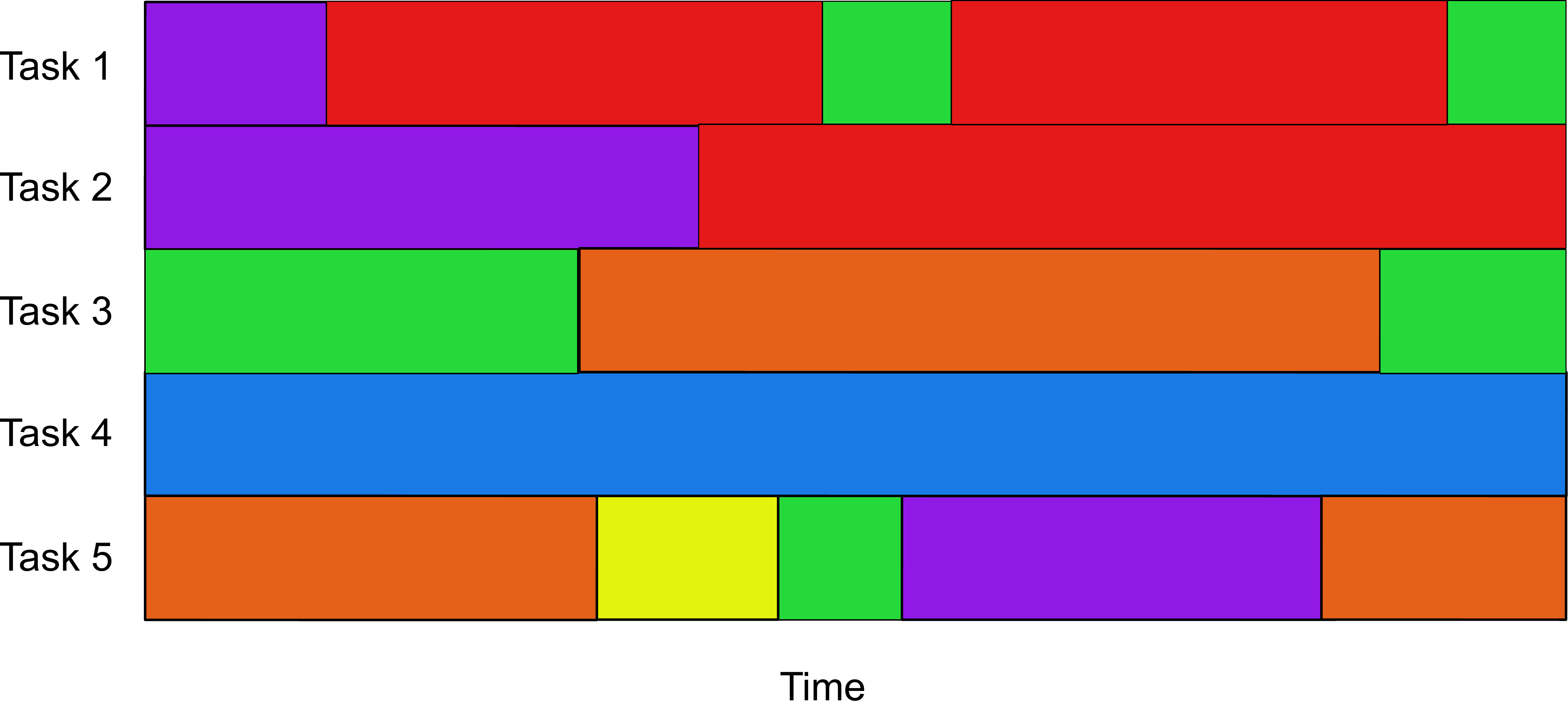} \vspace{0in}
\caption{\small A Coloring of Data Streams ($5$ tasks, $6$ modes, and $11$ switches).}\label{fig:model}
\end{minipage}
\begin{minipage}[h]{0.45\linewidth}
{\small
\begin{mdframed}
\begin{nstabbing}
For \= \=$\gti=1$ to $T$ do\\
\>\> Receive task $\ltau \in [\nL]$\,. \\
\>\> Set $i \leftarrow \ltau;\, t\leftarrow\gltF(\gti)$\,.  \\
\>\> Receive instance \hspace{.15cm} \= $\xtau \equiv  x^{i}_t \in \cX$\,.\\
\> \>Predict \>$\yhtau  \equiv \yhl\in \{-1,1\}$\,.\\
\>\> Receive label \> $\ytau  \equiv \yl \in \bset$\,.\\
\>\> Incur Loss \> $\lzo(\ytau,\yhtau)$\,.
\end{nstabbing}
\end{mdframed}
}
\caption{\small The Switching Multitask Model}\label{fig:mtp}
\end{minipage}
\end{figure}

The paper is organized as follows.  In the next section, we introduce our formal model for online switching multitask learning.  
In doing so we provide a brief review of some related online learning results which enable us to provide a prospectus for attainable regret bounds. This is done by considering the bounds achievable by non-polynomial time algorithms.  
We then provide a brief survey of related work as well as our notational conventions.  In Sections~\ref{sec:fin} and~\ref{sec:ker} we provide algorithms and bounds for finite hypothesis classes and RKHS hypothesis classes, respectively.  Finally, we provide a few concluding remarks in Section~\ref{sec:dis}.  The supplementary appendices contain our proofs.
\section{Online Learning with Switching, Memory, and Multiple Tasks}\label{sec:formalmodel}
We review the models and regret bounds for online learning in the single-task, switching, and switching with memory models as background for our multitask switching model with memory.

In the single-task online model a {\em learner} receives data sequentially so that on a trial $t=1,\ldots,T$:\\ 1) the learner receives an instance $x_t\in\cX$ from the {\em environment}, then 2) predicts a label $\yht\in\{-1,1\}$, then 3) receives a label from an environment $y_t\in \{-1,1\}$ and then 4) incurs a {\em zero-one} loss $\lzo(y_t,\yht) := [y_t \ne \yht]$.  There are no probabilistic assumptions on how the environment generates its instances or their labels; it is an arbitrary process which in fact may be adversarial.  The only restriction on the environment is that it does not ``see'' the learner's $\yht$ until after it reveals $y_t$. 
The learner's aim will be to  compete with a hypothesis class of predictors $\cH \subseteq \bset^\cX$ so as to minimize its {\em expected regret}, $R_T(h) := \sum_{t=1}^T \E [\lzo(y_t,\yht)] - \lzo(y_t,h(x_t))$ for every hypothesis $h\in\cH$, where the expectation is with respect to the learner's internal randomization.

In this paper we will consider two types of hypothesis classes: a finite set of hypotheses $\hfin$, and a set $\hK$ induced by a kernel $K$.   A ``multiplicative weight'' (\MW) 
algorithm~\cite{\MWM} 
that achieves a regret bound\footnote{Technically, when we say that an algorithm {\em achieves a bound}, it may be that the algorithm depends on a small set of parameters which we have then assumed are ``tuned'' optimally.}
of the form 
\begin{equation}\label{eq:exp}
R_T(h) \in \cO\left(\sqrt{\log(|\hfin|)T}\right) \quad (\forall h\in\hfin)
\end{equation}
was given in~\cite{CFHHSW97} for finite hypothesis classes. This is a special case of the framework of ``prediction with expert advice'' introduced in~\cite{AS90,\wmref}.  Given a reproducing kernel $K:\cX\times\cX\into\Re$ we denote the induced norm of the reproducing kernel Hilbert space (RKHS) $\hK$ as $\norm{\cdot}_{K}$  (for details on RKHS see~\cite{a-trk-50} also Appendix~\ref{app:rkhsrev}).
Given an instance sequence $\bxx := (x_1,\ldots,x_T)$, we let $\hKbin := \{h\in\hK : h(x_t) \in \bset,\forall t\in [T]\}$
denote the functions in $\hK$ that are binary-valued on the sequence.  
An analysis of online gradient descent (\OGD) with the hinge loss, kernel $K$ and randomized prediction~\cite[see e.g., Ch. 2 \& 3]{\OCObook} (proof included in Appendix~\ref{app:perc} for completeness) gives an expected regret bound of 
\begin{equation}\label{eq:ogd}
R_T(h)  \in \cO\left(\sqrt{\norm{h}^2_K X_K^2 T}\right) \quad (\forall h\in\hKbin)\,,
\end{equation}
where $X_K^2 \ge \max_{t\in [T]} K(x_t,x_t)$.

In the switching single-task model the hypothesis becomes a sequence of hypotheses $\bh=(h_1,h_2,\ldots,h_T)\in\cH^T$ and the regret is $R_T(\bh) :=\sum_{t=1}^T \E [\lzo(y_t,\yht)] - \lzo(y_t,h_t(x_t))$.   Two parameters of interest are the number of {\em switches} $\swi := \sum_{t=1}^{T-1} [h_t \ne h_{t+1}]$ and the number of {\em modes} $\dis := |\cup_{t=1}^{T} \{h_t\} |$, i.e., the number of the distinct hypotheses in the sequence.  In this work we are interested in {\em long-term memory}, that is, algorithms and bounds that are designed to exploit the case of $\dis \ll \swi$.

The methodology of~\cite{\fixedshare} may be used to derive an expected regret bound for  $\hfin$ in the switching single-task model of the form
\iftoggle{shrink}{
$R_T(\bh) \in \cO(\sqrt{(\swi \log(|\hfin|) + \swi \log({T}/{\swi}))T})$.
}{
$R_T(\bh) \in \cO\left(\sqrt{(\swi \log(|\hfin|) + \swi \log \frac{T}{\swi})T}\right)$. 
}
Freund in~\cite{freundopen} posed an open problem to improve the results of~\cite{\fixedshare} in the case of long-term memory ($\dis \ll \swi$).
Freund gave counting arguments that led to an exponential-time algorithm with a regret bound of
\iftoggle{shrink}{ 
$R_T(\bh) \in \cO(\sqrt{(\dis \log(|\hfin|) + \swi \log \dis  + \swi \log({T}/{\swi}))T})$. 
}{ 
$R_T(\bh) \in \cO\left(\sqrt{(\dis \log(|\hfin|) + \swi \log \dis  + \swi \log \frac{T}{\swi} )T}\right)$.
}
In~\cite{bousquet2003tracking} an efficient algorithm was given with nearly this bound, except for a small additional  additive ``$T \log\log T$'' term under the square root.  For the hypothesis class $\hKbin$
we may give non-memory bounds of the form 
$R_T(\bh) \in \cO(\sqrt{\swi \max_t \norm{h_t}^2_K X_K^2 T})$
by using a simple modification~\cite{\tracklin} of \OGD\ (see Appendix~\ref{app:perc}).  To the best of our knowledge there are no previous long-term memory bounds for $\hKbin$ (however see the discussion of~\cite{\switchgraphmem} in Section~\ref{sec:rw}); these will be a special case of our multitask model, to be introduced next.
\subsection{Switching Multitask Model}
In Figure~\ref{fig:mtp} we illustrate the protocol for our multitask model.  
The model is essentially the same as the switching single-task model, except that we now have $\nL$ tasks.  
On each (global) trial $\tau$ the environment reveals the active task $\ltau\in[\nL]$.  The ordering of tasks chosen by the environment is arbitrary, and therefore we may  switch tasks on every (global) trial $\tau$.  We use the following notational convention:
(global time) $\tau \equiv  \tmit$ (local time) where $i = \ltau$, $t = \gltF(\tau)$ and $\gltF(\tau) := \sum_{j=1}^{\tau} [\ell^j = \ltau]$.  Thus $\xtau\equiv\xit$, $\ytau\equiv \yit$, etc., where the mapping is determined implicitly by the task vector $\tv\in [\nL]^T$.
Each task $i\in [\nL]$ has its own data pair (instance, label) sequence $(x^{i}_1,y^{i}_1),\ldots,(x^{i}_{T^i},y^{i}_{T^i})$ where $T=T^1 +\ldots +T^{\nL}$.    
The {\em multitask hypotheses multiset} is denoted as 
$\gbh = (h^1,\ldots,h^T) \equiv (h^1_1,\ldots,h^1_{T^1},\ldots,h^{\nL}_{1},\ldots,h^{\nL}_{T^{\nL}})\in\cH^T$.  
In the multitask model we denote the number of switches as $\swi(\gbh) := \sum_{i=1}^\nL \sum_{t=1}^{T^i-1} [\hit \ne \hits]$, the set of modes as $\dis(\gbh) := \cup_{i=1}^{\nL} \cup_{t=1}^{T^i}\{h^i_t\}$ and the multitask regret as $R_T(\gbh) :=\sum_{i=1}^\nL \sum_{t=1}^{T^i} \Exp[\lzo(\yl,\yhl)] - \lzo(\yl,h^i_t(x^i_t))$.   In the following, we give motivating upper bounds based on exponential-time algorithms induced by ``meta-experts.''  We provide a lower bound with respect to $\hKbin$ in Proposition~\ref{prop:lb}.

The idea of ``meta-experts'' is to take the base class of hypotheses and to construct a class of ``meta-hypotheses'' by combining the original hypotheses to form new ones, and then apply an \MW\ algorithm to the constructed class; in other words, we reduce the ``meta-model'' to the ``base-model.''  In our setting, the base class is $\hfin\subseteq \bset^{\cX}$ and our meta-hypothesis class will be some $\cH' \subseteq \bset^{\cX'}$ where $\cX' := \{(x,t,i) : x\in\cX, t\in [T^i],i\in[\nL]\}$.
To construct this set we define $\bar{\cH}(\swi,\dis,\nL,\hfin,T^1,\ldots,T^\nL) := \{(h^1_1,\ldots,h^{\nL}_{T^{\nL}})=\bar{\bh}\in\hfin^{T} : \swi = \swi(\bar{\bh}), \dis = |\dis(\bar{\bh})|\}$ and then observe that  for each $\bar{\bh}\in \bar{\cH}$ we may define an $h' : \cX'\into\bset$ via $h'((x,t,i)) := {h}^i_t(x)$, where ${h}^i_t$ is an element of $\bar{\bh}$.   We thus construct $\cH'$ by converting each $\bar{\bh}\in \bar{\cH}$ to an $h'\in\cH'$. 
Hence we have reduced the switching multitask model to the single-task model with respect to $\cH'$.  We proceed to obtain a bound by observing that the cardinality of $\bar{\cH}$ is bounded above by $\binom{T-\nL}{\swi} \binom{\nex}{\dis}\dis^\nL(\dis-1)^\swi$ where $\nex = |\hfin|$.
If we then substitute into~\eqref{eq:exp} and then further upper bound we have
\iftoggle{shrink}{
\begin{equation}\label{eq:ineffin}
R_T(\gbh) \in  \cO\left(\sqrt{(\dis \log({\nex}/{\dis})+\nL \log\dis + \swi \log\dis+\swi \log(({T-\nL})/{\swi}))T}\right)\,,
\end{equation}
}{
\begin{equation}\label{eq:ineffin}
R_T(\gbh) \in  \cO\left(\sqrt{\left(\dis \log\left(\frac{\nex}{\dis}\right)+ (\nL + \swi )\log\dis+\swi \log\left(\frac{T-\nL}{\swi}\right)\right)T}\right)\,,
\end{equation}
}
for any $\gbh\in{\hfin^T}$ 
such that $\swi = \swi(\gbh)$ and  $\dis = |\dis(\gbh)|$.  The drawback is that the algorithm requires exponential time.  In Section~\ref{sec:fin} we will give an algorithm whose time to predict per trial is $\cO(|\hfin|)$ and whose bound is equivalent up to constant factors.

We cannot directly adapt the above argument to obtain an algorithm and bound for $\hKbin$ since the cardinality, in general, is infinite, and additionally we do not know $\bxx$ in advance.  However, the structure of the argument is the same.  Instead of using hypotheses from $\hKbin$ as building blocks to construct meta-hypotheses, we use multiple instantiations of an online algorithm for $\hKbin$ as our building blocks.
We let $\cA_K := \{a[1],\ldots,a[\dis]\}$ denote our set of $\dis$ instantiations that will act as a surrogate for the hypothesis class $\hKbin$.
We then construct the set,
$\bar{\cA}_K(\swi,\dis,\nL,T^1,\ldots,T^\nL) := \{\bar{\ba}\in\cA_K^{T} : \swi = \swi(\bar{\ba}), \dis = |\dis(\bar{\ba})|\}$.  Each  $\bar{\ba}\in \bar{\cA}_K$ now defines a meta-algorithm for the multitask setting.
That is, given an online multitask data sequence $(x_1^i,y_1^i), \ldots, (x^j_{T^j},y^j_{T^j})$, each element of $\bar{\ba}$ will ``color'' the corresponding data pair with one of the $\dis$ instantiations (we will use the function 
$\alpha:\{(t,i) : t\in [T^i],i\in[\nL]\}\into [\dis]$
to denote this mapping with respect to $\bar{\ba}$).  Each instantiation will {\em  receive as inputs  only the online sequence of the data pairs corresponding to its ``color''}; likewise, the prediction of meta-algorithm $\bar{\ba}$ will be that of the instantiation active on that trial.    We will use as our base algorithm \OGD.  Thus for the meta-algorithm $\bar{\ba}$ we have from~\eqref{eq:ogd}, 
\begin{equation}\label{eq:a} 
\sum_{i=1}^\nL \sum_{t=1}^{T^i} \Exp[\lzo(\yl,\yhl)]  \le  \sum_{i=1}^\nL \sum_{t=1}^{T^i}  \lzo(\yl,h[\alpha(\smash{{}^i_t})](x^i_t)) + \sum_{j=1}^\dis \cO\left(\sqrt{\norm{h[j]}^2_K X^2 T^j}\right)\,
\end{equation}
for any received instance sequence 
$\bxx \in\cX^T$ and for any $h[1],\ldots,h[\dis]\in\hKbin$.  
The $\MW$ algorithm~\cite{\wmref,\agg,\OLbook} does not work just for  hypothesis classes; more generally, it works for collections of algorithms.  Hence we may run the $\MW$ as a meta-meta-algorithm to combine all of the meta-algorithms $\bar{\ba}\in\bar{\cA}_K$.  
Thus by substituting the loss for  each meta-algorithm $\bar{\ba}$ (the R.H.S. of~\eqref{eq:a}) into~\eqref{eq:exp} and using the upper bound $\binom{T-\nL}{\swi}\dis^\nL(\dis-1)^\swi$
for the cardinality of $\bar{\cA}_K$, we obtain (using upper bounds for binomial coefficients and the inequality $\sum_i \sqrt{p_i q_i}\le\sqrt{(\sum_i p_i)(\sum_i q_i)}$)\,,
\begin{equation}\label{eq:inefker}\textstyle
R_T(\gbh) \in \cO\left(\sqrt{(\sum_{h\in \dis(\gbh)} \norm{h}^2_K X^2_K +\nL \log\dis + \swi \log \dis +\swi \log(({T-\nL})/{\swi}))T}\right)\,,
\end{equation}
for any received instance sequence 
$\bxx \in\cX^T$ and for any $\gbh\in{\hKbin}^T$
such that $\swi = \swi(\gbh)$ and  $\dis = |\dis(\gbh)|$.

The terms $\dis \log({\nex}/{\dis})$  (assuming $\dis \ll \nex$) and $\sum_{h\in \dis(\gbh)} \norm{h}^2_K X^2_K$ may be viewed as  {\em learner complexities}, i.e., the price we ``pay'' for identifying the hypotheses that fit the modes.  A salient feature of long-term memory bounds is that although the data pairs associated with each hypothesis are intermixed in the multitask sequence, we pay  the learner complexity only modestly in terms of potentially leading multiplicative constants.    A switching algorithm without long-term memory ``forgets'' and pays the full price for a mode on every switch or new task.   We gave exponential-time algorithms for $\hfin$ and $\hKbin$ with $\cO(1)$ leading multiplicative constants in the discussion leading to~\eqref{eq:ineffin} and~\eqref{eq:inefker}.
We give efficient algorithms for finite hypothesis classes and RKHS hypothesis classes in Sections~\ref{sec:fin} and~\ref{sec:ker}, with time complexities of $\cO(n)$ and $\cO(T^3)$ per trial, and in terms of learner complexities they gain only  leading multiplicative constants of  $\cO(1)$ and $\cO(\log T)$.
\subsection{Related Work}\label{sec:rw}
In this section we briefly describe other related work in the online setting that considers either {\em switching} or {\em multitask} models.

The first result for switching in the experts model was the WML algorithm~\cite{\wmref} which was generalized in~\cite{\fixedshare}.  There is an extensive literature building on those papers, with some prominent results including~\cite{bousquet2003tracking,wacky,\freezing,fixmir,GLL12,\bayessleep,fixmir,AKCV12,DGS15,alt2017growing,BDS19,ZLDW19}.
Relevant for our model are those papers~\cite{\MPP,\freezing,\bayessleep,fixmir,alt2017growing,BDS19,ZLDW19} that address the problem of long-term memory ($\dis \ll \swi$), in particular~\cite{bousquet2003tracking,\freezing,\bayessleep}. 

Analogous to the problem of long-term memory in online learning is the problem of catastrophic forgetting in artificial neural network research~\cite{MC89,F99}.  That is the problem of how a system can adapt to new information without forgetting the old.  In online learning that is the problem of how an algorithm can both quickly adapt its prediction hypothesis and recall a previously successful prediction hypothesis when needed.  
In the experts model this problem was first addressed by~\cite{bousquet2003tracking}, which gave an algorithm that stores each of its past state vectors, and then at each update mixes these vectors into the current state vector. In~\cite{\freezing}, an algorithm and bounds  were given that extended the base comparision class of experts to include Bernoulli models. 
An improved algorithm with a Bayesian intepretation based on the idea of ``circadian specialists'' was given for this setting in~\cite{\bayessleep}.  
Our construction of Algorithm~\ref{alg:fin} was based on this methodology.

The problem of linear regression with long term memory was posed as an open problem in~\cite[Sec. 5]{\bayessleep}.  Algorithm~\ref{alg:ker} gives an algorithm for linear interpolation in a RKHS with a regret bound that reflects long-term memory.
Switching linear prediction has been considered in~\cite{herbster2001tracking,onlinekernellearning,budget,\switchgraphmem}.  Only~\cite{\switchgraphmem} addresses the issue of long-term memory.  The methodology of~\cite{\switchgraphmem} is a direct inspiration for Algorithm~\ref{alg:ker}.  We significantly extend the result of~\cite[Eq.~(1)]{\switchgraphmem}.  Their result was i) restricted to a mistake as opposed to a regret bound, ii) restricted to finite positive definite matrices and  
iii) in their mistake bound the term analogous to $\sum_{h\in \dis(\gbh)} \norm{h}^2_K X^2_K$ was increased by a multiplicative factor of $\cOT(|\dis(\gbh)|)$, a significantly weaker result.

Multitask learning has been considered extensively in the batch setting, with some prominent early results including~\cite{B95,C97,EP04}.  In the online multitask {\em expert} setting~\cite{Abernethy07,ARB08,SahaRDV11,adamskiy2012putting} considered a model which may be seen as a special case of ours where each task is associated only with a single hypothesis, i.e., no internal switching within a task.   Also in the expert setting~\cite{RAB07,LPS09}  considered models where the prediction was made for all tasks simultaneously.  In~\cite{LPS09} the aim was to predict well relative to a set of possible predefined task interrelationships and in~\cite{RAB07} the interrelationships were to be discovered algorithmically.  The online multitask {\em linear} prediction setting was considered in~\cite{dekel2007online,Cavallanti,hstk17}.  The models of~\cite{Cavallanti,hstk17} are similar to ours, but like previous work in the expert setting, these models are limited to one ``hypothesis'' per task.  In the work of~\cite{dekel2007online}, the predictions were made for all tasks simultaneously through a joint loss function. 
\subsection{Preliminaries}\label{sec:prelim}
For any positive integer $m$, we define $[m] := \left\{1,2,\ldots,m\right\}$. For any predicate $[\mbox{\sc pred}] :=1$ if $\mbox{\sc pred}$ is true and equals 0 otherwise, and for any $x\in \Re$, $[x]_{+} := x [x>0]$.  
We denote the inner product of vectors as both $\bxx,\bw\in\Re^n$ as $\dotp{\bxx,\bw} =\bxx \cdot\bw= \sum_{i=1}^n x_i w_i$, component-wise multiplication $\bxx \odot \bw : = (x_1 w_1,\ldots,x_n w_n)$ and the norm as $\vn{\bw} = \sqrt{\dotp{\bw,\bw}}$.  If $f:\Re\to\Re$ and $\bxx\in\Re^n$ then $f(\bxx):= (f(x_1),\ldots,f(x_n))$.  The $x$th-coordinate vector is denoted $\be_X^x := ([x=z])_{z\in X}$;  we commonly abbreviate this to $\be^x$.   We denote the probability simplex as $\Delta_\cH := \{h\in [0,1]^\cH\} \cap \{h : \sum_{h\in\cH} =1\}$ and set $\Delta_n := \Delta_{[n]}$.   We denote the binary entropy as $\bine{p}:=p\log \frac{1}{p}+(1-p)\log \frac{1}{1-p}$.
If $\bvv \in \Delta_{\cH}$ then $h \sim \bvv$ denotes that $h$ is a random sample from the probability vector $\bvv$ over the set $\cH$.
For vectors $\bp \in \Re^m$ and $\bq \in \Re^n$ we define $\con{\bp}{\bq}\in \Re^{m+n}$ to be the concatenation of $\bp$ and $\bq$, which we regard as a column vector. Hence $\con{\bp}{\bq}^\trans\!{\con{\boldsymbol{\bar \bp}}{\boldsymbol{{\bar \bq}}}}=\bp^\trans \boldsymbol{\bar \bp}+\bq^\trans \boldsymbol{\bar \bq}$. 

 The notation $\bM^+$ and $\sqrt{\bM}$
 denotes the pseudo-inverse and the unique positive square root, respectively, of a positive semi-definite matrix $\bM$.  The trace of a square matrix is denoted by $\trace{\bY} := \sum_{i=1}^n Y_{ii}$ for $\bY\in \Re^{n\times n}$. 
The $m \times m$ identity matrix is denoted $\bI^m$.   
A function $K: \cX\times \cX\into\Re$ is a strictly positive definite (SPD) kernel iff for every finite $X \subseteq \cX$ the matrix $K(x,x')_{x,x'\in X}$ is symmetric and strictly positive definite, for example, the Gaussian kernel.
In addition, we define $\bS^m$ to be the set of $m \times m$ symmetric matrices and let $\bS^m_+$ and $\bS_{++}^m$ be the subset of positive semidefinite and strictly positive definite matrices, respectively. 
We define the squared radius of $\RN \in \bS_+^m$ as $\RRN := \max_{i\in [m]} M^+_{ii}$.   The (undirected) graph Laplacian matrix is defined by $\bL:= \bD-\bA$ where $\bD$ is the degree matrix and $\bA$ is the adjacency matrix.
The corresponding (strictly) positive definite {\em  PDLaplacian} of an $m$-vertex connected graph is 
$\bLb := \la+\RAD_L^{-1}\left(\frac{\one}{\m}\right)\left(\frac{\one}{\m}\right)^\trans$. 

\section{Finite Hypothesis Classes}\label{sec:fin}
\begin{algorithm}[h]
\begin{algorithmic}\small 
\caption{Predicting $\hfin$ in a switching multitask setting. \label{alg:fin}}
\renewcommand{\algorithmicrequire}{\textbf{Parameters:}} 
\REQUIRE $\hfin \subseteq \bset^\cX$; $\nL,\dis,\swi,T\in\N$ 
\\  
\renewcommand{\algorithmicrequire}{\textbf{Initialization:}}   \vspace{.1truecm} 
\REQUIRE $\nex := |\hfin|$;\,$\bpe^1 \leftarrow\frac{\bone}{\nex}$;
$\sip:=\frac{1}{\disj}$;\, 
$\bw^1_1\! =\!\! \cdots\!\! =\! \bw^\nL_1\! \leftarrow\!\sip\bone$
;\, $\ww:=1-\frac{\swi}{\gtu-\nL}$;\, $\sw:=\frac{\swi}{(\dis-1)(\gtu-\nL)}$ and
$ 
\eta:= \sqrt{\left(\dis\log\left(\frac{\nex}{\dis}\right)+\nL\dis\Bbine{\frac{1}{\dis}}+(\gtu-\nL)\Bbine{\frac{\swi}{\gtu-\nL}}
+(\dis-1)(\gtu-\nL)\Bbine{\frac{\swi}{(\dis-1)(\gtu-\nL)}}\right)\frac{2}{T}}
$

\renewcommand{\algorithmicrequire}{\textbf{For}}
\REQUIRE $\gti =1,\dots,T$  \vspace{.1truecm}
\STATE $\bullet$ Receive task $\ltau \in [\nL]$\,. \\ \vspace{.1truecm}
\STATE $\bullet$ Receive $\xtau\in\cX$\,. \\ \vspace{.1truecm}
\STATE $\bullet$ Set $i \leftarrow \ltau;\, t\leftarrow\gltF(\gti)$\,. \\ \vspace{.1truecm}
\STATE $\bullet$ Predict
\begin{align*}
&\bal{\gti}\leftarrow\frac{\bpe^{\gti}\odot\bawit}{\bpe^{\gti}\cdot\bawit},
~~~~~~\hat{h}^{\gti}\sim\bal{\gti},
~~~~~~\yhtau\leftarrow \hat{h}^{\gti}(\xtau)\,.
\end{align*}
\vspace{-.13in}
\STATE $\bullet$ Receive $\ytau \in \bset$\,. \\ \vspace{.1truecm}
\STATE $\bullet$ Update:
\begin{alignat*}{2}
&i)\ \ \ \ \forall h\in\hfin,~\lo{\gti}{h}=\lzo(h(\xtau),\ytau)\quad\quad\quad
&ii)&\ \bwaw\leftarrow\bawit\odot\exp(-\lr\blo{\gti})\\
&iii)\ \saw\leftarrow(\bpe^{\gti}\cdot\bawit)/(\bpe^{\gti}\cdot\bwaw)
&iv)&\ \bnaw\leftarrow \bone-\bawit+\saw\bwaw\\
&v)\ \ \,  \bpe^{\gti+1}\leftarrow\bpe^{\gti}\odot\bnaw
&vi)&\ \bawits\leftarrow(\sw(\bone-\bawit)+\ww\saw\bwaw)\odot\bnaw^{-1}\
\end{alignat*}
\end{algorithmic}
\end{algorithm}

In this section we present the algorithm and the regret bound for finite hypothesis classes, with proofs given  in Appendix~\ref{sec:proof_fin}.
The design and analysis of the algorithm is inspired by~\cite{adamskiy2012putting}, which considers a Bayesian setting where, on each trial, each hypothesis $h$ gives an estimated probability $P(y^{\gti}=\hat{y}|h)$ of the outcome $y^{\gti}$. The idea is for the learner to predict a probability $\hat{P}(y^{\gti}=\hat{y})$ and the loss incurred is the log loss, $\log(1/\hat{P}(y^{\gti}))$. Our algorithm, on the other hand, is framed in the well known ``Allocation'' setting~\cite{\hedge} where the learner must play, on trial $\tau$, a vector $\bal{\gti}\in\Delta_n$ and incurs a loss of $\blo{\gti}\cdot\bal{\gti}$ where all components of $\blo{\gti}$ are in $[0,1]$.

To gain some intuition about the algorithm we observe the following.  
The algorithm maintains and updates the following vectors:  a ``global'' probability vector $\bpe^{\gti}\in\Delta_{\hfin}$ and the ``local''  task weight vectors $\baw{1}_t,\ldots,\baw{\nL}_t\in[0,1]^{\hfin}$. Given an hypothesis $h\in\hfin$, the scalar $\pe{h}^{\gti}$ represents our ``confidence'', on trial $\gti$, that hypothesis $h$ is in $\dis(\gbh)$.  For a given task $i$, hypothesis $h\in\hfin$, and local time $t$, the scalar $\aw{i}{{t,h}}$ represents our confidence that $h=h^i_t$ if we knew that  $h$ was in $\dis(\gbh)$. Putting together, $\pe{h}^{\gti}\aw{i}{\gltF(\gti),h}$ represents our confidence, on trial $\gti$, that $h=h^i_{t}$. The weights $\bpe^{\gti}$ and $\baw{i}_t$ (for tasks $i$) are designed in such a way that, not only do they store all the information required by the algorithm, but also on each trial $\gti$ we need only update $\bpe^{\gti}$ and $\baw{\liv{\gti}}_t$\!.  Thus the algorithm predicts in $\cO(\nex)$ time per trial and requires $\cO(\nL \nex)$ space.  We bound the regret of the algorithm in the following theorem.

\begin{theorem}\label{thm:fin}
The expected regret  of Algorithm~\ref{alg:fin} with parameters $\hfin \subseteq \bset^\cX$; $\nL,\dis,\swi,T\in\N$  and
\[
C := \dis\log\left(\frac{\nex}{\dis}\right)+\nL\dis\Bbine{\frac{1}{\dis}}+(\gtu-\nL)\Bbine{\frac{\swi}{\gtu-\nL}}
+(\dis-1)(\gtu-\nL)\Bbine{\frac{\swi}{(\dis-1)(\gtu-\nL)}}
\]
is bounded above by
\begin{equation*}
\sum_{i=1}^\nL \sum_{t=1}^{T^i} \Exp[\lzo(\yl,\yhl)] - \lzo(\yl,h^i_t(x^i_t))  \le\sqrt{2 C T}
\end{equation*}
 for any $\gbh\in{\hfin^T}$  such that $\swi = \swi(\gbh)$, $\dis \ge |\dis(\gbh)|$, $\dis>1$.  Furthermore,
\[
C\le
\dis\log\left(\frac{\nex}{\dis}\right)+\nL(\log(\dis)+1)+\swi\left(\log(\dis-1)+2\log\left(\frac{\gtu-\nL}{\swi}\right)+2\right)\,.
\]
\end{theorem}
In further comparison to~\cite{\bayessleep} we observe that we can obtain bounds for the log loss  
with our algorithm by defining $\hat{P}(y^{\gti}=\hat{y}):=\sum_{h}\al{\gti}{h}P(y^{\gti}=\hat{y}|h)$ and redefining $\lo{\gti}{h}:=-\frac{1}{\lr}\log(P(y^{\gti}=\hat{y}|h))$ in the update.    The resultant theorem then matches the bound of~\cite[Thm.~4]{\bayessleep} for single-task learning with long-term memory ($\nL=1$) and the bound of~\cite[Thm.~6]{\bayessleep} for multitask learning with no switching ($\swi=0$).
\section{RKHS Hypothesis Classes}\label{sec:ker}
Our algorithm and its analysis builds on the algorithm for {online inductive matrix completion with side-information} (\IMCSI) from~\cite[Theorem 1, Algorithm 2 and Proposition 4]{\mcside}.   \IMCSI\ is an example of a matrix multiplicative weight algorithm~\cite{tsuda2005matrix,AHK12}.  We give notation and background from~\cite{\mcside} to provide insight.

The max-norm (or $\gamma_2$ norm \cite{CMSM07}) of a matrix $\bU \in \gm$ is defined by
\iftoggle{shrink}{ 
\begin{equation}
\label{eq:maxnorm}
\maxnorm{\U} := \min_{\PP \QQ^\trans = \U} \left\{\max_{1\leq i\leq m} \vn{\PP_{i}} \times \max_{1\leq j\leq n}\vn{\QQ_{j}}\right\}\,,
\end{equation}
}{ 
\begin{equation}
\label{eq:maxnorm}
\maxnorm{\U} := \min_{\PP \QQ^\trans = \U} \left\{\max_{1\leq i\leq m} \vn{\PP_{i}} \times \max_{1\leq j\leq n}\vn{\QQ_{j}}\right\}\,,
\end{equation} 
}
where the minimum is over all matrices $\PP \in \gmd$ and $\QQ \in \gnd$ and every integer $d$.
We denote the class of $m \times d$ {\em row-normalized} matrices as
$\RNM^{m,d} := \{\hat{\bP} \subset \Re^{m \times d} : \norm{\hat{\bP}_{i}} =1,\, i\in [m]\}$.   The quasi-dimension of a matrix is defined as follows.
\begin{definition}[{{\cite[Equation~(3)]{\mcside}}}]\label{def:qd}
The {\em quasi-dimension} of a matrix  $\U\in\gm$ with respect to $\RN\in\bS_{++}^m,\, \CN\in\bS_{++}^n$ at~$\marh$ as
\begin{equation}\label{eq:defpdim}
\qD(\bU) := \min_{\nP \nQ^\trans = {\marh}\U} \trc\left(\nP^{\trans}\RN\nP  \right) \RRN
+
\trc\left(\nQ^{\trans}\CN\nQ \right)\RCN\,,
\end{equation}
where the infimum is over all row-normalized matrices $\nP \in \RNM^{m,d}$ and $\nQ \in \RNM^{n,d}$ and every integer~$d$. If the infimum does not exist then $\qD(\bU) :=+\infty$  (The infimum exists iff $\maxnorm{\U} \le 1/\marh$).
\end{definition}    
The algorithm \IMCSI\ addresses the problem of the online prediction of a binary comparator matrix $\U$ with side information.   The side  information is supplied as a pair of kernels over the row indices and the column indices.  In~\cite[Theorem 1]{\mcside} a regret bound $\cOT(\sqrt{({\upD}/{\gamma^{2}}){T}})$ is given, where $1/\gamma^2\ge\maxnormsqr{\bU}$ and $\upD\ge\qD(\bU)$ are parameters of the algorithm that serve as upper estimates on $\maxnormsqr{\bU}$ and $\qD(\bU)$.  The first estimate $1/\gamma^2$ is an upper bound on the squared max-norm (Eq.~\eqref{eq:maxnorm}) which like the trace-norm may be seen as a proxy for the rank of the matrix~\cite{\maxnormrank}.  The second estimate $\upD$ is an upper bound of the {\em quasi-dimension} (Eq.~\eqref{eq:defpdim}) which measures the quality of the side-information.  
The quasi-dimension depends upon the ``best'' factorization $(1/\gamma) \nP \nQ^\trans=\U$, which will be smaller when the row $\nP$ (column $\nQ$) factors are in congruence with the row (column) kernel.
We bound the quasi-dimension in Theorem~\ref{thm:itdit} in Appendix~\ref{app:ker} as a key step to proving Theorem~\ref{thm:ker}. 

In the reduction of our problem to a matrix completion problem with side information, the row indices correspond to the domain of the learner-supplied kernel $K$ and the column indices correspond to the temporal dimension.  On each trial we receive an $\xtau$ (a.k.a. $\xit$).  Thus the column of the comparator matrix (now $\bH$) corresponding to time $\tau$ will contain the entries $\Htau =(\htau(x^\upsilon))_{\upsilon\in[T]}$.  
Although we are predicting functions that are changing over time, the underlying assumption is that the change is sporadic; otherwise it is infeasible to prove a non-vacuous bound.  Thus we expect  $\Hit\approx\Hits$ and as such our column side-information kernel should reflect this expectation.  Topologically we would therefore expect a kernel to present as $\nL$ separate time {\em paths}, where nearness in time is nearness on the path.  In the following we introduce the {\em path-tree-kernel} (the essence of the construction was first introduced in~\cite{herbster2009online}), which satisfies this expectation in the single-task case.  We then adapt this construction to the multitask setting.

A {\em path-tree} kernel $\PTK : [T] \times [T] \into \Re$, is formed via the Laplacian of a fully complete binary {\em tree} with $N := 2^{\ceil{\log_2 T}+1}-1$ vertices.  The {\em path} corresponds to the first $T$ leaves of the tree, numbered sequentially from the leftmost to the rightmost leaf of the first $T$ leaves.  Denote this Laplacian as $\bL$ where the path is identified with $[T]$ and the remaining vertices are identified with $[N]\setminus [T]$.  Then using the definition   
$\bLb := \la+\left(\frac{\one}{N}\right)\left(\frac{\one}{N}\right)^\trans\RAD_{\bL}^{-1}$ we define $\PTK(\tau,\upsilon) := ({\bLb})^{+}_{\tau\upsilon}$ where $\tau, \, \upsilon \in [T]$.   We extend the path-tree kernel to a {\em multitask-path-tree} kernel by dividing the path into $\tasks$ contiguous segments, where segment $i$ is a path of length $T^i$, and the task vector $\tv\in[\tasks]^T$ determines the mapping from global trial $\tau$ to task $\ltau$ and local trial $\gltF(\tau)$.
We define $\PTKT^{\tv,T^1,\ldots,T^{\nL}} : [T] \times [T] \into \Re$ as
\iftoggle{shrink}{ 
$
\PTKT^{\liV,T^1,\ldots,T^{\nL}}(\tau,\upsilon) := \PTK\left(\sum_{i=1}^{\ltau-1} T^i + \gltF(\tau),\sum_{i=1}^{\lups-1} T^i + \gltF(\upsilon)\right)\,. $
}{
\[ \PTKT^{\liV,T^1,\ldots,T^{\nL}}(\tau,\upsilon) := \PTK\left(\sum_{i=1}^{\ltau-1} T^i + \gltF(\tau),\sum_{i=1}^{\lups-1} T^i + \gltF(\upsilon)\right)\,.\]
}
Observe we do not need to know the task vector $\tv$ in advance; we only require upper bounds on the lengths of the tasks to be able to use this kernel.   
Finally, we note that it is perhaps surprising that we use a tree rather than a path directly.  We discuss this issue following Lemma~\ref{lem:PTK} in Appendix~\ref{app:ker}.

Algorithm~\ref{alg:ker} requires $\cO(t^3)$ time per  trial $t$ since we need to compute the eigendecomposition of three $\cO(t)\times\cO(t)$ matrices as well as sum  $\cO(t)\times\cO(t)$ matrices up to $t$ times. 
\begin{algorithm}[h]
\begin{algorithmic}\small 
\caption{Predicting $\hKbin$ in a switching multitask setting. \label{alg:ker}}
\renewcommand{\algorithmicrequire}{\textbf{Parameters:}} 
\REQUIRE Tasks $\nL \in \N$, task lengths $T^1,\ldots,T^{\nL}\in \N$, $T:=\sum_{i=1}^{\nL} T^i$, learning rate: $\lr>0 $, complexity estimate: $\CIH>0$, modes: $\dis\in [T]$, SPD Kernel $K:\cX\times\cX\into\Re$, $\PTKT := \PTKT^{\tv,T^1,\ldots,T^{\nL}}:[T]\times [T]\rightarrow\Re$,  
with $ \max_{\gti\in[T]} K(\xtau,\xtau)\le \KRH $, and $\KPH := 2 \ceil{\log_2 T}$.
\\  \vspace{.1truecm}
\renewcommand{\algorithmicrequire}{\textbf{Initialization:}}   \vspace{.1truecm} \vspace{.1truecm}
\REQUIRE $\uset \leftarrow \emptyset\,, \cX^1 \leftarrow \emptyset\,, \,\cT^1 \leftarrow \emptyset\, \,.$   \vspace{.1truecm}
\renewcommand{\algorithmicrequire}{\textbf{For}}
\REQUIRE $\gti =1,\dots,T$  \vspace{.1truecm}
\STATE $\bullet$ Receive task $\ltau \in [\nL]$\,. \\ \vspace{.1truecm}
\STATE $\bullet$ Receive $\xtau\in\cX$\,.\\ \vspace{.1truecm}
\STATE $\bullet$ Set $i \leftarrow \ltau;\, t\leftarrow\gltF(\gti);x^i_t \equiv \xtau$\,. 
\STATE $\bullet$ Define \vspace{-.1in}
\begin{align*}
&\bK^{\gti}   := (K(x,z))_{x,z\in \cX^{\gti}\cup \{\xtau\}}\,; \quad \bPTK := (\PTKT(\gti,\upsilon))_{\gti,\upsilon\in \cT^\gti\cup \{\gti\}}\,,\\
& \quad \XTau(\upsilon) := \con{\frac{\sqrt{\bK^{\gti}}{\be^{\xtup}}}{\sqrt{2\KRH}}}{\frac{\sqrt{\bPTK}\be^{\upsilon}}{\sqrt{2\KPH}}} \con{\frac{\sqrt{\bK^{\gti}}{\be^{\xtup}}}{\sqrt{2\KRH}}}{\frac{\sqrt{\bPTK}\be^{\upsilon}}{\sqrt{2\KPH}}}^{\trans}\,, \\
&\quad \wem{\gti} \leftarrow \exp \left( \log\left(\frac{\CIH}{2T m}\right) \id^{|\cX^{\gti}|+|\cT^{\gti}| +2} + \sum_{\upsilon\in\uset}  \lr\y{\upsilon}  \XTau(\upsilon) \right)\,.
\end{align*}
\vspace{-.13in}
\STATE $\bullet$  Predict 
\begin{equation*} \Yrvtau \sim \mbox{\sc Uniform}(-\marh,\marh)\,;\quad \ybtau \leftarrow \tr{\wem{\tau}\XTau}-1 \,;\quad \yhl := \yhtau \leftarrow \sign(\ybtau-\Yrvtau)\,.
\end{equation*}\vspace{-.16in}
\STATE $\bullet$ Receive label $\yl :=\ytau \in \{-1,1\}$\,.\vspace{.1truecm}
\STATE $\bullet$ If $\ytau\ybtau \leq \frac{1}{\sqrt{\dis}}$ then 
\begin{equation*}
\uset \leftarrow \uset \cup \{t\}\,,\ \ \cX^{\tau+1} \leftarrow \cX^{\tau} \cup \{\xtau\}, \text{ and } \cT^{\tau+1} \leftarrow \cT^{\tau} \cup \{\tau\}\,.
\end{equation*}\vspace{-.16in}
\STATE $\bullet$  Else $\cX^{\tau+1} \leftarrow \cX^{\tau}$ and $\cT^{\tau+1} \leftarrow \cT^{\tau}$\,.
\end{algorithmic}
\end{algorithm}
We bound the regret of the algorithm as follows.
\begin{theorem}\label{thm:ker}
The expected regret of Algorithm~\ref{alg:ker} with upper estimates, $\swi \ge \swi(\gbh)$, $\dis \ge |\dis(\gbh)|$,
\[
\CIH\ge \CI(\gbh) := \left(\sum_{h\in \dis(\gbh)} \norm{h}^2_K X^2_K + 2(\nL+{\swi}-1)\dis \lceil\log_2 T\rceil^2+2m^2\right)\,,\vspace{-.07in} 
\]
$\KRH\ge\max_{\gti\in[T]} K(\xtau,\xtau)$, and learning rate $\lr = \sqrt{\frac{\CIH \log(2T) }{2T\dis}}$ is bounded  by\vspace{-.07in} 
\begin{equation}
\sum_{i=1}^\nL \sum_{t=1}^{T^i} \Exp[\lzo(\yl,\yhl)] - \lzo(\yl,h^i_t(x^i_t)) \le 4\sqrt{2\CIH\, T\log(2T)}
\end{equation}
with received instance sequence $\bxx \in\cX^T$ and for any $\gbh \in{\hKbin}^T$.
\end{theorem}
Comparing roughly to the bound of the exponential-time algorithm (see~\eqref{eq:inefker}),  we see that the $\log m$ term has been replaced by an $m$ term and we have gained a multiplicative factor of $\log 2T$.  From the perspective of long-term memory, we note that
the potentially dominant learner complexity term $\sum_{h\in \dis(\gbh)} \norm{h}^2_K X^2_K$ has only increased by a slight $\log 2T$ term. 
To gain more insight into the problem we also have the following simple lower bound.
\begin{proposition}\label{prop:lb}  For any (randomized) algorithm and any $\nL,\swi,\dis,\Gamma\in \N$, with $\swi +\nL \ge \dis> 1$ and $\Gamma \ge \dis \log_2 \dis$, there exists a kernel $K$ and a $T_0\in\mathbb{N}$ such that for every $T\ge T_0$:
\[
\sum_{\tau=1}^{T} \Exp[\lzo(\ytau,\yhtau)] -\lzo(\ytau,\htau(\xtau)) \in  \Omega\left(\sqrt{ \left(\Gamma +\nL \log\dis + \swi \log \dis \right)T}\right)\,,
\]
for some multitask sequence $(x^1,y^1),\ldots,(x^T,y^T)\in(\cX\times\{-1,1\})^T$ and some $\gbh\in[{\hKbin}]^T$ such that $\dis \ge |\dis(\gbh)|$, $\swi \ge \swi(\gbh)$, $\sum_{h\in \dis(\gbh)} \norm{h}^2_K X^2_K\ge |\dis(\gbh)| \log_2 \dis$, where $X^2_K=\max_{\tau\in [T]} K(\xtau,\xtau)$.
\end{proposition}
Comparing the above proposition to the bound of the exponential-time algorithm (see~\eqref{eq:inefker}), the most striking difference is the absence of the $\log T$ terms.  We conjecture that these terms are not necessary for the $0$-$1$ loss.  A proof of Theorem~\ref{thm:ker} and a proof sketch of Proposition~\ref{prop:lb} are given in Appendix~\ref{app:ker}.\vspace{-.1in}
\section{Discussion}\label{sec:dis}
We have presented a novel multitask setting which generalizes single-task switching under the long-term memory setting.  We gave algorithms for finite hypothesis classes and for RKHS hypothesis classes with per trial prediction times of $\cO(n)$ and $\cO(T^3)$.  We proved upper bounds on the regret for both cases as well as a lower bound in the RKHS case.  An open problem is to resolve the gap in the RKHS case. 
On the algorithmic side, both algorithms depend on a number of parameters.  There is extensive research in online learning methods to design parameter-free methods. Can some of these methods be applied here (see e.g.,~\cite{\coinflipping})?  For a non-parametric hypothesis class, intuitively it seems we must expect some dependence on $T$.  However can we perhaps utilize decay methods such as~\cite{DSS18,CCG07} or sketching methods~\cite{CDST99} that have had success in simpler models to improve running times?  More broadly, for what other infinite hypothesis classes can we give efficient regret-bounded algorithms in this switching multitask setting with long-term memory?
\section{Acknowledgements}
This research was sponsored by the U.S. Army Research Laboratory and the U.K. Ministry of Defence under Agreement Number W911NF-16-3-0001. The views and conclusions contained in this document are those of the authors and should not be interpreted as representing the official policies, either expressed or implied, of the U.S. Army Research Laboratory, the U.S. Government, the U.K. Ministry of Defence or the U.K. Government. The U.S. and U.K. Governments are authorized to reproduce and distribute reprints for Government purposes notwithstanding any copyright notation hereon.  This research was further supported by the Engineering and Physical Sciences Research Council [grant number EP/L015242/1].

\newpage

{
\bibliographystyle{unsrt}
\bibliography{mp31}

\begin{thebibliography}{10}

\bibitem{bousquet2003tracking}
O.~Bousquet and M.K. Warmuth.
\newblock Tracking a small set of experts by mixing past posteriors.
\newblock {\em Journal of Machine Learning Research}, 3:363--396, 2003.

\bibitem{AS90}
Volodimir~G. Vovk.
\newblock Aggregating strategies.
\newblock In {\em Proceedings of the Third Annual Workshop on Computational
  Learning Theory}, COLT '90, pages 371--386, 1990.

\bibitem{LW94}
Nick Littlestone and Manfred~K. Warmuth.
\newblock The weighted majority algorithm.
\newblock {\em Inf. Comput.}, 108(2):212--261, February 1994.

\bibitem{nicolobook}
Nicolo Cesa-Bianchi and Gabor Lugosi.
\newblock {\em Prediction, Learning, and Games}.
\newblock Cambridge University Press, New York, NY, USA, 2006.

\bibitem{a-trk-50}
N.~Aronszajn.
\newblock Theory of reproducing kernels.
\newblock {\em Trans. Amer. Math. Soc.}, 68:337--404, 1950.

\bibitem{AHK12}
Sanjeev Arora, Elad Hazan, and Satyen Kale.
\newblock The multiplicative weights update method: a meta-algorithm and
  applications.
\newblock {\em Theory of Computing}, 8(6):121--164, 2012.

\bibitem{CFHHSW97}
Nicol\`{o} Cesa-Bianchi, Yoav Freund, David Haussler, David~P. Helmbold,
  Robert~E. Schapire, and Manfred~K. Warmuth.
\newblock How to use expert advice.
\newblock {\em J. ACM}, 44(3):427?485, May 1997.

\bibitem{Shalev-Shwartz2011}
S.~Shalev-Shwartz.
\newblock {Online Learning and Online Convex Optimization}.
\newblock {\em Foundations and Trends{\textregistered} in Machine Learning},
  4(2):107--194, 2011.

\bibitem{herbster1998tracking}
M.~Herbster and M.K. Warmuth.
\newblock Tracking the best expert.
\newblock {\em Machine Learning}, 32(2):151--178, 1998.

\bibitem{freundopen}
Y.~Freund.
\newblock Private communication, 2000.
\newblock Also posted on http://www.learning-theory.org.

\bibitem{herbster2001tracking}
M.~Herbster and M.K. Warmuth.
\newblock Tracking the best linear predictor.
\newblock {\em Journal of Machine Learning Research}, 1:281--309, 2001.

\bibitem{ourJMLR15}
M.~Herbster, S.~Pasteris, and S.~Pontil.
\newblock Predicting a switching sequence of graph labelings.
\newblock {\em Journal of Machine Learning Research}, 16:2003--2022, 2015.

\bibitem{wacky}
A.~Gy{\"o}rgy, T.~Linder, and G.~Lugosi.
\newblock Tracking the best of many experts.
\newblock In {\em Proceedings 18th Annual Conference on Learning Theory}, pages
  204--216, 2005.

\bibitem{koolen2010freezing}
Wouter~M. Koolen and Tim van Erven.
\newblock Freezing and sleeping: Tracking experts that learn by evolving past
  posteriors, 2010.

\bibitem{fixmir}
N.~Cesa-Bianchi, P.~Gaillard, G.~Lugosi, and G.~Stoltz.
\newblock Mirror descent meets fixed share (and feels no regret).
\newblock In {\em Advances in Neural Information Processing Systems 24}, pages
  989--997, 2012.

\bibitem{GLL12}
A.~{Gy{\"o}rgy}, T.~{Linder}, and G.~{Lugosi}.
\newblock Efficient tracking of large classes of experts.
\newblock {\em IEEE Transactions on Information Theory}, 58(11):6709--6725, Nov
  2012.

\bibitem{adamskiy2012putting}
Wouter~M. Koolen, Dmitry Adamskiy, and Manfred~K. Warmuth.
\newblock Putting bayes to sleep.
\newblock In {\em Proceedings of the 25th International Conference on Neural
  Information Processing Systems - Volume 1}, NIPS?12, page 135?143, Red Hook,
  NY, USA, 2012. Curran Associates Inc.

\bibitem{AKCV12}
D.~Adamskiy, W.~M. Koolen, A.~Chernov, and V.~Vovk.
\newblock A closer look at adaptive regret.
\newblock In {\em Proceedings of the 23rd International Conference on
  Algorithmic Learning Theory}, ALT'12, pages 290--304, 2012.

\bibitem{DGS15}
A.~Daniely, A.~Gonen, and S.~Shalev-Shwartz.
\newblock Strongly adaptive online learning.
\newblock In {\em Proceedings of the 32nd International Conference on
  International Conference on Machine Learning - Volume 37}, ICML'15, pages
  1405--1411, 2015.

\bibitem{alt2017growing}
Jaouad Mourtada and Odalric-Ambrym Maillard.
\newblock Efficient tracking of a growing number of experts.
\newblock In {\em Proceedings of the 28th International Conference on
  Algorithmic Learning Theory {(ALT)}}, volume~76 of {\em Proceedings of
  Machine Learning Research}, pages 517--539, 2017.

\bibitem{BDS19}
Maria{-}Florina Balcan, Travis Dick, and Dravyansh Sharma.
\newblock Online optimization of piecewise lipschitz functions in changing
  environments.
\newblock {\em CoRR}, abs/1907.09137, 2019.

\bibitem{ZLDW19}
Kai Zheng, Haipeng Luo, Ilias Diakonikolas, and Liwei Wang.
\newblock Equipping experts/bandits with long-term memory.
\newblock In H.~Wallach, H.~Larochelle, A.~Beygelzimer, F.~d.~Alch\'{e}-Buc,
  E.~Fox, and R.~Garnett, editors, {\em Advances in Neural Information
  Processing Systems 32}, pages 5929--5939. Curran Associates, Inc., 2019.

\bibitem{MC89}
Michael Mccloskey and Neil~J. Cohen.
\newblock Catastrophic interference in connectionist networks: {T}he sequential
  learning problem.
\newblock {\em The Psychology of Learning and Motivation}, 24:104--169, 1989.

\bibitem{F99}
Robert~M. French.
\newblock Catastrophic forgetting in connectionist networks.
\newblock {\em Trends in Cognitive Sciences}, 3(4):128 -- 135, 1999.

\bibitem{onlinekernellearning}
J.~{Kivinen}, A.~J. {Smola}, and R.~C. {Williamson}.
\newblock {Online learning with kernels}.
\newblock {\em IEEE Transactions on Signal Processing}, 52:2165--2176, 2004.

\bibitem{budget}
N.~Cesa-Bianchi and C.~Gentile.
\newblock Tracking the best hyperplane with a simple budget perceptron.
\newblock In {\em Proceedings of the 18th Conference on Learning Theory}, pages
  483--498, 2006.

\bibitem{B95}
Jonathan Baxter.
\newblock Learning internal representations.
\newblock In {\em Proceedings of the Eighth Annual Conference on Computational
  Learning Theory}, COLT ?95, page 311?320, New York, NY, USA, 1995.
  Association for Computing Machinery.

\bibitem{C97}
Rich Caruana.
\newblock Multitask learning.
\newblock {\em Machine Learning}, 28(1):41--75, 1997.

\bibitem{EP04}
Theodoros Evgeniou and Massimiliano Pontil.
\newblock Regularized multi--task learning.
\newblock In {\em Proceedings of the Tenth ACM SIGKDD International Conference
  on Knowledge Discovery and Data Mining}, pages 109--117, New York, NY, USA,
  2004. Association for Computing Machinery.

\bibitem{Abernethy07}
J.~Abernethy, P.~Bartlett, and A.~Rakhlin.
\newblock Multitask learning with expert advice.
\newblock In {\em Proceedings 20th Annual Conference on Learning Theory}, pages
  484--498, 2007.

\bibitem{ARB08}
Alekh Agarwal, Alexander Rakhlin, and Peter Bartlett.
\newblock Matrix regularization techniques for online multitask learning.
\newblock Technical Report UCB/EECS-2008-138, EECS Department, University of
  California, Berkeley, Oct 2008.

\bibitem{SahaRDV11}
S.~Avishek, R.~Piyush, H.~Daum{\'{e}} III, and S.~Venkatasubramanian.
\newblock Online learning of multiple tasks and their relationships.
\newblock In {\em Proceedings of the 14th International Conference on
  Artificial Intelligence and Statistics}, pages 643--651, 2011.

\bibitem{RAB07}
Alexander Rakhlin, Jacob~D. Abernethy, and Peter~L. Bartlett.
\newblock Online discovery of similarity mappings.
\newblock In Zoubin Ghahramani, editor, {\em Machine Learning, Proceedings of
  the Twenty-Fourth International Conference {(ICML} 2007), Corvallis, Oregon,
  USA, June 20-24, 2007}, volume 227 of {\em {ACM} International Conference
  Proceeding Series}, pages 767--774. {ACM}, 2007.

\bibitem{LPS09}
G{\'{a}}bor Lugosi, Omiros Papaspiliopoulos, and Gilles Stoltz.
\newblock Online multi-task learning with hard constraints.
\newblock In {\em {COLT} 2009 - The 22nd Conference on Learning Theory,
  Montreal, Quebec, Canada, June 18-21, 2009}, 2009.

\bibitem{dekel2007online}
O.~Dekel, P.M. Long, and Y.~Singer.
\newblock Online learning of multiple tasks with a shared loss.
\newblock {\em Journal of Machine Learning Research}, 8(10):2233--2264, 2007.

\bibitem{Cavallanti}
G.~Cavallanti, N.~Cesa-Bianchi, and C.~Gentile.
\newblock Linear algorithms for online multitask classification.
\newblock {\em Journal of Machine Learning Research}, 1:2901--2934, 2010.

\bibitem{hstk17}
Christoph Hirnschall, Adish Singla, Sebastian Tschiatschek, and Andreas Krause.
\newblock Coordinated online learning with applications to learning user
  preferences, 2017.

\bibitem{FS97}
Yoav Freund and Robert~E Schapire.
\newblock A decision-theoretic generalization of on-line learning and an
  application to boosting.
\newblock {\em J. Comput. Syst. Sci.}, 55(1):119?139, August 1997.

\bibitem{HPT19}
Mark Herbster, Stephen Pasteris, and Lisa Tse.
\newblock Online matrix completion with side information.
\newblock {\em arXiv preprint arXiv:1906.07255}, 2019.

\bibitem{tsuda2005matrix}
K.~Tsuda, G.~R{\"a}tsch, and M.K. Warmuth.
\newblock Matrix exponentiated gradient updates for on-line learning and
  bregman projection.
\newblock {\em Journal of Machine Learning Research}, 6:995--1018, 2005.

\bibitem{CMSM07}
N.~Linial, S.~Mendelson, G.~Schechtman, and A.~Shraibman.
\newblock Complexity measures of sign matrices.
\newblock {\em Combinatorica}, 27(4):439--463, 2007.

\bibitem{LRSTS10}
Jason~D Lee, Ben Recht, Nathan Srebro, Joel Tropp, and Russ~R Salakhutdinov.
\newblock Practical large-scale optimization for max-norm regularization.
\newblock In J.~D. Lafferty, C.~K.~I. Williams, J.~Shawe-Taylor, R.~S. Zemel,
  and A.~Culotta, editors, {\em Advances in Neural Information Processing
  Systems 23}, pages 1297--1305. Curran Associates, Inc., 2010.

\bibitem{herbster2009online}
M.~Herbster, G.~Lever, and M.~Pontil.
\newblock Online prediction on large diameter graphs.
\newblock In {\em Advances in Neural Information Processing Systems 21}, pages
  649--656, 2008.

\bibitem{OP16}
Francesco Orabona and D\'{a}vid P\'{a}l.
\newblock Coin betting and parameter-free online learning.
\newblock In {\em Proceedings of the 30th International Conference on Neural
  Information Processing Systems}, NIPS16, pages 577--585, Red Hook, NY, USA,
  2016. Curran Associates Inc.

\bibitem{DSS18}
Ofer Dekel, Shai Shalev-shwartz, and Yoram Singer.
\newblock The forgetron: A kernel-based perceptron on a fixed budget.
\newblock In Y.~Weiss, B.~Sch\"{o}lkopf, and J.~C. Platt, editors, {\em
  Advances in Neural Information Processing Systems 18}, pages 259--266. MIT
  Press, 2006.

\bibitem{CCG07}
Giovanni Cavallanti, Nicol{\`{o}} Cesa{-}Bianchi, and Claudio Gentile.
\newblock Tracking the best hyperplane with a simple budget perceptron.
\newblock {\em Mach. Learn.}, 69(2-3):143--167, 2007.

\bibitem{CDST99}
Yair Carmon, John~C. Duchi, Aaron Sidford, and Kevin Tian.
\newblock A rank-1 sketch for matrix multiplicative weights.
\newblock In Alina Beygelzimer and Daniel Hsu, editors, {\em Conference on
  Learning Theory, {COLT} 2019, 25-28 June 2019, Phoenix, AZ, {USA}}, volume~99
  of {\em Proceedings of Machine Learning Research}, pages 589--623. {PMLR},
  2019.

\bibitem{FSSW97}
Yoav Freund, Robert~E. Schapire, Yoram Singer, and Manfred~K. Warmuth.
\newblock Using and combining predictors that specialize.
\newblock In {\em Proceedings of the Twenty-Ninth Annual ACM Symposium on
  Theory of Computing}, STOC 97, pages 334--343, New York, NY, USA, 1997.
  Association for Computing Machinery.

\bibitem{herbster2006prediction}
M.~Herbster and M.~Pontil.
\newblock Prediction on a graph with a perceptron.
\newblock In {\em Advances in Neural Information Processing Systems 19}, pages
  577--584, 2006.

\bibitem{Klein1993}
Douglas Klein and Milan Randic.
\newblock Resistance distance.
\newblock {\em Journal of Mathematical Chemistry}, 12:81--95, 12 1993.

\bibitem{herbster2009fast}
M.~Herbster, M.~Pontil, and S.~Rojas-Galeano.
\newblock Fast prediction on a tree.
\newblock In {\em Advances in Neural Information Processing Systems}, pages
  657--664, 2009.

\bibitem{Novikoff62}
A.B. Novikoff.
\newblock On convergence proofs on perceptrons.
\newblock In {\em Proceedings of the Symposium on the Mathematical Theory of
  Automata}, pages 615--622, 1962.

\bibitem{Forster2001}
J.~Forster, N.~Schmitt, and H.U. Simon.
\newblock Estimating the optimal margins of embeddings in euclidean half
  spaces.
\newblock In {\em Proceedings Computational Learning Theory}, pages 402--415,
  2001.

\bibitem{litt88}
N.~Littlestone.
\newblock Learning quickly when irrelevant attributes abound: A new
  linear-threshold algorithm.
\newblock {\em Machine Learning}, 2:285--318, April 1988.

\bibitem{BPS09}
Shai Ben{-}David, D{\'{a}}vid P{\'{a}}l, and Shai Shalev{-}Shwartz.
\newblock Agnostic online learning.
\newblock In {\em {COLT} 2009 - The 22nd Conference on Learning Theory,
  Montreal, Quebec, Canada, June 18-21, 2009}, 2009.

\bibitem{Cristianini2000}
Nello Cristianini and John Shawe-Taylor.
\newblock {\em An Introduction to Support Vector Machines and Other
  Kernel-based Learning Methods}.
\newblock Cambridge University Press, 2000.

\bibitem{Zinkevich}
Martin Zinkevich.
\newblock {Online Convex Programming and Generalized Infinitesimal Gradient
  Ascent}.
\newblock In {\em Proceedings, Twentieth International Conference on Machine
  Learning}, volume~2, pages 928--935, 2003.

\bibitem{CLW96}
N.~Cesa-Bianchi, P.~M. Long, and M.~K. Warmuth.
\newblock Worst-case quadratic loss bounds for on-line prediction of linear
  functions by gradient descent.
\newblock {\em IEEE Transactions on Neural Networks}, 7(3):604--619, 1996.
\newblock Earlier version in 6th COLT, 1993.

\end{thebibliography}
}
\newpage
\appendix
\newcommand{\musgt}{{\cop{}(\was{\gti})}}
\newcommand{\copgti}{{\bar{u}^{\tau}}}

\section{Proofs for Section~\ref{sec:fin}}
\label{sec:proof_fin}
In this section, we  prove Theorem~\ref{thm:fin}. In doing so, we will create and analyze two algorithms that are more generic and define reductions between them. Our analysis builds on the results of~\cite{\bayessleep,\specials,\hedge}. 
We use the {\em Allocation} model and the {\em Hedge} algorithm from~\cite{\hedge}.  We adapt the Allocation model to the {\em Specialist} setting from~\cite{\specials}.  Finally, we define a set of specialists called {\em circadians}, a concept introduced in~\cite{\bayessleep}, to enable multitask switching in the Allocation model.  In~\cite{\bayessleep} sets of circadians were defined to solve the problem of switching with long-term memory and multitask learning without switching.  We have generalized those previous two models into a switching multitask model with long-term memory.  

This section is organized as follows.
In Section~\ref{mapss}, we introduce the {\em Multitask Allocation} model (cf. Figure~\ref{fig:map}).
In Section~\ref{sapss}, we introduce the {\em Specialist Allocation} model (cf. Figure~\ref{fig:sap}) and the {\em Specialist Hedge} algorithm. In Section~\ref{ourspecsec}, we define a set of specialists called circadians which will enable us to reduce Multitask Allocation (with switching) to Specialist Allocation.  In Section~\ref{eqsec}, we prove that linear-time Algorithm~\ref{alg:fin} is equivalent (in that the losses are the same) to the Specialist Hedge algorithm with the exponential set of circadians from Section~\ref{ourspecsec}.  Given the equivalence of the algorithms, in the remainder of the section we turn to the problem of proving the regret bound.

In Section~\ref{Shcsec}, we introduce a new set of specialists (shortened circadians) which enables 
an improved regret analysis.  In Section~\ref{cdss} we define a comparator distribution for the Specialist Allocation model from a comparator sequence of experts for the Multitask Allocation model with long-term memory. In Section~\ref{iress} we quantify a term in the analysis of the Specialist Hedge algorithm and use this to derive a regret bound for the Multitask Allocation model with switching in Section~\ref{marss}. Finally, in 
Section~\ref{redss} we reduce the finite hypothesis class setting to the  allocation setting, giving us the regret bound of Theorem~\ref{thm:fin}.

We introduce the following notation for this section.   We define $\str$ to be the set of all $(\li,\ti)$ such that $\li\in\na{\nL}$ and $\ti\in\na{\ltu{\li}+1}$.  Given sets $X$, $Y$ and $Z$ with $Y\subseteq X$, and a function $\mu:X\rightarrow Z$ we define the function $\rtr{\mu}{Y}:Y\rightarrow Z$ such that $\rtr{\mu}{Y}(x):=\mu(x)$ for all $x\in Y$.   Given a set $\mathcal{W}$ and a function $u:\mathcal{W}\into \Re$ let $u(\mathcal{W}) = \sum_{w\in \mathcal{W}} u(\mathcal{W})$.
For functions $f,g:\hye\rightarrow\mathbb{R}^+$ with $\sum_{\hi\in\hye}f(\hi)=\sum_{\hi\in\hye}g(\hi)=1$, we let $\rel{f}{g}$ be their relative entropy:
$$\rel{f}{g}:=\sum_{\hi\in\hye}f(\hi)\log\left(\frac{f(\hi)}{g(\hi)}\right)\,.$$ 
Given functions $a,b:\mathcal{W}\rightarrow\mathbb{R}$, for some set $\mathcal{W}$, we define $\inprod{a}{b}:=\sum_{\hi\in\mathcal{W}}a(\hi)b(\hi).$
\subsection{Multitask Allocation Model}\label{mapss}

We will consider the following generalisation of the finite hypothesis class setting to the {\em Multitask Allocation} model. This is described in Figure \ref{fig:map}. The Multitask Allocation model is a multitask version of the Hedge algorithm in \cite{\hedge}. We refer to the elements of $\na{\nex}$ as ``experts''. On each trial $\tau$, Nature chooses a task $\ltau\in[\nL]$ which is revealed to Learner. Learner then chooses a vector $\bal{\gti}\in\simp$ and Nature chooses a vector $\blo{\gti}\in\lset$. Finally, $\blo{\gti}$ is revealed to Learner, who then incurs a loss $\lot{\gti}:=\bal{\gti}\cdot\blo{\gti}$. In this section we will create an algorithm for Learner, and then reduce the finite hypothesis class setting to the Multitask Allocation model, which gives us  Algorithm~\ref{alg:fin}. 

For all $\gti\in T$, we define $\tbc{\gltf{\gti}}{\liv{\gti}}:=\lot{\gti}$ and $\blc{\gltf{\gti}}{\liv{\gti}}=\blo{\gti}$, and as in the finite hypothesis class setting, we define $\gltF(\tau) := \sum_{j=1}^{\tau} [\ell^j = \ltau]$.

\begin{figure}
\begin{tabbing}
For \= $\gti=1$ \=to $T$ do\\
\>\> Receive task $\ltau \in [\nL]$. \\
\>\> Predict $\bal{\gti}\in\simp$.\\
\>\> Receive $\blo{\gti}\in\lset$.\\
\>\> Incur loss $\lot{\gti}:=\bal{\gti}\cdot\blo{\gti}$.\\
\end{tabbing}
\caption{The Multitask Allocation Model}\label{fig:map}
\end{figure}

\subsection{The Specialist Hedge Algorithm}\label{sapss}

Here we introduce the Specialist Hedge algorithm for the {\em Specialist Allocation} model. The Specialist Allocation model is given in Figure~\ref{fig:sap}. The model is based on a set $\hye$ of specialists. On trial $\gti$, Nature chooses a set $\was{\gti}$ of specialists that are ``awake'' (any specialist not in $\was{\gti}$ is said to be ``asleep''). The set $\was{\gti}$ is known to the Learner at the start of the trial. For every awake specialist $\hi\in\hye$, Nature then chooses a cost $\barc{\gti}{\hi}$ but does not reveal it to the Learner.  The Learner must then choose a probability distribution $\arv{\gti}$ over $\was{\gti}$. Finally, $\arc{\gti}$ is revealed to the Learner and the Learner incurs a loss which is the expected cost of a specialist drawn from $\arv{\gti}$. This framework in the log loss setting, in which we have awake and asleep specialists, was introduced in \cite{\specials}. 

\newcommand{\CBT}{\bar{C}^{\tau}}
\begin{figure}
\begin{tabbing}
For \= $\gti=1$ \=to $T$ do\\
\>\> Receive non-empty set $\was{\gti} \subseteq\hye$. \\
\>\> Predict a function $\arv{\gti}:\was{\gti}\rightarrow[0,1]$ with $\sum_{\hi\in\was{\gti}}\barv{\gti}{\hi}=1$.\\
\>\> Receive a function $\arc{\gti}:\was{\gti}\rightarrow[0,1]$.\\
\>\> Incur loss $\CBT := \inprod{\arv{\gti}}{\arc{\gti}}$.
\end{tabbing}
\caption{The Specialist Allocation Model}\label{fig:sap}
\end{figure}

The Specialist Hedge algorithm, which is inspired by the algorithms in~\cite{\specials} and~\cite{\hedge}, is then found in Algorithm~\ref{alg:sh}.

\begin{algorithm}\caption{Specialist Hedge}
\begin{algorithmic}
\renewcommand{\algorithmicrequire}{\textbf{Parameters:}} 
\REQUIRE Learning rate $\lr$; initial weights $\pv{1}:\hye\rightarrow[0,1]$ with $\sum_{\hi\in\hye}\pv{1}(\hi)=1$ \vspace{.1truecm}
\renewcommand{\algorithmicrequire}{\textbf{Initialization:}}
\REQUIRE $\bmv_1 = \frac{1}{n}\bm{1}^n$.  \vspace{.1truecm}
\renewcommand{\algorithmicrequire}{\textbf{For}}
\REQUIRE $\gti =1,\dots,T$  \vspace{.1truecm}
\STATE $\bullet$ Receive $\was{\gti} \subseteq \hye$.
\STATE $\bullet$ For $\hi\in\was{\gti}$ set: 
$$\barv{\gti}{\hi}\leftarrow\frac{\pv{\gti}(\hi)}{\sum_{\hi'\in\was{\gti}}\pv{\gti}(\hi')}.$$
\STATE  $\bullet$ Receive $\arc{\gti}\in [0,1]^{\was{\gti}}$.
\STATE  $\bullet$ For all $\hi\in\hye\setminus\was{\gti}$ set $\pv{\gti+1}(\hi)\leftarrow\pv{\gti}(\hi)$.
\STATE  $\bullet$ For all $\hi\in\was{\gti}$ set: 
\[ 
\pv{\gti+1}(\hi)\leftarrow\pv{\gti}(\hi)\exp(-\lr\barc{\gti}{\hi})\left(\frac{\sum_{\hi'\in\was{\gti}}\pv{\gti}(\hi')}{\sum_{\hi'\in\was{\gti}}\pv{\gti}(\hi')\exp(-\lr\barc{\gti}{\hi'})}\right)\,.
\]
\end{algorithmic}
\label{alg:sh}
\end{algorithm}

We have the following theorem.

\begin{theorem}\label{allospecth}
Let $\cop{}:\hye\rightarrow[0,1]$ be any function such that $\sum_{\hi\in\hye}\cop{}(\hi)=1$. On any trial $\gti$ we define $\musgt:=\sum_{\hi\in\was{\gti}}\cop{}(\hi)$ and for all $\hi\in\was{\gti}$ define $\copgti(\hi):=\cop{}(\hi)/\musgt$. Then we have:
\[
\sum_{\gti=1}^{T}\musgt\inprod{\arv{\gti}-\copgti}{\arc{\gti}}\leq\frac{1}{\lr}\rel{\cop{}}{\pv{1}}+\frac{\lr}{2}\sum_{\gti=1}^T\musgt\,.
\]
\end{theorem}
If
$$\lr=\sqrt{\frac{2\rel{\cop{}}{\pv{1}}}{\sum_{\gti=1}^T\musgt}}$$
then
$$\sum_{\gti=1}^{T}\musgt\inprod{\arv{\gti}-\copgti}{\arc{\gti}}\leq\sqrt{2\rel{\cop{}}{\pv{1}}\sum_{\gti=1}^T\musgt}\,.$$

\begin{proof}
The proof is similar to that in \cite{\specials}, and utilizes the allocation model introduced in~\cite{\hedge}.

We first show, by induction on $\gti$, that for all $\gti\in\na{T+1}$ we have $\sum_{\hi\in\hye}\pv{\gti}(\hi)=1$. This is clearly true for $\gti=1$. Now suppose that it is true for $\gti=\gti'$ for some $\gti'\in\na{T}$. We now show that it is true for $\gti=\gti'+1$.
\begin{align}
\sum_{\hi\in\hye}\pv{\gti'+1}(\hi)&=\sum_{\hi\in\hye\setminus\was{\gti'}}\pv{\gti'+1}(\hi)+\sum_{\hi\in\was{\gti}}\pv{\gti'+1}(\hi)\\
&=\sum_{\hi\in\hye\setminus\was{\gti'}}\pv{\gti'}(\hi)+\sum_{\hi\in\was{\gti'}}\pv{\gti'+1}(\hi)\\
&=\sum_{\hi\in\hye\setminus\was{\gti'}}\pv{\gti'}(\hi)+\sum_{\hi\in\was{\gti'}}\frac{\pv{\gti'}(\hi)\exp(-\lr\barc{\gti'}{\hi})\sum_{\hi'\in\was{\gti'}}\pv{\gti}(\hi')}{\sum_{\hi'\in\was{\gti'}}\pv{\gti'}(\hi')\exp(-\lr\barc{\gti'}{\hi'})}\\
&=\sum_{\hi\in\hye\setminus\was{\gti'}}\pv{\gti'}(\hi)+\frac{\sum_{\hi\in\was{\gti'}}\pv{\gti'}(\hi)\exp(-\lr\barc{\gti'}{\hi})\sum_{\hi'\in\was{\gti'}}\pv{\gti}(\hi')}{\sum_{\hi'\in\was{\gti'}}\pv{\gti'}(\hi')\exp(-\lr\barc{\gti'}{\hi'})}\\
&=\sum_{\hi\in\hye\setminus\was{\gti'}}\pv{\gti'}(\hi)+\sum_{\hi\in\was{\gti'}}\pv{\gti'}(\hi)\\
&=\sum_{\hi\in\hye}\pv{\gti'}(\hi)\\
&=1
\end{align}

With this in hand, we now have that $\rel{\cop{}}{\pv{\gti}}$ is defined and positive for all $\gti\in\na{T}$.

Let $\zed^{\gti}:=\sum_{\hi'\in\was{\gti}}\barv{\gti}{\hi'}\exp(-\lr\barc{\gti}{\hi'})$. By definition of $\arv{\gti}$ we have:
\begin{equation}\label{progleq}
\zed^{\gti}=\frac{\sum_{\hi'\in\was{\gti}}\pv{\gti}(\hi')\exp(-\lr\barc{\gti}{\hi'})}{\sum_{\hi'\in\was{\gti}}\pv{\gti}(\hi')}
\end{equation} so:
\begin{align}
\notag&\rel{\cop{}}{\pv{\gti}}-\rel{\cop{}}{\pv{\gti+1}}\\
\notag=&\sum_{\hi\in\hye}\cop{}(\hi)\left(\log\left(\frac{\cop{}(\hi)}{\pv{\gti}(\hi)}\right)-\log\left(\frac{\cop{}(\hi)}{\pv{\gti+1}(\hi)}\right)\right)\\
\notag=&\sum_{\hi\in\hye}\cop{}(\hi)\log\left(\frac{\pv{\gti+1}(\hi)}{\pv{\gti}(\hi)}\right)\\
\label{progl1}=&\sum_{\hi\in\was{\gti}}\cop{}(\hi)\log\left(\frac{\pv{\gti+1}(\hi)}{\pv{\gti}(\hi)}\right)\\
\notag=&\musgt\sum_{\hi\in\was{\gti}}\frac{\cop{}(\hi)}{\musgt}\log\left(\frac{\pv{\gti+1}(\hi)}{\pv{\gti}(\hi)}\right)\\
\notag=&\musgt\sum_{\hi\in\was{\gti}}\copgti(\hi)\log\left(\frac{\pv{\gti+1}(\hi)}{\pv{\gti}(\hi)}\right)\\
\label{progl2}=&\musgt\sum_{\hi\in\was{\gti}}\copgti(\hi)\log\left(\frac{\exp(-\lr\barc{\gti}{\hi})\sum_{\hi'\in\was{\gti}}\pv{\gti}(\hi')}{\sum_{\hi'\in\was{\gti}}\pv{\gti}(\hi')\exp(-\lr\barc{\gti}{\hi'})}\right)\\
\label{progl3}=&\musgt\sum_{\hi\in\was{\gti}}\copgti(\hi)\log\left(\frac{\exp(-\lr\barc{\gti}{\hi})}{\zed^{\gti}}\right)\\
\notag=&-\musgt\lr\left(\sum_{\hi\in\was{\gti}}\copgti(\hi)\barc{\gti}{\hi}\right)-\musgt\log(\zed^{\gti})\left(\sum_{\hi\in\was{\gti}}\copgti(\hi)\right)\\
\notag=&-\musgt\lr\inprod{\copgti}{\arc{\gti}}-\musgt\log(\zed^{\gti})\\
\notag=&-\musgt\lr\inprod{\copgti}{\arc{\gti}}-\musgt\log\left(\sum_{\hi\in\was{\gti}}\barv{\gti}{\hi}\exp(-\lr\barc{\gti}{\hi})\right)\\
\label{progl4}\geq&-\musgt\lr\inprod{\copgti}{\arc{\gti}}-\musgt\log\left(\sum_{\hi\in\was{\gti}}\barv{\gti}{\hi}\left(1-\lr\barc{\gti}{\hi}+\frac{1}{2}\lr^2\left(\barc{\gti}{\hi}\right)^2\right)\right)\\
\notag=&-\musgt\lr\inprod{\copgti}{\arc{\gti}}-\musgt\log\left(1-\lr\inprod{\arv{\gti}}{\arc{\gti}}+\frac{1}{2}\lr^2\sum_{\hi\in\was{\gti}}\barv{\gti}{\hi}\left(\barc{\gti}{\hi}\right)^2\right)\\
\notag=&-\musgt\lr\inprod{\copgti}{\arc{\gti}}-\musgt\log\left(1-\lr\inprod{\arv{\gti}}{\arc{\gti}}+\frac{1}{2}\lr^2\sum_{\hi\in\was{\gti}}\barv{\gti}{\hi}\left(\barc{\gti}{\hi}\right)^2\right)\\
\notag\geq&-\musgt\lr\inprod{\copgti}{\arc{\gti}}-\musgt\log\left(1-\lr\inprod{\arv{\gti}}{\arc{\gti}}+\frac{1}{2}\lr^2\sum_{\hi\was{\gti}}\barv{\gti}{\hi}\right)\\
\notag\geq&-\musgt\lr\inprod{\copgti}{\arc{\gti}}-\musgt\log\left(1-\lr\inprod{\arv{\gti}}{\arc{\gti}}+\frac{1}{2}\lr^2\right)\\
\label{progl5}\geq&-\musgt\lr\inprod{\copgti}{\arc{\gti}}+\musgt\lr\inprod{\arv{\gti}}{\arc{\gti}}-\musgt\frac{1}{2}\lr^2\\
=&\musgt\lr\inprod{\arv{\gti}-\copgti}{\arc{\gti}}-\musgt\frac{1}{2}\lr^2
\end{align}
where Equation \eqref{progl1} comes from the fact that if $\hi\in\hye\setminus\was{\gti}$ then $\pv{\gti+1}(\hi)=\pv{\gti}(\hi)$ so $\log(\pv{\gti+1}(\hi)/\pv{\gti}(\hi))=0$, Equation \eqref{progl2} comes from the update of $\pv{\gti}(\hi)$ to $\pv{\gti+1}(\hi)$ when $\hi\in\was{\gti}$, Equation \eqref{progl3} comes from Equation \eqref{progleq}, Equation \eqref{progl4} comes from the inequality $\exp(x)\leq 1-x+x^2/2$ for $x\geq 0$, and Equation \eqref{progl5} comes from the inequality $\log(1+x)\leq x$.

A telescoping sum then gives us:
\begin{align}
\rel{\cop{1}}{\pv{\gti}}&\geq\rel{\cop{1}}{\pv{\gti}}-\rel{\cop{}}{\pv{T+1}}\\
&=\sum_{\gti\in\na{\ntr}}\rel{\cop{}}{\pv{\gti}}-\rel{\cop{}}{\pv{\gti+1}}\\
&=\sum_{\gti\in\na{\ntr}}\left(\musgt\lr\inprod{\arv{\gti}-\copgti}{\arc{\gti}}-\musgt\frac{1}{2}\lr^2\right).
\end{align}
Dividing by $\lr$ and rearranging then gives us the result.
\end{proof}

\subsection{{\sc MA} Specialists}\label{ourspecsec}

Here we reduce the Multitask Allocation model (cf. Figure~\ref{fig:map}) with switching to the Specialist Allocation model (cf. Figure~\ref{fig:sap}). Specifically we define $\hye$, $\was{\gti}$, $\pv{1}$ and $\arc{\gti}$ in Figure \ref{fig:sap} and $\pv{1}$ in Algorithm \ref{alg:sh} from $\nex$, $\liv{\gti}$ and $\blo{\gti}$ in Figure \ref{fig:map}. 

Our reduction is similar to that in \cite{adamskiy2012putting}, which reduced a single task switching model 
to the specialist model of~\cite{\specials}. In~\cite{adamskiy2012putting}, the specialists were expert-circadian pairs, where a circadian was an instance of a Markov chain over the trials: with states ``awake'' and ``asleep''. A specialist predicted the same as its corresponding expert.  A specialist was awake on a trial if and only if the state of its circadian is ``awake'' on that trial. The initial weight of a specialist was proportional to its probability as an instance of the Markov chain. In our reduction, we again define experts as expert-circadian pairs, except this time a circadian is an instance of a set of $\nL$ independent Markov chains: one Markov chain for each task. The Markov chain corresponding to a task $\li$ runs over the trials $t$ in which $\gltf{t}=\li$, in order.

We first characterize the set of all trials as follows:

 \begin{definition}
We define $\str$ to be the set of all $(\li,\ti)$ such that $\li\in\na{\nL}$ and $\ti\in\na{\ltu{\li}+1}$.
\end{definition}

We now define the transition matrix of the Markov chains.

\begin{definition}\label{trmd}
We define:
\begin{itemize}
\item $\trn{1}:=\sip$~~~~~~~~~$\trn{0}:=1-\sip$
\item $\trm{1}{1}:=\ww$~~~~~~~$\trm{1}{0}:=1-\ww$
\item $\trm{0}{1}:=\sw$~~~~~~~$\trm{0}{0}:=1-\sw$
\end{itemize}
where $\sw$, $\ww$ and $\sip$ are defined in the description of  Algorithm \ref{alg:fin}.
\end{definition}

Now we define the notions of circadians, and their weights,

\begin{definition}\label{cirdef}
A {\em circadian} is a function from $\str$ into $\{0,1\}$. Let $\crs$ be the set of all circadians. Given a circadian $\ci\in\crs$ its weight $\crw{\ci}$ is defined as:
$$\crw{\ci}:=\prod_{\li\in\na{\nL}}\trn{\ci(\li,1)}\prod_{\ti\in\na{\ltu{\li}}}\trm{\ci(\li,\ti)}{\ci(\li,\ti+1)}\,.$$
\end{definition}

Finally, we define the multitask allocation ({\sc MA}) specialists, their initial weights and the sets $\{\was{\gti}:\gti\in\na{T}\}$

\begin{definition}
Each specialist is a pair $(\ei,\ci)$ where $\ei$ is an expert and $\ci$ is a circadian. i.e:
$$\hye=[\nex]\times\crs\,.$$
Each specialist $(\ei,\ci)$ has an initial weight: 
$$\pv{1}(\ei,\ci):=\frac{1}{\nex}\crw{\ci}\,.$$
On each trial $\gti\in\na{T}$ the set $\was{\gti}$, of awake specialists, is defined as:
$$\was{\gti}:=\{(\ei,\ci)\in\hye~|~\ci(\liv{\gti},\gltf{\gti})=1\}\,.$$
On each trial $\gti$, our cost function $\arc{\gti}:\was{\gti}\rightarrow[0,1]$ is defined as:
$$\barc{\gti}{(\ei,\ci)}:=\lo{\gti}{\ei}\,.$$
\end{definition}

\subsection{Equivalence of Algorithm \ref{alg:fin} to the Specialist Hedge Algorithm}\label{eqsec}
In this subsection we prove that Algorithm~\ref{alg:fin} and the Specialist Hedge algorithm are equivalent, given the reduction from the Multitask Allocation model to the Specialist Allocation model defined in Subsection \ref{ourspecsec}. Specifically, we prove the following theorem.
\begin{theorem}\label{equivth}
Apply the reduction in Subsection~\ref{ourspecsec} 
from the Multitask Allocation model (cf. Figure~\ref{fig:map}) with switching to the Specialist Allocation model (cf. Figure~\ref{fig:sap}). Suppose that for all trials $\gti$, we have that $\bal{\gti} (\blo{\gti})$ and $\arv{\gti} (\arc{\gti})$ are defined as in Algorithm~\ref{alg:fin} and the Specialist Hedge algorithm (Algorithm~\ref{alg:sh}).  Then for all $\gti\in\na{T}$:
$$\lot{\gti}=\bal{\gti}\cdot\blo{\gti}=\inprod{\arv{\gti}}{\arc{\gti}}=\CBT.$$
\end{theorem}

As in the multitask algorithm of \cite{adamskiy2012putting}, the information learned (up to some trial) by the specialist algorithm can be represented by a set of $\na{\nex}$-dimensional vectors: one ``global'' vector and, for each task, a ``local'' vector. However, our global and local vector do not correspond exactly to those of \cite{adamskiy2012putting} in the special case (of our algorithm) of no switching. Like in \cite{\bayessleep}, the prediction and update of the specialist algorithm, on a particular trial $\gti$, depend only on the global vector and the local vector of task $\gltf{\gti}$. Since we have to analyze many Markov chains simultaneously, our proof of equivalence is considerably more involved than the proofs of either of the algorithms in~\cite{adamskiy2012putting}.

In the proof of Theorem \ref{equivth} we shall, for all trials $\gti\in\na{T}$, let $\pv{\gti}:\hye\rightarrow[0,1]$ be defined as in the Specialist Allocation algorithm.

We will need a way to refer to the current local time of a particular task. We will do this via the following function.
\begin{definition}
For all $\gti\in\na{\gtu+1}$ and learners $\li\in\na{\nL}$ we define:
$$\nxt{\gti}{\li}:=\gltf{\min\{\gti'\geq\gti:\gti'\in\na{\gtu},~\liv{\gti'}=\li\}}$$
or, if $\{\gti'\geq\gti:\gti'\in\na{\gtu},~\liv{\gti'}=\li\}$ is empty, we define $\nxt{\gti}{\li}:=\ltu{\li}+1$.
\end{definition}

Given an expert $\ei$, trial $\gti$ and $\benl\in\{0,1\}^{\nta}$ we will now define $\qf{\gti}{\ei}(\benl)$ as the sum, over $\ci$, of the weights of specialists $(\ci,\ei)$ in which, given any task $\li$, the value of $\ci$ on that task, and at its current local time, is equal to $\enl{\li}$. The formal definition of $\qf{\gti}{\ei}$ is as follows.

\begin{definition}
Given an expert $\ei\in\na{\nex}$ and trial $\gti\in\na{\gtu}$ we define the function $\qf{\gti}{\ei}:\{0,1\}^{\nta}\rightarrow[0,1]$ by:
$$\qf{\gti}{\ei}(\benl)=\sum_{\ci\in\crs}\pv{\gti}(\ei,\ci)\idf{\forall \li\in\na{\nL}~,~\ci(\li,\nxt{\gti}{\li})=\enl{\li}}$$
\end{definition}

We now define what we call {\em circadian tails.} Given a  trial $\gti$, a circadian tail is a truncated circadian: truncated so that it is only defined for future trials.

\begin{definition}
For all $\gti\in\na{\gtu}$ we define the respective {\em tail set} $\strt{\gti}$ as the set of all $(\li,\ti)$ such that $\li\in\na{\nL}$ and $\nxt{\gti}{\li}\leq\ti\leq\ltu{\li}+1$.

For all trials $\gti\in\na{\gtu}$ a respective {\em circadian tail} is a function from $\strt{\gti}$ into $\{0,1\}$. We define $\crst{\gti}$ to be the set of respective circadian tails of $\gti$.

Given a trial $\gti\in\na{\gtu}$ and a respective circadian tail $\ci\in\crst{\gti}$ we define its ``weight'' as:
$$\trw{\ci}:=\prod_{\li\in\na{\nL}}\prod_{\ti=\nxt{\gti}{\li}}^{\ltu{\li}}\trm{\ci(\li,\ti)}{\ci(\li,\ti+1)}\,.$$

Given a trial $\gti\in\na{\gtu}$ and a respective circadian tail $\ci\in\crst{\gti}$ we define its ``start'' as the vector $\bsrt{\ci}\in\{0,1\}^{\nta}$ defined by:
$$\srt{\ci}{\li}=\ci(\li,\nxt{\gti}{\li})\,.$$
\end{definition}

We now define, for every trial, a ``renormalisation constant'' $\nc{\gti}$. Awake experts are multiplied by $\nc{\gti}$ after the initial ``Hedge''-like update on trial $\gti$.

\begin{definition}
For all trials $\gti\in\na{\gtu}$ we define: \[\nc{\gti}:=\frac{\sum_{(\ei,\ci)\in\was{\gti}}\pv{\gti}(\ei,\ci)}{\sum_{(\ei',\ci')\in\was{\gti}}\pv{\gti}(\ei',\ci')\exp(-\lr\lo{\gti}{\ei'})}\,.\]
\end{definition}

We will utilize the following equalities throughout the proof.

\begin{lemma}\label{nxtl}
For all trials $\gti\in\na{\gtu}$ we have:
$$\nxt{\gti+1}{\li}=\nxt{\gti}{\li}+1=\gltf{\gti}+1~~~\operatorname{if}~\li=\liv{\gti}\,,$$
$$\nxt{\gti+1}{\li}=\nxt{\gti}{\li}~~~\operatorname{if}~\li\neq\liv{\gti}\,.$$
This then implies that:
$$\strt{\gti}=\strt{\gti+1}\cup\{(\liv{\gti},\gltf{\gti})\}\,.$$
\end{lemma}

\begin{proof}
Immediate.
\end{proof}

Given $\benl\in\{0,1\}^{\nta}$ we now define $\bsbl{\gti}{j}$ as equal to $\benl$ except for the $\liv{\gti}$-th component, which is equal to $j$.
\begin{definition}
Given $j\in\{0,1\}$ and $\benl\in\{0,1\}^{\nta}$ we define $\bsbl{\gti}{j}$ by:
\[ \sbl{\gti}{j}{\li}:=\begin{cases}
j & \li=\liv{\gti} \\
\enl{\li} & \li\neq\liv{\gti}
\end{cases}\quad (i\in [\nta])\,.
\]
\end{definition}

The next lemma shows that the weights of circadian tails (respective to a given trial) which have the same starting values sum to one.

\begin{lemma}\label{eq1l}
Given  $\gti\in\na{\gtu+1}$ and $\benl\in\{0,1\}^{\nta}$ we have:
$$\sum_{\ci\in\crst{\gti}}\idf{\bsrt{\ci}=\benl}\trw{\ci}=1$$
\end{lemma}

\begin{proof}
We prove by reverse induction on $\gti$ (i.e. from $\gti=\gtu+1$ to $\gti=1$).

In the case that $\gti=\gtu+1$ we have, for all $\li\in\na{\nL}$ that $\nxt{\gtu}{\li}=\ltu{\li}+1$. This implies that $\strt{\gtu}$ is the set of all $(\li,\ti)$ such that $\li\in\na{\nL}$ and $\ti=\ltu{\li}+1$. This means that there is a unique $\ci\in\crst{\gtu}$ with $\bsrt{\ci}=\benl$. By definition of $\trw{\cdot}$ we have, for all $\ci\in\crst{\gtu}$, that $\trw{\ci}=1$. Putting together gives us the result.

Now suppose it holds for $\gti=\gti'+1$ for some $\gti'\in\na{\gtu-1}$. We shall now show that it holds for $\gti=\gti'$ which will complete the proof. By Lemma \ref{nxtl} we, given $\ci'\in\crst{\gti'+1}$, have that there exists $\ci\in\crst{\gti'}$ with $\bsrt{\ci}=\benl$ and $\rtr{\ci}{\strt{\gti'+1}}=\ci'$ if and only if $\bsrt{\ci'}=\bsbl{\gti'}{j}$ for some $j\in\{0,1\}$. Furthermore, given $\ci'\in\crst{\gti'+1}$, such an $\ci$ is unique.

Given $j\in\{0,1\}$, $\ci'\in\crst{\gti'+1}$, and $\ci\in\crst{\gti'}$ with $\rtr{\ci}{\strt{\gti'+1}}=\ci'$, $\bsrt{\ci}=\benl$ and $\bsrt{\ci'}=\bsbl{\gti'}{j}$, we have, from Lemma \ref{nxtl} and definition of $\trw{\cdot}$, that 
\begin{align}
\notag\trw{\ci}&=\trm{\ci(\liv{\gti'},\gltf{\gti'})}{\ci(\liv{\gti'},\gltf{\gti'+1})}\trw{\ci'}\\
\notag&=\trm{\enl{\liv{\gti'}}}{j}\trw{\ci'}
\end{align}
Putting together gives us:
\begin{align}
\notag\sum_{\ci\in\crst{\gti'}}\idf{\bsrt{\ci}=\benl}\trw{\ci}&=\sum_{j\in\{0,1\}}\sum_{\ci'\in\crst{\gti'+1}}\idf{\bsrt{\ci'}=\bsbl{\gti'}{j}}\trm{\enl{\liv{\gti'}}}{j}\trw{\ci'}\\
\notag&=\sum_{j\in\{0,1\}}\trm{\enl{\liv{\gti'}}}{j}\sum_{\ci'\in\crst{\gti'+1}}\idf{\bsrt{\ci'}=\bsbl{\gti'}{j}}\trw{\ci'}
\end{align}
By the inductive hypothesis we have that $\sum_{\ci'\in\crst{\gti'+1}}\idf{\bsrt{\ci'}=\bsbl{\gti'}{j}}\trw{\ci'}=1$ for $j\in\{0,1\}$. Substituting into the above equation gives us
\begin{equation}
\notag\sum_{\ci\in\crst{\gti'}}\idf{\bsrt{\ci}=\benl}\trw{\ci}=\sum_{j\in\{0,1\}}\trm{\enl{\liv{\gti'}}}{j}=1
\end{equation}
which proves the inductive hypothesis holds for $\gti=\gti'$. This completes the proof.
\end{proof}

The next lemma quantifies the function $\qf{1}{\ei}$.

\begin{lemma}\label{ivth}
Given any expert $\ei\in\na{\nex}$ and vector $\benl\in\{0,1\}$ we have:
$$\qf{1}{\ei}(\benl)=\frac{1}{\nex}\prod_{\li\in\na{\nL}}\trn{\enl{\li}}$$
\end{lemma}

\begin{proof}
Since $\crs=\crst{1}$ we have that:
\begin{align}
\notag\qf{1}{\ei}(\benl)&=\sum_{\ci\in\crs}\pv{1}(\ei,\ci)\idf{\ci\in\crs~|~\forall \li\in\na{\nL}~,~\ci(\li,\nxt{1}{\li})=\enl{\li}}\\
\notag&=\sum_{\ci\in\crst{1}}\pv{1}(\ei,\ci)\idf{\ci\in\crs~|~\forall \li\in\na{\nL}~,~\ci(\li,\nxt{1}{\li})=\enl{\li}}\\
\notag&=\sum_{\ci\in\crst{1}}\pv{1}(\ei,\ci)\idf{\bsrt{\ci}=\benl}\\
\notag&=\sum_{\ci\in\crst{1}}\frac{1}{\nex}\crw{\ci}\idf{\bsrt{\ci}=\benl}\\
\notag&=\sum_{\ci\in\crst{1}}\frac{1}{\nex}\left(\prod_{\li\in\na{\nL}}\trn{\ci(\li,1)}\right)\trw{\ci}\idf{\bsrt{\ci}=\benl}\\
\notag&=\sum_{\ci\in\crst{1}}\frac{1}{\nex}\left(\prod_{\li\in\na{\nL}}\trn{\ci(\li,\nxt{1}{\li})}\right)\trw{\ci}\idf{\bsrt{\ci}=\benl}\\
\notag&=\sum_{\ci\in\crst{1}}\frac{1}{\nex}\left(\prod_{\li\in\na{\nL}}\trn{\enl{\li}}\right)\trw{\ci}\idf{\bsrt{\ci}=\benl}\\
\notag&=\frac{1}{\nex}\left(\prod_{\li\in\na{\nL}}\trn{\enl{\li}}\right)\sum_{\ci\in\crst{1}}\trw{\ci}\idf{\bsrt{\ci}=\benl}
\end{align}
which, by Lemma \ref{eq1l}, is equal to $\frac{1}{\nex}\prod_{\li\in\na{\nL}}\trn{\enl{\li}}$.
\end{proof}

The next lemma shows how $\qf{\gti}{\ei}$ updates to $\qf{\gti+1}{\ei}$.

\begin{lemma}\label{nxtqt}
Given any trial $\gti\in\na{\gtu}$, expert $\ei\in\na{\nex}$, and vector $\benl\in\{0,1\}^{\nta}$ we have:
$$\qf{\gti+1}{\ei}(\benl)=\qf{\gti}{\ei}(\bsbl{\gti}{0})\trm{0}{\enl{\liv{\gti}}}+\nc{\gti}\exp(-\lr\lo{\gti}{\ei})\qf{\gti}{\ei}(\bsbl{\gti}{1})\trm{1}{\enl{\liv{\gti}}}.$$
\end{lemma}

\begin{proof}
We will prove lemma the statement by also proving that for all  respective circadian tails $\ci\in\crst{\gti+1}$ we have:
$$\sum_{\ci'\in\crs}\pv{\gti+1}(\ei,\ci')\idf{\rtr{\ci'}{\strt{\gti+1}}=\ci}=\qf{\gti+1}{\ei}(\bsrt{\ci})\trw{\ci}.$$ We prove both this and the lemma statement via induction on $\gti$.

We first prove that for all $\ci\in\crst{1}$ we have: $$\sum_{\ci'\in\crs}\pv{1}(\ei,\ci')\idf{\rtr{\ci'}{\strt{1}}=\ci}=\qf{1}{\ei}(\bsrt{\ci})\trw{\ci},$$
 which will seed the induction. Since $\strt{1}=\str$ and hence $\{\ci'\in\crs~|~\rtr{\ci'}{\strt{1}}=\ci\}=\{\ci\}$ we have:
 \begin{align}
\notag \sum_{\ci'\in\crs}\pv{1}(\ei,\ci')\idf{\rtr{\ci'}{\strt{1}}=\ci}&= \pv{1}(\ei,\ci)\\
\notag &=\frac{1}{\nex}\crw{\ci}\\
 \notag&=\frac{1}{\nex}\left(\prod_{\li\in\na{\nL}}\trn{\ci(\li,1)}\right)\trw{\ci}\\
  \notag&=\frac{1}{\nex}\left(\prod_{\li\in\na{\nL}}\trn{\ci(\li,\nxt{1}{\li})}\right)\trw{\ci}\\
 \notag&=\frac{1}{\nex}\left(\prod_{\li\in\na{\nL}}\trn{\srt{\ci}{\li}}\right)\trw{\ci}
 \end{align}
 which, by Lemma \ref{ivth}, is equal to $\qf{1}{\ei}(\bsrt{\ci})\trw{\ci}$.

Now suppose that, for some $\gti'\in\na{\gtu}$, the inductive hypothesis holds for $\gti=\gti'$. We now show that it holds for $\gti=\gti'+1$. First, for all $j\in\{0,1\}$ we define $\ci_j$ to be the circadian tail in $\crst{\gti'}$ defined by $\ci_j(\liv{\gti'},\gltf{\gti'}):=j$ and $\rtr{\ci_j}{\strt{\gti'+1}}:=\ci$. By Lemma \ref{nxtl}, $\ci_j$ is unique and:
$$\{\ci''\in\strt{\gti'}~|~\rtr{\ci''}{\strt{\gti'+1}}=\ci\}=\{\ci_0,\ci_1\}.$$
This implies that:
\begin{equation}\label{eq1}
\{\ci'\in\crs~|~\rtr{\ci'}{\strt{\gti'+1}}=\ci\}=\bigcup_{j\in\{0,1\}}\{\ci'\in\crs~|~\rtr{\ci'}{\strt{\gti'}}=\ci_j\}
\end{equation}
where the two sets on the right hand side are disjoint. This then implies that:
\begin{equation}\label{eq2}
\sum_{\ci'\in\crs}\pv{\gti'+1}(\ei,\ci')\idf{\rtr{\ci'}{\strt{\gti'+1}}=\ci}=\sum_{j\in\{0,1\}}\sum_{\ci'\in\crs}\pv{\gti'+1}(\ei,\ci')\idf{\rtr{\ci'}{\strt{\gti'}}=\ci_j}.
\end{equation}
If $\idf{\rtr{\ci'}{\strt{\gti'}}=\ci_0}$ then $\ci'(\liv{\gti'},\gltf{\gti'}):=0$ so from the definition of $\pv{\gti'+1}$ we have $\pv{\gti'+1}(\ei,\ci)=\pv{\gti'}(\ci')$. On the other hand, if $\idf{\rtr{\ci'}{\strt{\gti'}}=\ci_1}$ then $\ci'(\liv{\gti'},\gltf{\gti'}):=1$ so from the definition of $\pv{\gti'+1}$ we have $\pv{\gti'+1}(\ei,\ci)=\nc{\gti'}\exp(-\lr\lo{\gti}{\ei})\pv{\gti'}(\ci')$. Substituting into the inductive hypothesis gives us:
\begin{align}
\notag\sum_{\ci'\in\crs}\pv{\gti'+1}(\ei,\ci')\idf{\rtr{\ci'}{\strt{\gti'}}=\ci_0}&=\sum_{\ci'\in\crs}\pv{\gti'}(\ei,\ci')\idf{\rtr{\ci'}{\strt{\gti'}}=\ci_0}\\
\notag&=\qf{\gti'}{\ei}(\bsrt{\ci_0})\trw{\ci_0}
\end{align}
and
\begin{align}
\notag&\sum_{\ci'\in\crs}\pv{\gti'+1}(\ei,\ci')\idf{\rtr{\ci'}{\strt{\gti'}}=\ci_1}\\
\notag=&\nc{\gti'}\exp(-\lr\lo{\gti'}{\ei})\sum_{\ci'\in\crs}\pv{\gti'}(\ei,\ci')\idf{\rtr{\ci'}{\strt{\gti'}}=\ci_1}\\
\notag=&\nc{\gti'}\exp(-\lr\lo{\gti'}{\ei})\qf{\gti'}{\ei}(\bsrt{\ci_1})\trw{\ci_1}.
\end{align}
Substituting into Equation \eqref{eq2} gives us:
\begin{align}
\notag&\sum_{\ci'\in\crs}\pv{\gti'+1}(\ei,\ci')\idf{\rtr{\ci'}{\strt{\gti'+1}}=\ci}\\
\label{eq3}=&\qf{\gti'}{\ei}(\bsrt{\ci_0})\trw{\ci_0}+\nc{\gti'}\exp(-\lr\lo{\gti'}{\ei})\qf{\gti'}{\ei}(\bsrt{\ci_1})\trw{\ci_1}.
\end{align}
Now let $\benl:=\bsrt{\ci}$. Given $j\in\{0,1\}$ we have, from definition of $\ci_j$ and Lemma \ref{nxtl}, $\bsrt{\ci_j}=\bsbl{\gti'}{j}$. Also, from the definition of $\ci_j$ and the function $\trw{\cdot}$ we have $\trw{\ci_j}=\trm{j}{\ci(\liv{\gti'},\gltf{\gti'}+1)}\trw{\ci}$ which is equal to $\trm{j}{\enl{\liv{\gti'}}}\trw{\ci}$. Substituting into Equation \eqref{eq3} gives us:
\begin{align}
\notag&\sum_{\ci'\in\crs}\pv{\gti'+1}(\ei,\ci')\idf{\rtr{\ci'}{\strt{\gti'+1}}=\ci}\\
\notag=&\qf{\gti'}{\ei}(\bsbl{\gti'}{0})\trw{\ci_0}+\nc{\gti'}\exp(-\lr\lo{\gti'}{\ei})\qf{\gti'}{\ei}(\bsbl{\gti'}{1})\trw{\ci_1}\\
\label{eq4}=&\left(\qf{\gti'}{\ei}(\bsbl{\gti'}{0})\trm{0}{\enl{\liv{\gti'}}}+\nc{\gti'}\exp(-\lr\lo{\gti'}{\ei})\qf{\gti'}{\ei}(\bsbl{\gti'}{1})\trm{1}{\enl{\liv{\gti'}}}\right)\trw{\ci}.
\end{align}
Letting: $$r:=\qf{\gti'}{\ei}(\bsbl{\gti'}{0})\trm{0}{\enl{\liv{\gti'}}}+\nc{\gti'}\exp(-\lr\lo{\gti'}{\ei})\qf{\gti'}{\ei}(\bsbl{\gti'}{1})\trm{1}{\enl{\liv{\gti'}}}$$ we will now show that $r=\qf{\gti'+1}{\ei}(\benl)$ which, by Equation \eqref{eq4} shows that the inductive hypothesis holds for $\gti=\gti'$.

From the definition of $\qf{\gti'+1}{\ei}$ we have, for any $\benl\in\{0,1\}^{\nL}$:
$$\qf{\gti'+1}{\ei}(\benl)=\sum_{\ci'\in\crs}\pv{\gti'+1}(\ei,\ci')\idf{\forall \li\in\na{\nL}~,~\ci'(\li,\nxt{\gti'+1}{\li})=\enl{\li}}.$$
We have that $\idf{\forall \li\in\na{\nL}~,~\ci'(\li,\nxt{\gti'+1}{\li})=\enl{\li}}=1$ if and only if there exists $\ci''\in\strt{\gti'+1}$ with $\bsrt{\ci''}=\benl$ and $\rtr{\ci'}{\strt{\gti'+1}}=\ci''$. This gives us:
\begin{align}
\notag\qf{\gti'+1}{\ei}(\benl)&=\sum_{\ci'\in\crs}\pv{\gti'+1}(\ei,\ci')\sum_{\ci''\in\strt{\gti'+1}}\idf{\bsrt{\ci''}=\benl}\idf{\rtr{\ci'}{\strt{\gti'+1}}=\ci''}\\
\notag&=\sum_{\ci''\in\strt{\gti'+1}}\idf{\bsrt{\ci''}=\benl}\sum_{\ci'\in\crs}\pv{\gti'+1}(\ei,\ci')\idf{\rtr{\ci'}{\strt{\gti'+1}}=\ci''}
\end{align}
which, by using Equation \eqref{eq4} and definition of $r$, gives us:
\begin{align}
\notag\qf{\gti'+1}{\ei}(\benl)&=\sum_{\ci''\in\strt{\gti'+1}}\idf{\bsrt{\ci''}=\benl}r\trw{\ci''}\\
\notag&=r\sum_{\ci''\in\strt{\gti'+1}}\idf{\bsrt{\ci''}=\benl}\trw{\ci''}
\end{align}
which, by Lemma \ref{eq1l}, is equal to $r$. Equation \eqref{eq4} then shows that the inductive hypothesis holds for $\gti=\gti'$ which completes the proof.
\end{proof}

We now utilize the vectors $\talt{\li}{\nxt{\gti}{\li}}$ and $\bpit{\gti}$ that are defined in the description of the algorithm.

\begin{definition}
For any trial $\gti\in\na{\gtu}$, and learner $\li\in\na{\nL}$ let $\balt{\li}{\gti}{1}:=\talt{\li}{\nxt{\gti}{\li}}$ and let $\balt{\li}{\gti}{0}:=\bs{1}-\balt{\li}{\gti}{1}$.
\end{definition}

We start our analysis of the algorithm with the following lemma.

\begin{lemma}\label{sul}
Given a trial $\gti\in\na{\gtu}$, expert $\ei\in\na{\nex}$ and a function $\tq:\{0,1\}^{\nL}\rightarrow[0,1]$ defined by: \[\tq(\benl):=\pit{\gti}{i}\prod_{\li\in\na{\nL}}\alt{\li}{\gti}{\enl{\li}}{i}\]
we have that for all $j\in\{0,1\}$:
\[\sum_{\benl\in\{0,1\}^{\nL}}\idf{\enl{\liv{\gti}}=j}\tq(\benl)=\pit{\gti}{i}\alt{\liv{\gti}}{\gti}{j}{i}.\]
\end{lemma}

\begin{proof}
Without loss of generality we assume $\liv{\gti}=\nL$. We now take the inductive hypothesis over $\li'\in\na{\nL}$ that for any $\benl'\in\{0,1\}^{\nL}$ we have:
\[\sum_{\benl\in\{0,1\}^{\nL}}\idf{\forall \li\geq\li',\enl{\li}=\enl{\li}'}\tq(\benl)=\pit{\gti}{i}\prod_{\li=\li'}^{\nL}\alt{\li}{\gti}{\enl{\li}'}{i}.\]
In the case that $\li'=1$ we have:
\begin{align}
\notag\sum_{\benl\in\{0,1\}^{\nL}}\idf{\forall \li\geq\li',\enl{\li}=\enl{\li}'}\tq(\benl)&=\sum_{\benl\in\{0,1\}^{\nL}}\idf{\benl=\benl'}\tq(\benl)\\
\notag&=\tq(\benl')\\
\notag&=\pit{\gti}{i}\prod_{\li\in\na{\nL}}\alt{\li}{\gti}{\enl{\li}'}{i}\\
\notag&=\pit{\gti}{i}\prod_{\li=\li'}^{\nL}\alt{\li}{\gti}{\enl{\li}'}{i}.
\end{align}
So the inductive hypothesis holds for $\li'=1$.

Now suppose the inductive hypothesis holds for $\li'=\li''$. We now show that it holds for $\li'=\li''+1$. We have:
\begin{align}
\notag&\sum_{\benl\in\{0,1\}^{\nL}}\idf{\forall \li\geq\li',\enl{\li}=\enl{\li}'}\tq(\benl)\\
\notag=&\sum_{\benl\in\{0,1\}^{\nL}}\idf{\forall \li\geq\li''+1,\enl{\li}=\enl{\li}'}\tq(\benl)\\
\notag=&\sum_{\benl\in\{0,1\}^{\nL}}\left(\sum_{k\in\{0,1\}}\idf{\enl{\li''}=k~\wedge~\forall \li\geq\li''+1,\enl{\li}=\enl{\li}'}\right)\tq(\benl)\\
\notag=&\sum_{k\in\{0,1\}}\sum_{\benl\in\{0,1\}^{\nL}}\idf{\enl{\li''}=k~\wedge~\forall \li\geq\li''+1,\enl{\li}=\enl{\li}'}\tq(\benl)
\end{align}
which, by the inductive hypothesis, is equal to:
\begin{align}
\notag&\sum_{k\in\{0,1\}}\pit{\gti}{i}\alt{l''}{\gti}{k}{i}\prod_{\li=\li''+1}^{\nL}\alt{\li}{\gti}{\enl{\li}'}{i}\\
\notag=&\pit{\gti}{i}\left(\alt{l''}{\gti}{0}{i}+\alt{l''}{\gti}{1}{i}\right)\prod_{\li=\li''+1}^{\nL}\alt{\li}{\gti}{\enl{\li}'}{i}\\
\notag=&\pit{\gti}{i}\prod_{\li=\li''+1}^{\nL}\alt{\li}{\gti}{\enl{\li}'}{i}\\
\notag=&\pit{\gti}{i}\prod_{\li=\li'}^{\nL}\alt{\li}{\gti}{\enl{\li}'}{i}.
\end{align}
So the inductive hypothesis holds for $\li'=\li''+1$. Hence the inductive hypothesis holds for for all $\li'\in\na{\nL}$ and specifically holds for $\li'=\nL$. So, letting $\li'=\nL$, and letting $\benl'$ be any vector in $\{0,1\}^{\nL}$ with $\enl{\nL}'=j$ we have:
\begin{align}
\notag\sum_{\benl\in\{0,1\}^{\nL}}\idf{\enl{\liv{\gti}}=j}\tq(\benl)&=\sum_{\benl\in\{0,1\}^{\nL}}\idf{\enl{\nL}=\enl{\nL}'}\tq(\benl)\\
\notag&=\sum_{\benl\in\{0,1\}^{\nL}}\idf{\forall \li\geq\nL,\enl{\li}=\enl{\li}'}\tq(\benl)\\
\notag&=\pit{\gti}{i}\prod_{\li=\nL}^{\nL}\alt{\li}{\gti}{\enl{\li}'}{i}\\
\notag&=\pit{\gti}{i}\alt{\nL}{\gti}{\enl{\nL}'}{i}\\
\notag&=\pit{\gti}{i}\alt{\nL}{\gti}{j}{i}\\
\notag&=\pit{\gti}{i}\alt{\liv{\gti}}{\gti}{j}{i}.
\end{align}
\end{proof}

We are now ready to quantify $\qf{\gti}{\ei}$ for any trial $\gti$ and expert $\ei$.

\begin{lemma}\label{fqth}
For any trial $\gti\in\na{\gtu}$, any expert $\ei\in\na{\nex}$, and any vector $\benl\in\{0,1\}^{\nL}$ we have:
$$\qf{\gti}{\ei}(\benl)=\pit{\gti}{\ei}\prod_{\li\in\na{\nL}}\alt{\li}{\gti}{\enl{\li}}{\ei}$$
\end{lemma}

\begin{proof}
We prove by induction over $\gti$. In the case that $\gti=1$ we have, from Lemma \ref{ivth}, that:
\begin{align}
\notag\qf{1}{\ei}(\benl)&=\frac{1}{\nex}\prod_{\li\in\na{\nL}}\trn{\enl{\li}}\\
\notag&=\pit{1}{\ei}\prod_{\li\in\na{\nL}}\trn{\enl{\li}}\\
\notag&=\pit{1}{\ei}\prod_{\li\in\na{\nL}}\alt{\li}{1}{\enl{\li}}{\ei}.
\end{align}
Now suppose that, for some trial $\gti'\in\na{\gtu}$, the inductive hypothesis holds for $\gti=\gti'$. We now show that it holds for $\gti=\gti'+1$. For all $\ei'\in\na{\nex}$ let $\btt{\gti'}$, $\dlt{\gti'}{\ei'}$ and $\ept{\gti'}{\ei'}$ be the values of  $\saw$, $\waw{\ei'}$ and $\naw{\ei'}$ formed by the algorithm on trial $\gti'$ respectively. We have:
\begin{align}
\notag&\sum_{(\ei',\ci)\in\was{\gti}}\pv{\gti'}(\ei',\ci)\\
\notag=&\sum_{\ei'\in\na{\nex}}\sum_{\ci\in\crs}\idf{\ci(\liv{\gti'},\gltf{\gti'})=1}\pv{\gti'}(\ei',\ci)\\
\notag=&\sum_{\ei'\in\na{\nex}}\sum_{\ci\in\crs}\idf{\ci(\liv{\gti'},\nxt{\gti'}{\liv{\gti'}})=1}\pv{\gti'}(\ei',\ci)\\
\notag=&\sum_{\ei'\in\na{\nex}}\sum_{\ci\in\crs}\idf{\exists \benl\in\{0,1\}^{\nL}:\enl{\liv{\gti'}}=1\wedge\forall \li\in\na{\nL},\ci(\li,\nxt{\gti'}{\li})=\enl{\li}}\pv{\gti'}(\ei',\ci)\\
\notag=&\sum_{\ei'\in\na{\nex}}\sum_{\ci\in\crs}\sum_{\benl\in\{0,1\}^{\nL}}\idf{\enl{\liv{\gti'}}=1}\idf{\forall \li\in\na{\nL},\ci(\li,\nxt{\gti'}{\li})=\enl{\li}}\pv{\gti'}(\ei',\ci)\\
\notag=&\sum_{\ei'\in\na{\nex}}\sum_{\benl\in\{0,1\}^{\nL}}\idf{\enl{\liv{\gti'}}=1}\sum_{\ci\in\crs}\idf{\forall \li\in\na{\nL},\ci(\li,\nxt{\gti'}{\li})=\enl{\li}}\pv{\gti'}(\ei',\ci)\\
\notag=&\sum_{\ei'\in\na{\nex}}\sum_{\benl\in\{0,1\}^{\nL}}\idf{\enl{\liv{\gti'}}=1}\qf{\gti'}{\ei'}(\benl)
\end{align}
which, by the inductive hypothesis, is equal to:
$$\sum_{\ei'\in\na{\nex}}\sum_{\benl\in\{0,1\}^{\nL}}\idf{\enl{\liv{\gti'}}=1}\pit{\gti'}{\ei'}\prod_{\li\in\na{\nL}}\alt{\li}{\gti'}{\enl{\li}}{\ei'}.$$
By Lemma~\ref{sul}, this is equal to:
$$\sum_{\ei'\in\na{\nex}}\pit{\gti'}{\ei'}\alt{\liv{\gti'}}{\gti'}{1}{\ei'}.$$
Similarly we have:
$$\sum_{(\ei',\ci)\in\was{\gti}}\pv{\gti'}(\ei',\ci)\exp(-\lr\lo{\gti}{\ei'})=\sum_{\ei'\in\na{\nex}}\pit{\gti'}{\ei'}\alt{\liv{\gti'}}{\gti'}{1}{\ei}\exp(-\lr\lo{\gti}{\ei'})=\sum_{\ei'\in\na{\nex}}\pit{\gti'}{\ei'}\waw{\ei'}.$$
Hence we have that:
\begin{align}
\notag\nc{\gti'}&=\frac{\sum_{(\ei',\ci)\in\was{\gti'}}\pv{\gti'}(\ei',\ci)}{\sum_{(\ei',\ci')\in\was{\gti'}}\pv{\gti'}(\ei',\ci')\exp(-\lr\lo{\gti'}{\ei'})}\\
\notag&=\frac{\sum_{\ei'\in\na{\nex}}\pit{\gti'}{\ei'}\alt{\liv{\gti'}}{\gti'}{1}{\ei'}}{\sum_{\ei'\in\na{\nex}}\pit{\gti'}{\ei'}\waw{\ei'}}\\
\notag&=\btt{\gti'}.
\end{align}
Substituting into Lemma \ref{nxtqt} gives us:
\begin{equation}\label{mteq1}
\qf{\gti'+1}{\ei'}(\benl)=\qf{\gti'}{\ei'}(\bsbl{\gti'}{0})\trm{0}{\enl{\liv{\gti'}}}+\btt{\gti'}\exp(-\lr\lo{\gti}{\ei'})\qf{\gti'}{\ei'}(\bsbl{\gti'}{1})\trm{1}{\enl{\liv{\gti'}}}.
\end{equation}
By the inductive hypothesis we have:
\begin{equation}\label{mteq2}
\qf{\gti'}{\ei'}(\bsbl{\gti'}{0})\trm{0}{\enl{\liv{\gti'}}}=\left(\pit{\gti'}{\ei'}\alt{\liv{\gti'}}{\gti'}{0}{\ei'}\prod_{\li\in\na{\nL}:\li\neq\liv{\gti'}}\alt{\li}{\gti'}{\enl{\li}}{\ei'}\right)\trm{0}{\enl{\liv{\gti'}}}
\end{equation}
and
\begin{align}
\notag&\btt{\gti'}\exp(-\lr\lo{\gti}{\ei'})\qf{\gti'}{\ei'}(\bsbl{\gti'}{1})\trm{1}{\enl{\liv{\gti'}}}\\
\label{mteq3}=&\btt{\gti'}\exp(-\lr\lo{\gti}{\ei'})\left(\pit{\gti'}{\ei'}\alt{\liv{\gti'}}{\gti'}{1}{\ei'}\prod_{\li\in\na{\nL}:\li\neq\liv{\gti'}}\alt{\li}{\gti'}{\enl{\li}}{\ei'}\right)\trm{1}{\enl{\liv{\gti'}}}.
\end{align}
For $j\in\{0,1\}$ let: 
$$\zet{j}:=\alt{\liv{\gti'}}{\gti'}{0}{\ei'}\trm{0}{j}+\btt{\gti'}\exp(-\lr\lo{\gti'}{\ei'})\alt{\liv{\gti'}}{\gti'}{1}{\ei'}\trm{1}{j}.$$
Note that for $j\in\{0,1\}$ we have:
\begin{align}
\notag\zet{j}&=\alt{\liv{\gti'}}{\gti'}{0}{\ei'}\trm{0}{j}+\btt{\gti'}\dlt{\gti'}{\ei'}\trm{1}{j}\\
\notag&=\left(1-\alt{\liv{\gti'}}{\gti'}{1}{\ei'}\right)\trm{0}{j}+\btt{\gti'}\dlt{\gti'}{\ei'}\trm{1}{j}
\end{align}
so:
\begin{align}
\notag\zet{0}+\zet{1}&=\left(1-\alt{\liv{\gti'}}{\gti'}{1}{\ei'}\right)(\trm{0}{0}+\trm{0}{1})+\btt{\gti'}\dlt{\gti'}{\ei'}(\trm{1}{0}+\trm{1}{1})\\
\notag&=\left(1-\alt{\liv{\gti'}}{\gti'}{1}{\ei'}\right)+\btt{\gti'}\dlt{\gti'}{\ei'}\\
\label{mteq4}&=\ept{\gti'}{\ei'}
\end{align}
and also:
\begin{align}
\notag\zet{1}&=\left(1-\alt{\liv{\gti'}}{\gti'}{1}{\ei'}\right)\trm{0}{1}+\btt{\gti'}\dlt{\gti'}{\ei'}\trm{1}{1}\\
\notag&=\left(1-\alt{\liv{\gti'}}{\gti'}{1}{\ei'}\right)\sw+\btt{\gti'}\dlt{\gti'}{\ei'}\ww\\
\label{mteq5}&=\sw-\alt{\liv{\gti'}}{\gti'}{1}{\ei'}\sw+\btt{\gti'}\dlt{\gti'}{\ei'}\ww.
\end{align}
Combining equations \eqref{mteq4} and \eqref{mteq5} and noting the update to $\alt{\liv{\gti'}}{\gti'+1}{1}{\ei'}$ gives us:
\begin{equation}
\notag\frac{\zet{1}}{\zet{0}+\zet{1}}=\alt{\liv{\gti'}}{\gti'+1}{1}{\ei'}
\end{equation}
which also gives us:
\begin{equation}
\notag\frac{\zet{0}}{\zet{0}+\zet{1}}=1-\frac{\zet{1}}{\zet{0}+\zet{1}}=1-\alt{\liv{\gti'}}{\gti'+1}{1}{\ei'}=\alt{\liv{\gti'}}{\gti'+1}{0}{\ei'}.
\end{equation}
So we have shown that for all $j\in\{0,1\}$, we have:
\begin{equation}\label{mteq9}
\frac{\zet{j}}{\zet{0}+\zet{1}}=\alt{\liv{\gti'}}{\gti'+1}{j}{\ei'}.
\end{equation}
Substituting equations \eqref{mteq2} and \eqref{mteq3} into Equation \eqref{mteq1} gives us:
\begin{align}
\notag\qf{\gti'+1}{\ei'}(\benl)&=\pit{\gti'}{\ei'}\left(\prod_{\li\in\na{\nL}:\li\neq\liv{\gti'}}\alt{\li}{\gti'}{\enl{\li}}{\ei'}\right)\zet{\enl{\liv{\gti'}}}\\
\notag&=\pit{\gti'}{\ei'}\left(\prod_{\li\in\na{\nL}:\li\neq\liv{\gti'}}\alt{\li}{\gti'+1}{\enl{\li}}{\ei'}\right)\zet{\enl{\liv{\gti'}}}\\
\notag&=\pit{\gti'}{\ei'}(\zet{0}+\zet{1})\left(\prod_{\li\in\na{\nL}:\li\neq\liv{\gti'}}\alt{\li}{\gti'+1}{\enl{\li}}{\ei'}\right)\frac{\zet{\enl{\liv{\gti'}}}}{\zet{0}+\zet{1}}\\
\label{mteq6}&=\pit{\gti'}{\ei'}\ept{\gti'}{\ei'}\left(\prod_{\li\in\na{\nL}:\li\neq\liv{\gti'}}\alt{\li}{\gti'+1}{\enl{\li}}{\ei'}\right)\frac{\zet{\enl{\liv{\gti'}}}}{\zet{0}+\zet{1}}\\
\label{mteq7}&=\pit{\gti'+1}{\ei'}\left(\prod_{\li\in\na{\nL}:\li\neq\liv{\gti'}}\alt{\li}{\gti'+1}{\enl{\li}}{\ei'}\right)\frac{\zet{\enl{\liv{\gti'}}}}{\zet{0}+\zet{1}}\\
\label{mteq8}&=\pit{\gti'+1}{\ei'}\left(\prod_{\li\in\na{\nL}:\li\neq\liv{\gti'}}\alt{\li}{\gti'+1}{\enl{\li}}{\ei'}\right)\alt{\liv{\gti'}}{\gti'+1}{\enl{\liv{\gti'}}}{\ei'}\\
\notag&=\pit{\gti'+1}{\ei'}\left(\prod_{\li\in\na{\nL}}\alt{\li}{\gti'+1}{\enl{\li}}{\ei'}\right)
\end{align}
where Equation \eqref{mteq6} comes from Equation \eqref{mteq4}, Equation \eqref{mteq7} comes from the update to $\pit{\gti'+1}{\ei'}$ and Equation \eqref{mteq8} comes from Equation \eqref{mteq9}. The inductive hypothesis hence holds for $\gti=\gti'+1$.
\end{proof}

The next lemma rewrites the selection $\bal{\gti}$ in terms of the current weights of the specialists.

\begin{lemma}\label{alth}
For all trials $\gti\in\na{\gtu}$ we have:
$$\al{\gti}{\ei}=\frac{\sum_{\ci\in\crs}\idf{(\ei,\ci)\in\was{\gti}}\pv{\gti}(\ei,\ci)}{\sum_{(\ei',\ci)\in\was{\gti}}\pv{\gti}(\ei',\ci)}\,.$$
\end{lemma}

\begin{proof}
Note first that for all $\ei'\in\na{\nex}$ and $\ci\in\crs$ we have that $(\ei',\ci)\in\was{\gti}$ if and only if $\ci(\liv{\gti},\gltf{\gti})=1$ which is equivalent to $\ci(\liv{\gti},\nxt{\gti}{\liv{\gti}})=1$ and is in turn equivalent to the existence of an $\benl\in\{0,1\}^{\nL}$ with $\enl{\liv{\gti}}=1$ and $\idf{\forall \li\in\na{\nL}, \ci(\li,\nxt{\gti}{\li})=\enl{\li}}=1$. Hence we have that:
\begin{align}
\notag&\sum_{\ci\in\crs}\idf{(\ei',\ci)\in\was{\gti}}\pv{\gti}(\ei',\ci)\\
\notag&=\sum_{\ci\in\crs}\pv{\gti}(\ei',\ci)\sum_{\benl\in\{0,1\}^{\nL}}\idf{\enl{\liv{\gti}}=1}\idf{\forall \li\in\na{\nL}, \ci(\li,\nxt{\gti}{\li})=\enl{\li}}\\
\notag&=\sum_{\benl\in\{0,1\}^{\nL}}\idf{\enl{\liv{\gti}}=1}\sum_{\ci\in\crs}\idf{\forall \li\in\na{\nL}, \ci(\li,\nxt{\gti}{\li})=\enl{\li}}\pv{\gti}(\ei',\ci)\\
\notag&=\sum_{\benl\in\{0,1\}^{\nL}}\idf{\enl{\liv{\gti}}=1}\qf{\gti}{\ei'}(\benl)
\end{align}
which, by Lemma \ref{fqth}, is equal to:
$$\sum_{\benl\in\{0,1\}^{\nL}}\idf{\enl{\liv{\gti}}=1}\pit{\gti}{\ei'}\prod_{\li\in\na{\nL}}\alt{\li}{\gti}{\enl{\li}}{\ei'}.$$
Applying Lemma \ref{sul} then gives us:
\begin{equation}
\notag\sum_{\ci\in\crs}\idf{(\ei',\ci)\in\was{\gti}}\pv{\gti}(\ei',\ci)=\pit{\gti}{\ei'}\alt{\liv{\gti}}{\gti}{j}{\ei'}.
\end{equation}
This means that:
\begin{equation}
\notag\sum_{\ci\in\crs}\idf{(\ei,\ci)\in\was{\gti}}\pv{\gti}(\ei,\ci)=\pit{\gti}{\ei}\alt{\liv{\gti}}{\gti}{j}{\ei}
\end{equation}
and:
\begin{align}
\notag\sum_{(\ei',\ci)\in\was{\gti}}\pv{\gti}(\ei',\ci)&=\sum_{\ei'\in\na{\nex}}\sum_{\ci\in\crs}\idf{(\ei',\ci)\in\was{\gti}}\pv{\gti}(\ei',\ci)\\
\notag&=\sum_{\ei'\in\na{\nex}}\pit{\gti}{\ei'}\alt{\liv{\gti}}{\gti}{j}{\ei'}
\end{align}
so:
\begin{align}
\notag\frac{\sum_{\ci\in\crs}\idf{(\ei,\ci)\in\was{\gti}}\pv{\gti}(\ei,\ci)}{\sum_{(\ei',\ci)\in\was{\gti}}\pv{\gti}(\ei',\ci)}&=\frac{\pit{\gti}{\ei}\alt{\liv{\gti}}{\gti}{j}{\ei}}{\sum_{\ei'\in\na{\nex}}\pit{\gti}{\ei'}\alt{\liv{\gti}}{\gti}{j}{\ei'}}\\
\notag&=\al{\gti}{\ei}.
\end{align}
\end{proof}

We are now ready to prove Theorem \ref{equivth}. From the Specialist Hedge (Algorithm \ref{alg:sh}) we have:
$$\barv{\gti}{\ei,\ci}=\frac{\pv{\gti}(\ei,\ci)}{\sum_{(\ei',\ci')\in\was{\gti}}\pv{\gti}(\ei',\ci')}$$
for all $(\ei,\ci)\in\was{\gti}$. So
\begin{align}
\inprod{\arv{\gti}}{\arc{\gti}}&=\sum_{(\ei,\ci)\in\was{\gti}}\barv{\gti}{\ei,\ci}\barc{\gti}{\ei,\ci}\\
&=\sum_{(\ei,\ci)\in\was{\gti}}\barv{\gti}{\ei,\ci}\lo{\gti}{\ei}\\
&=\sum_{\ei\in\na{\nex}}\sum_{\ci\in\crs}\idf{(\ei,\ci)\in\was{\gti}}\barv{\gti}{\ei,\ci}\lo{\gti}{\ei}\\
&=\sum_{\ei\in\na{\nex}}\lo{\gti}{\ei}\sum_{\ci\in\crs}\frac{\idf{(\ei,\ci)\in\was{\gti}}\pv{\gti}(\ei,\ci)}{\sum_{(\ei',\ci)\in\was{\gti}}\pv{\gti}(\ei',\ci)}\\
\label{altheq}&=\sum_{\ei\in\na{\nex}}\lo{\gti}{\ei}\al{\gti}{\ei}\\
&=\blo{\gti}\cdot\bal{\gti}\\
&=\lot{\gti}
\end{align}
where Equation \eqref{altheq} comes from Lemma \ref{alth}.

\subsection{Shortening the Circadians}\label{Shcsec}

In order to achieve slightly better bounds, we will now shorten the circadians and reduce, instead, to the Specialist Allocation model defined with the shortened circadians. We shall show that Specialist Hedge with the shortened circadians is identical to Specialist Hedge with the full-length circadians. We start with the definition of shortened circadians:

\begin{definition}
We define $\sht$ to be the set of all $(\li,\ti)$ such that $\li\in\na{\nL}$ and $\ti\in\na{\ltu{\li}}$. A ``shortened circadian'' is a function from $\sht$ into $\{0,1\}$. Let $\sha$ be the set of all shortened circadians.
\end{definition}

The new reduction is identical to that of Section \ref{ourspecsec} except that it uses the shortened circadians instead of the full-length ones. In particular we have:

$$\pv{1}(\ei,\ci):=\prod_{\li\in\na{\nL}}\trn{\ci(\li,1)}\prod_{\ti\in\na{\ltu{\li}-1}}\trm{\ci(\li,\ti)}{\ci(\li,\ti+1)}$$

for all $\ci\in\sha$. 

\begin{definition}
We let $\was{\gti}'$ and $\was{\gti}''$ be the set $\was{\gti}$ when constructing specialists from shortened circadians and full-length circadians respectively.
\end{definition}

We now partition the full-length circadians as follows:

\begin{definition}
Given a shortened circadian $\ci\in\sha$ we define:
$$\cset{\ci}:=\{\ci'\in\crs~|~\ci'(\li,\ti)=\ci(\li,\ti)~\forall (\li,\ti)\in\sht\}.$$
\end{definition}

We note that the sets $\{\was{\gti}:\gti\in\na{\gtu}\}$ in both instances of the Specialist Allocation model are related as follows:

\begin{lemma}\label{wasrellem}
Given a trial $\gti\in\na{\gtu}$, a shortened circadian $\ci\in\sha$ and an expert $\ei$, we have that $(\ei,\ci)\in\was{\gti}'$ if and only if for all $\ci'\in\cset{\ci}$ we have $(\ei,\ci')\in\was{\gti}''$.
\end{lemma}

\begin{proof}
We have $(\ei,\ci)\in\was{\gti}'$ if and only if $\ci(\liv{\gti},\gltf{\gti})=1$ which, since $\gltf{\gti}\in\na{T^{\li}}$, happens if and only if $\ci'(\liv{\gti},\gltf{\gti})=1$ for all $\ci'\in\cset{\ci}$. This implies the result.
\end{proof}

We now show the equivalence of the new reduction to the old, by considering the function $\pv{\gti}$ in both instances of the Specialist Hedge algorithm. We start with the function $\pv{1}$.

\begin{lemma}\label{inpvl}
Given a shortened circadian $\ci\in\sha$ and an expert $\ei$, we have:
$$\pv{1}(\ei,\ci)=\sum_{\ci'\in\cset{\ci}}\pv{1}(\ei,\ci').$$
\end{lemma}

\begin{proof}
We have: 
\begin{align*}
&\sum_{\ci'\in\cset{\ci}}\pv{1}(\ei,\ci')\\
=&\sum_{\ci'\in\cset{\ci}}\frac{1}{\nex}\prod_{\li\in\na{\nL}}\trn{\ci'(\li,1)}\prod_{\ti\in\na{\ltu{\li}}}\trm{\ci'(\li,\ti)}{\ci'(\li,\ti+1)}\\
=&\sum_{\ci'\in\cset{\ci}}\frac{1}{\nex}\prod_{\li\in\na{\nL}}\trn{\ci'(\li,1)}\left(\prod_{\ti\in\na{\ltu{\li}-1}}\trm{\ci'(\li,\ti)}{\ci'(\li,\ti+1)}\right)\trm{\ci'(\li,\ltu{\li})}{\ci'(\li,\ltu{\li}+1)}\\
=&\sum_{\ci'\in\cset{\ci}}\frac{1}{\nex}\left(\prod_{\li\in\na{\nL}}\trn{\ci'(\li,1)}\prod_{\ti\in\na{\ltu{\li}-1}}\trm{\ci'(\li,\ti)}{\ci'(\li,\ti+1)}\right)\left(\prod_{\li\in\na{\nL}}\trm{\ci'(\li,\ltu{\li})}{\ci'(\li,\ltu{\li}+1)}\right)\\
=&\sum_{\ci'\in\cset{\ci}}\frac{1}{\nex}\left(\prod_{\li\in\na{\nL}}\trn{\ci(\li,1)}\prod_{\ti\in\na{\ltu{\li}-1}}\trm{\ci(\li,\ti)}{\ci(\li,\ti+1)}\right)\left(\prod_{\li\in\na{\nL}}\trm{\ci'(\li,\ltu{\li})}{\ci'(\li,\ltu{\li}+1)}\right)\\
=&\sum_{\ci'\in\cset{\ci}}\frac{1}{\nex}\pv{1}(\ei,\ci)\left(\prod_{\li\in\na{\nL}}\trm{\ci'(\li,\ltu{\li})}{\ci'(\li,\ltu{\li}+1)}\right)\\
=&\frac{1}{n}\pv{1}(\ei,\ci)\sum_{\ci'\in\cset{\ci}}\prod_{\li\in\na{\nL}}\trm{\ci'(\li,\ltu{\li})}{\ci'(\li,\ltu{\li}+1)}\\
=&\frac{1}{n}\pv{1}(\ei,\ci)\sum_{\bs{f}\in\{0,1\}^{\nL}}\prod_{\li\in\na{\nL}}\trm{\ci'(\li,\ltu{\li})}{f_{\li}}\\
=&\frac{1}{n}\pv{1}(\ei,\ci)\prod_{\li\in\na{\nL}}\sum_{f_{\li}\in\{0,1\}}\trm{\ci'(\li,\ltu{\li})}{f_{\li}}\\
=&\frac{1}{n}\pv{1}(\ei,\ci)\prod_{\li\in\na{\nL}}\left(\trm{\ci'(\li,\ltu{\li})}{0}+\trm{\ci'(\li,\ltu{\li})}{1}\right)\\
=&\frac{1}{n}\pv{1}(\ei,\ci)\prod_{\li\in\na{\nL}}1\\
=&\frac{1}{n}\pv{1}(\ei,\ci).
\end{align*}
\end{proof}

We now extend to all $\gti\in\na{\gtu}$.

\begin{lemma}\label{wassuml}
Given a trial $\gti\in\na{\gtu}$, a shortened circadian $\ci\in\sha$ and an expert $\ei$ we have:
$$\pv{\gti}(\ei,\ci)=\sum_{\ci'\in\cset{\ci}}\pv{\gti}(\ei,\ci').$$
\end{lemma}

\begin{proof}
We prove by induction on $\gti$. By Lemma \ref{inpvl} it holds in the case that $\gti=1$. Now suppose it holds for $\gti=\gti'$ for some $\gti'\in\na{\gtu-1}$. We will now show that it holds for $\gti=\gti'+1$. We have:
\begin{align}
\notag\pv{\gti'+1}(\ei,\ci)&=\frac{\pv{\gti'}(\ei,\ci)\exp(-\lr\lo{\gti'}{\ei})\sum_{(\ei'',\ci'')\in\was{\gti'}'}\pv{\gti'}(\ei'',\ci'')}{\sum_{(\ei'',\ci'')\in\was{\gti'}'}\pv{\gti'}(\ei'',\ci'')
\exp(-\lr\lo{\gti'}{\ei''})}\\
\notag&=\frac{\left(\sum_{\ci'\in\cset{\ci}}\pv{\gti'}(\ei,\ci')\right)\exp(-\lr\lo{\gti'}{\ei})\sum_{(\ei'',\ci'')\in\was{\gti'}'}\sum_{\ci'''\in\cset{\ci''}}\pv{\gti'}(\ei'',\ci''')}{\sum_{(\ei'',\ci'')\in\was{\gti'}'}\left(\sum_{\ci'''\in\cset{\ci''}}\pv{\gti'}(\ei'',\ci''')\right)\exp(-\lr\lo{\gti'}{\ei''})}\\
\notag&=\sum_{\ci'\in\cset{\ci}}\frac{\pv{\gti'}(\ei,\ci')\exp(-\lr\lo{\gti'}{\ei})\sum_{(\ei'',\ci'')\in\was{\gti'}'}\sum_{\ci'''\in\cset{\ci''}}\pv{\gti'}(\ei'',\ci''')}{\sum_{(\ei'',\ci'')\in\was{\gti'}'}\left(\sum_{\ci'''\in\cset{\ci''}}\pv{\gti'}(\ei'',\ci''')\right)\exp(-\lr\lo{\gti'}{\ei''})}\\
\label{nrline1}&=\sum_{\ci'\in\cset{\ci}}\frac{\pv{\gti'}(\ei,\ci')\exp(-\lr\lo{\gti'}{\ei})\sum_{(\ei'',\ci''')\in\was{\gti'}''}\pv{\gti'}(\ei'',\ci''')}{\sum_{(\ei'',\ci''')\in\was{\gti'}''}\left(\pv{\gti'}(\ei'',\ci''')\right)\exp(-\lr\lo{\gti'}{\ei''})}\\
\notag&=\sum_{\ci'\in\cset{\ci}}\pv{\gti'+1}(\ei,\ci')
\end{align}
where Equation \eqref{nrline1} comes from Lemma \ref{wasrellem}. This proves the inductive hypothesis.
\end{proof}

We then have:
\begin{align}
\notag\sum_{(\ei,\ci)\in\was{\gti}'}\arv{\gti}(\ei,\ci)\arc{\gti}(\ei,\ci)&=\sum_{(\ei,\ci)\in\was{\gti}'}\frac{\pv{\gti}(\ei,\ci)}{\sum_{(\ei'',\ci'')\in\was{\gti}'}\pv{\gti}(\ei'',\ci'')}\arc{\gti}(\ei,\ci)\\
\notag&=\sum_{(\ei,\ci)\in\was{\gti}'}\frac{\pv{\gti}(\ei,\ci)}{\sum_{(\ei'',\ci'')\in\was{\gti}'}\pv{\gti}(\ei'',\ci'')}\lo{\gti}{\ei}\\
\label{nrline3}&=\sum_{(\ei,\ci)\in\was{\gti}'}\frac{\sum_{\ci'\in\cset{\ci}}\pv{\gti}(\ei,\ci')}{\sum_{(\ei'',\ci'')\in\was{\gti}'}\sum_{\ci'''\in\cset{\ci''}}\pv{\gti}(\ei'',\ci''')}\lo{\gti}{\ei}\\
\notag&=\frac{\sum_{(\ei,\ci)\in\was{\gti}'}\sum_{\ci'\in\cset{\ci}}\pv{\gti}(\ei,\ci')}{\sum_{(\ei'',\ci'')\in\was{\gti}'}\sum_{\ci'''\in\cset{\ci''}}\pv{\gti}(\ei'',\ci''')}\lo{\gti}{\ei}\\
\label{nrline2}&=\frac{\sum_{(\ei,\ci')\in\was{\gti}''}\pv{\gti}(\ei,\ci')}{\sum_{(\ei'',\ci''')\in\was{\gti}''}\pv{\gti}(\ei'',\ci''')}\lo{\gti}{\ei}\\
\notag&=\sum_{(\ei,\ci')\in\was{\gti}''}\frac{\pv{\gti}(\ei,\ci')}{\sum_{(\ei'',\ci''')\in\was{\gti}''}\pv{\gti}(\ei'',\ci''')}\lo{\gti}{\ei}\\
\notag&=\sum_{(\ei,\ci')\in\was{\gti}''}\frac{\pv{\gti}(\ei,\ci')}{\sum_{(\ei'',\ci''')\in\was{\gti}''}\pv{\gti}(\ei'',\ci''')}\arc{\gti}(\ei,\ci')\\
\notag&=\sum_{(\ei,\ci')\in\was{\gti}''}\arv{\gti}(\ei,\ci')\arc{\gti}(\ei,\ci')
\end{align}
where Equation \eqref{nrline3} and Equation \eqref{nrline2} come from Lemma \ref{wassuml} and Lemma \ref{wasrellem} respectively. So the outputs, $\inprod{\arv{\gti}}{\arc{\gti}}$ of Specialist Hedge, when run with shortened circadians and full-length circadians are the same.

From Theorem \ref{equivth} we then have that the loss, on trial $\gti$, of Algorithm \ref{alg:fin} is equal to  
$$\lot{\gti}=\bal{\gti}\cdot\blo{\gti}=\inprod{\arv{\gti}}{\arc{\gti}}=\CBT$$ 
when Specialist Hedge is run with shortened circadians. From here on, we will consider only the case when Specialist Hedge is run on the Specialist Allocation model with shortened circadians.

\subsection{Our Comparator Distribution}\label{cdss}

We assume we have an arbitrary comparator sequence $\com = (\sh{1}{1},\ldots, \sh{1}{T^1},\ldots,\sh{\nL}{1},\ldots,\sh{\nL}{T^{\nL}})\in[\nex]^T$ for the Multitask Allocation model (cf. Figure~\ref{fig:map}) with switching, and we consider the Specialist Allocation model (cf. Figure~\ref{fig:sap}) formed from the Multitask Allocation model via the reduction given in Subsection~\ref{Shcsec}.

We now define, from $\com$, our comparator distribution $\cop{}$ that appears in Theorem \ref{allospecth} and quantify the values $\inprod{\copgti}{\arc{\gti}}$ and $\musgt$ that also appear in Theorem \ref{allospecth}.

\begin{definition}
We have an arbitrary comparator $\com = (\sh{1}{1},\ldots, \sh{1}{T^1},\ldots,\sh{\nL}{1},\ldots,\sh{\nL}{T^{\nL}})\in[\nex]^T$. For all $\gti\in\na{T}$ we define $\sh{\gti}{}:=\sh{\liv{\gti}}{\gltf{\gti}}$
Given $\ei\in\ran$, we define its ``respective shortened circadian'' $\rcr{\ei}\in\sha$ by:
$$\rcr{\ei}(\li,\ti)=\idf{\sh{\li}{\ti}=\ei}$$
for all $(\li,\ti)\in\sht$. We define $\ufu:\hye\rightarrow[0,1]$ by:
$$\upd{\ei,\ci}=\frac{1}{\siz{\com}}\idf{\ei\in\ran\wedge\ci=\rcr{\ei}}.$$
Note that $\ufu$ is a probability distribution on $\hye$.
\end{definition}

We define $\copgti$ and $\musgt$ as in Theorem \ref{allospecth}:

\begin{definition}
On any trial $\gti$ we define $\musgt:=\sum_{\hi\in\was{\gti}}\cop{}(\hi)$ and for all $\hi\in\was{\gti}$ define $\copgti(\hi):=\cop{}(\hi)/\musgt$.
\end{definition}

We now quantify $\musgt$:

\begin{lemma}\label{musize}
On any trial $\gti$ we have $\musgt=1/\siz{\com}$.
\end{lemma}

\begin{proof}
We have:
\begin{align}
\notag\musgt&=\sum_{\hi\in\was{\gti}}\cop{}(\hi)\\
\notag&=\sum_{\hi\in\hye}\idf{\hi\in\was{\gti}}\cop{}(\hi)\\
\notag&=\sum_{\ei\in\ran}\idf{(\ei,\rcr{\ei})\in\was{\gti}}\cop{}(\rcr{\ei})\\
\notag&=\sum_{\ei\in\ran}\idf{\ei=\sh{\gti}{}}\cop{}(\ei,\rcr{\ei})\\
\notag&=\cop{}\left(\sh{\gti}{},\rcr{\sh{\gti}{}}\right)\\
\label{progl6}&=1/\siz{\com}.
\end{align}
\end{proof}

We also quantify $\inprod{\copgti}{\arc{\gti}}$:

\begin{lemma}\label{musize2}
Given the reduction in Subsection~\ref{ourspecsec} 
from the Multitask Allocation model \ref{fig:map} with switching to the Specialist Allocation model (cf. Figure~\ref{fig:sap}), on any trial $\gti\in\na{T}$ we have:
$$\inprod{\copgti}{\arc{\gti}}=\lo{\gti}{\sh{\gti}{}}.$$
\end{lemma}

\begin{proof}
We have:
\begin{align}
\notag \inprod{\copgti}{\arc{\gti}}&=\sum_{\hi\in\was{\gti}}\copgti(\hi)\barc{\gti}{\hi}\\
\notag&=\sum_{(\ei,\ci)\in\was{\gti}}\copgti(\ei,\ci)\barc{\gti}{\ei,\ci}\\
\label{ccreq}&=\sum_{(\ei,\ci)\in\was{\gti}}\copgti(\ei,\ci)\lo{\gti}{\ei}\\
\notag&=\sum_{(\ei,\ci)\in\hye}\idf{(\ei,\ci)\in\was{\gti}}\copgti(\ei,\ci)\lo{\gti}{\ei}\\
\notag&=\sum_{(\ei,\ci)\in\hye}\idf{(\ei,\ci)\in\was{\gti}}\idf{\ci=\rcr{\ei}}\copgti(\ei,\ci)\lo{\gti}{\ei}\\
\notag&=\sum_{\ei\in\na{\nex}}\idf{(\ei,\rcr{\ei})\in\was{\gti}}\copgti(\ei,\rcr{\ei})\lo{\gti}{\ei}\\
\notag&=\sum_{\ei\in\na{\nex}}\idf{\rcr{\ei^{\liv{\gti}}_{\gltf{\gti}}}=1}\copgti(\ei,\rcr{\ei})\lo{\gti}{\ei}\\
\notag&=\sum_{\ei\in\na{\nex}}\idf{\ei=\sh{{\liv{\gti}}}{{\gltf{\gti}}}}\copgti(\ei,\rcr{\ei})\lo{\gti}{\ei}\\
\notag&=\sum_{\ei\in\na{\nex}}\idf{\ei=\sh{\gti}{}}\copgti(\ei,\rcr{\ei})\lo{\gti}{\ei}\\
\notag&=\copgti(\sh{\gti}{},\rcr{\ei})\lo{\gti}{\sh{\gti}{}}\\
\notag&=\frac{\cop{}(\sh{\gti}{},\rcr{\ei})}{\musgt}\lo{\gti}{\sh{\gti}{}}\\
\notag&=\frac{1}{\siz{\com}\musgt}\lo{\gti}{\sh{\gti}{}}\\
\end{align}
where Equation \eqref{ccreq} comes from the reduction. Substituting in the equality in Lemma \ref{musize} then gives us the result.
\end{proof}

\subsection{The Initial Relative Entropy}\label{iress}

In this subsection, we quantify the term $\rel{\cop{}}{\pv{1}}$ that appears in Theorem \ref{allospecth} when analysing the reduction of Subsection \ref{Shcsec} given our comparator sequence $\com = (\sh{1}{1},\ldots, \sh{1}{T^1},\ldots,\sh{\nL}{1},\ldots,\sh{\nL}{T^{\nL}})\in[\nex]^T$ for the Multitask Allocation model (see Figure~\ref{fig:map}) with switching. Specifically, we prove the following theorem:

\begin{theorem}\label{tuneth}
Setting $\sip:=1/\siz{\com}$, $\ww:=1-\nsw{\com}/(\gtu-\nL)$ and $\sw:=\nsw{\com}/((\siz{\com}-1)(\gtu-\nL))$ we have:
\begin{align}
\notag\siz{\com}\rel{\cop{}}{\pv{1}}&\leq\siz{\com}\log\left(\frac{\nex}{\siz{\com}}\right)+\nL\siz{\com}\Bbine{\frac{1}{\siz{\com}}}\\
\notag&+(\gtu-\nL)\Bbine{\frac{\nsw{\com}}{\gtu-\nL}}+(\siz{\com}-1)(\gtu-\nL)\Bbine{\frac{\nsw{\com}}{(\siz{\com}-1)(\gtu-\nL)}}
\end{align}
and
\begin{align}
\notag\siz{\com}\rel{\cop{}}{\pv{1}}&\leq\siz{\com}\log\left(\frac{\nex}{\siz{\com}}\right)+\nL(\log(\siz{\com})+1)\\
\notag&+\nsw{\com}\left(\log(\siz{\com}-1)+2\log\left(\frac{\gtu-\nL}{\nsw{\com}}\right)+2\right).
\end{align}
\end{theorem}

Our first lemma quantifies the sum of $\log\left(\trn{\rcr{\ei}(\li,1)}\right)$ over all experts $\ei$ in $\ran$.

\begin{lemma}\label{sl1}
For all tasks $\li\in\na{\nL}$ we have:
$$\sum_{\ei\in\ran}\log\left(\trn{\rcr{\ei}(\li,1)}\right)=\log(\sip)+(\siz{\com}-1)\log(1-\sip).$$
\end{lemma}

\begin{proof}
For $\ei:=\com(\li,1)$ we have  $\rcr{\ei}(\li,1)=1$ and for $\ei\in\ran$ with $\ei\neq\com(\li,1)$ we have $\rcr{\ei}(\li,1)=0$. This implies that:
\begin{align}
\notag\sum_{\ei\in\ran}\log\left(\trn{\rcr{\ei}(\li,1)}\right)&=\log(\trn{1})+(|\ran|-1)\log(\trn{0})\\
\notag&=\log(\sip)+(\siz{\com}-1)\log(1-\sip).
\end{align}
\end{proof}

The next two lemmas quantify, given some $(\li, t)\in\na{\nL}\times\na{\ltu{\li}-1}$, the sum of $\log\left(\trm{\rcr{\ei}(\li,\ti)}{\rcr{\ei}(\li,\ti+1)}\right)$ over all experts $\ei$ in $\ran$. The first lemma is in the case that $\com(\li,\ti)=\com(\li,\ti+1)$.

\begin{lemma}\label{sl2}
For all tasks $\li\in\na{\nL}$ and all $\ti\in\na{\ltu{\li}-1}$ with $\com(\li,\ti)=\com(\li,\ti+1)$ we have:
$$\sum_{\ei\in\ran}\log\left(\trm{\rcr{\ei}(\li,\ti)}{\rcr{\ei}(\li,\ti+1)}\right)=\log(\ww)+(\siz{\com}-1)\log(1-\sw).$$
\end{lemma}

\begin{proof}
For $\ei:=\com(\li,\ti)$ we have, since also $\ei=\com(\li,\ti+1)$, that $\rcr{\ei}(\li,\ti)=\rcr{\ei}(\li,\ti+1)=1$. For $\ei\in\ran$ with $\ei\neq\com(\li,\ti)$, since also $\ei\neq\com(\li,\ti+1)$, that $\rcr{\ei}(\li,\ti)=\rcr{\ei}(\li,\ti+1)=0$. This implies that:
\begin{align}
\notag\sum_{\ei\in\ran}\log\left(\trm{\rcr{\ei}(\li,\ti)}{\rcr{\ei}(\li,\ti+1)}\right)&=\log\left(\trm{1}{1}\right)+(|\ran|-1)\log\left(\trm{0}{0}\right)\\
\notag&=\log(\ww)+(\siz{\com}-1)\log(1-\sw).
\end{align}
\end{proof}

The second lemma is in the case that $\com(\li,\ti)\neq\com(\li,\ti+1)$.

\begin{lemma}\label{sl3}
For all tasks $\li\in\na{\nL}$ and all $\ti\in\na{\ltu{\li}-1}$ with $\com(\li,\ti)\neq\com(\li,\ti+1)$, we have:
$$\sum_{\ei\in\ran}\log\left(\trm{\rcr{\ei}(\li,\ti)}{\rcr{\ei}(\li,\ti+1)}\right)=\log(1-\ww)+\log(\sw)+(\siz{\com}-2)\log(1-\sw).$$
\end{lemma}

\begin{proof}
Let $\ei':=\com(\li,\ti)$ and $\ei'':=\com(\li,\ti+1)$. Note that $\ei'\neq\ei''$. Since $\ei'=\com(\li,\ti)$, we have $\ei'\neq\com(\li,\ti+1)$, so $\rcr{\ei'}(\li,\ti)=1$ and $\rcr{\ei'}(\li,\ti+1)=0$. Since $\ei''=\com(\li,\ti+1)$, we have $\ei'\neq\com(\li,\ti)$, so $\rcr{\ei'}(\li,\ti)=0$ and $\rcr{\ei'}(\li,\ti+1)=1$. For $\ei\in\ran\setminus\{\ei',\ei''\}$ we have $\ei\neq\com(\li,\ti)$ and $\ei\neq\com(\li,\ti+1)$, so $\rcr{\ei}(\li,\ti)=0$ and $\rcr{\ei}(\li,\ti+1)=0$. Hence, we have:
\begin{align}
\notag&\sum_{\ei\in\ran}\log\left(\trm{\rcr{\ei}(\li,\ti)}{\rcr{\ei}(\li,\ti+1)}\right)\\
\notag=&\log\left(\trm{\rcr{\ei'}(\li,\ti)}{\rcr{\ei'}(\li,\ti+1)}\right)+\log\left(\trm{\rcr{\ei''}(\li,\ti)}{\rcr{\ei''}(\li,\ti+1)}\right)+\sum_{\ei\in\ran\setminus\{\ei',\ei''\}}\log\left(\trm{\rcr{\ei}(\li,\ti)}{\rcr{\ei}(\li,\ti+1)}\right)\\
\notag=&\log(\trm{1}{0})+\log(\trm{0}{1})+\sum_{\ei\in\ran\setminus\{\ei',\ei''\}}\log\left(\trm{0}{0}\right)\\
\notag=&\log(\trm{1}{0})+\log(\trm{0}{1})+(|\ran-2|)\log\left(\trm{0}{0}\right)\\
\notag=&\log(1-\ww)+\log(\sw)+(\siz{\com}-2)\log(1-\sw).
\end{align}
\end{proof}

With the above lemmas in hand, we now quantify the relative entropy between the distribution $\cop{}$ and the initial weights of the specialists, in terms of $\sip$, $\ww$ and $\sw$.
\begin{theorem}\label{re1th}
Given that $\sip$, $\ww$ and $\sw$ are all in $[0,1]$, we have:
\begin{align}
\notag\siz{\com}\rel{\cop{}}{\pv{1}}&\leq \siz{\com}\log\left(\frac{\nex}{\siz{\com}}\right)-\nL\left(\log(\sip)+(\siz{\com}-1)\log(1-\sip)\right)\\
\notag&\quad-\log(\ww)\left(\gtu-\nL-\nsw{\com}\right)-\log(1-\ww)\nsw{\com}-\log(\sw)\nsw{\com}\\
\notag&\quad-((\siz{\com}-1)(\gtu-\nL)-\nsw{\com})\log(1-\sw).
\end{align}
\end{theorem}

\begin{proof}
We have
\begin{align}
\notag&\siz{\com}\rel{\cop{}}{\pv{1}}\\
\notag&=\siz{\com}\sum_{(\ei,\ci)\in\hye}\cop{}(\ei,\ci)\log\left(\frac{\cop{}(\ei,\ci)}{\pv{1}(\ei,\ci)}\right)\\
\notag&=\siz{\com}\sum_{(\ei,\ci)\in\hye}\cop{}(\ei,\ci)\log\left(\frac{\nex\cop{}(\ei,\ci)}{\crw{\ci}}\right)\\
\notag&=\siz{\com}\sum_{\ei\in\ran}\cop{}(\ei,\rcr{\ei})\log\left(\frac{\nex\cop{}(\ei,\rcr{\ei})}{\crw{\rcr{\ei}}}\right)\\
\notag&=\siz{\com}\sum_{\ei\in\ran}\frac{1}{\siz{\com}}\log\left(\frac{\nex/\siz{\com}}{\crw{\rcr{\ei}}}\right)\\
\notag&=\sum_{\ei\in\ran}\log\left(\frac{\nex/\siz{\com}}{\crw{\rcr{\ei}}}\right)\\
\notag&=\sum_{\ei\in\ran}\log\left(\frac{\nex}{\siz{\com}}\right)-\sum_{\ei\in\ran}\log\left(\crw{\rcr{\ei}}\right)\\
\notag&=\siz{\com}\log\left(\frac{\nex}{\siz{\com}}\right)-\sum_{\ei\in\ran}\log\left(\crw{\rcr{\ei}}\right)\\
\notag&=\siz{\com}\log\left(\frac{\nex}{\siz{\com}}\right)-\sum_{\ei\in\ran}\log\left(\prod_{\li\in\na{\nL}}\trn{\rcr{\ei}(\li,1)}\prod_{\ti\in\na{\ltu{\li}-1}}\trm{\rcr{\ei}(\li,\ti)}{\rcr{\ei}(\li,\ti+1)}\right)\\
\notag&=\siz{\com}\log\left(\frac{\nex}{\siz{\com}}\right)-\sum_{\ei\in\ran}\sum_{\li\in\na{\nL}}\left(\log\left(\trn{\rcr{\ei}(\li,1)}\right)+\sum_{\ti\in\na{\ltu{\li}-1}}\log\left(\trm{\rcr{\ei}(\li,\ti)}{\rcr{\ei}(\li,\ti+1)}\right)\right)\\
\notag&=\siz{\com}\log\left(\frac{\nex}{\siz{\com}}\right)-\sum_{\li\in\na{\nL}}\left(\sum_{\ei\in\ran}\log\left(\trn{\rcr{\ei}(\li,1)}\right)+\sum_{\ti\in\na{\ltu{\li}-1}}\sum_{\ei\in\ran}\log\left(\trm{\rcr{\ei}(\li,\ti)}{\rcr{\ei}(\li,\ti+1)}\right)\right)\\
\label{re1eq}&=\siz{\com}\log\left(\frac{\nex}{\siz{\com}}\right)-\nL\left(\log(\sip)+(\siz{\com}-1)\log(1-\sip)\right)-\sum_{\li\in\na{\nL}}\sum_{\ti\in\na{\ltu{\li}-1}}\sum_{\ei\in\ran}\log\left(\trm{\rcr{\ei}(\li,\ti)}{\rcr{\ei}(\li,\ti+1)}\right)
\end{align}
where Equation \eqref{re1eq} comes from Lemma \ref{sl1}. By Lemma \ref{sl2} we have:
\begin{align}
\notag&\sum_{\li\in\na{\nL}}\sum_{\ti\in\na{\ltu{\li}-1}}\idf{\sh{\li}{\ti}=\sh{\li}{\ti+1}}\sum_{\ei\in\ran}\log\left(\trm{\rcr{\ei}(\li,\ti)}{\rcr{\ei}(\li,\ti+1)}\right)\\
\notag=&\sum_{\li\in\na{\nL}}\sum_{\ti\in\na{\ltu{\li}-1}}\idf{\sh{\li}{\ti}=\sh{\li}{\ti+1}}\left(\log(\ww)+(\siz{\com}-1)\log(1-\sw)\right)\\
\notag=&\left(\log(\ww)+(\siz{\com}-1)\log(1-\sw)\right)\sum_{\li\in\na{\nL}}\sum_{\ti\in\na{\ltu{\li}-1}}\idf{\sh{\li}{\ti}=\sh{\li}{\ti+1}}\\
\notag=&\left(\log(\ww)+(\siz{\com}-1)\log(1-\sw)\right)\sum_{\li\in\na{\nL}}\sum_{\ti\in\na{\ltu{\li}-1}}(1-\idf{\sh{\li}{\ti}\neq\sh{\li}{\ti+1}})\\
\notag=&\left(\log(\ww)+(\siz{\com}-1)\log(1-\sw)\right)\left(\sum_{\li\in\na{\nL}}(\ltu{\li}-1)-\sum_{\li\in\na{\nL}}\sum_{\ti\in\na{\ltu{\li}-1}}\idf{\sh{\li}{\ti}\neq\sh{\li}{\ti+1}}\right)\\
\notag=&\left(\log(\ww)+(\siz{\com}-1)\log(1-\sw)\right)\left(\gtu-\nL-\sum_{\li\in\na{\nL}}\sum_{\ti\in\na{\ltu{\li}-1}}\idf{\sh{\li}{\ti}\neq\sh{\li}{\ti+1}}\right)\\
\notag=&\left(\log(\ww)+(\siz{\com}-1)\log(1-\sw)\right)\left(\gtu-\nL-\nsw{\com}\right)
\end{align}
and by Lemma \ref{sl3} we have:
\begin{align}
\notag&\sum_{\li\in\na{\nL}}\sum_{\ti\in\na{\ltu{\li}-1}}\idf{\sh{\li}{\ti}\neq\sh{\li}{\ti+1}}\sum_{\ei\in\ran}\log\left(\trm{\rcr{\ei}(\li,\ti)}{\rcr{\ei}(\li,\ti+1)}\right)\\
\notag=&\sum_{\li\in\na{\nL}}\sum_{\ti\in\na{\ltu{\li}-1}}\idf{\sh{\li}{\ti}\neq\sh{\li}{\ti+1}}\left(\log(1-\ww)+\log(\sw)+(\siz{\com}-2)\log(1-\sw)\right)\\
\notag=&\left(\log(1-\ww)+\log(\sw)+(\siz{\com}-2)\log(1-\sw)\right)\sum_{\li\in\na{\nL}}\sum_{\ti\in\na{\ltu{\li}-1}}\idf{\sh{\li}{\ti}\neq\sh{\li}{\ti+1}}\\
\notag=&\left(\log(1-\ww)+\log(\sw)+(\siz{\com}-2)\log(1-\sw)\right)\nsw{\com}
\end{align}
so:
\begin{align}
\notag&\sum_{\li\in\na{\nL}}\sum_{\ti\in\na{\ltu{\li}-1}}\sum_{\ei\in\ran}\log\left(\trm{\rcr{\ei}(\li,\ti)}{\rcr{\ei}(\li,\ti+1)}\right)\\
\notag=&\sum_{\li\in\na{\nL}}\sum_{\ti\in\na{\ltu{\li}-1}}(\idf{\sh{\li}{\ti}\neq\sh{\li}{\ti+1}}+\idf{\sh{\li}{\ti}\neq\sh{\li}{\ti+1}})\sum_{\ei\in\ran}\log\left(\trm{\rcr{\ei}(\li,\ti)}{\rcr{\ei}(\li,\ti+1)}\right)\\
\notag=&\left(\left(\log(\ww)+(\siz{\com}-1)\log(1-\sw)\right)\left(\gtu-\nL-\nsw{\com}\right)\right)+\left(\left(\log(1-\ww)+\log(\sw)+(\siz{\com}-2)\log(1-\sw)\right)\nsw{\com}\right)\\
\notag=&\log(\ww)\left(\gtu-\nL-\nsw{\com}\right)+\log(1-\ww)\nsw{\com}+\log(\sw)\nsw{\com}+((\siz{\com}-1)(\gtu -\nL)-\nsw{\com})\log(1-\sw).
\end{align}
Substituting into Equation \eqref{re1eq} gives us the result.
\end{proof}

The next lemma will assist us in tuning $\sip$, $\ww$ and $\sw$.
\begin{lemma}\label{bef1l}
Given $x,y\in\mathbb{R}^+$, if we set $z:=\frac{x}{x+y}$ then:
$$-x\log(z)-y\log(1-z)=(x+y)\Bbine{\frac{x}{x+y}}$$
and
$$-x\log(z)-y\log(1-z)\leq x\log\left(\frac{x+y}{x}\right)+x.$$
\end{lemma}

\begin{proof}
We have $x=(x+y)z$ and $y=(x+y)(1-z)$ so 
\begin{align}
-\label{bineline1}x\log(z)-y\log(1-z)&=(x+y)(-z\log(z)-(1-z)\log(1-z))\\
\notag&=(x+y)\bine{z}\\
\notag&=(x+y)\Bbine{\frac{x}{x+y}}.
\end{align}

We recall the standard inequality: 
\[ \frac{1}{z}\bine{z}\leq\log(1/z)+1. \]
Plugging into the Equation \eqref{bineline1} gives us:
$$-x\log(z)-y\log(1-z)\leq (x+y)(z\log(1/z)+z).$$
Substituting in the value of $z$ then gives us the result.
\end{proof}

The next lemma utilizes the inequality in Lemma \ref{bef1l} to give us the values of certain quantities in Theorem \ref{re1th} when $\sip$, $\ww$ and $\sw$ are tuned.
\begin{lemma}\label{bef2l}
Setting $\sip:=1/\siz{\com}$, $\ww:=1-\nsw{\com}/(\gtu-\nL)$ and $\sw:=\nsw{\com}/((\siz{\com}-1)(\gtu-\nL))$ we have:
\begin{enumerate}
\item $ \begin{aligned}[t]
-\log(\sip)-(\siz{\com}-1)\log(1-\sip)&=\siz{\com}H\bigg({\frac{1}{\siz{\com}}}\bigg)\\
&\leq\log(\siz{\com})+1
\end{aligned}$
\item $ \begin{aligned}[t]
-(\gtu-\nL-\nsw{\com})\log(\ww)-\nsw{\com}\log(1-\ww)&=(\gtu-\nL)\Bbine{\frac{\nsw{\com}}{\gtu-\nL}}\\
&\leq\nsw{\com}\log\left(\frac{\gtu-\nL}{\nsw{\com}}\right)+\nsw{\com}
\end{aligned}$
\item $ \begin{aligned}[t]
-\nsw{\com}\log(\sw)+((\siz{\com}-1)(\gtu-\nL)&-\nsw{\com})\log(1-\sw) \\
&=(\siz{\com}-1)(\gtu-\nL)\Bbine{\frac{\nsw{\com}}{(\siz{\com}-1)(\gtu-\nL)}}\\
&\leq\nsw{\com}\log\left(\frac{(\siz{\com}-1)(\gtu-\nL)}{\nsw{\com}}\right)+\nsw{\com} \\
\end{aligned}$
\end{enumerate}
\end{lemma}

\begin{proof}
All three items follow from Lemma \ref{bef1l} with $x$, $y$ and $z$ defined, for each item, as follows:
\begin{enumerate}
\item $z:=\sip$, $x:=1$ and $y:=\siz{\com}-1$
\item $z:=1-\ww$, $x:=\nsw{\com}$ and $y:=\gtu-\nL-\nsw{\com}$
\item $z:=\sw$, $x:=\nsw{\com}$ and $y:=(\siz{\com}-1)(\gtu-\nL)-\nsw{\com}$.
\end{enumerate}
\end{proof}

We are now ready  to prove Theorem \ref{tuneth}. 
Combining Theorem \ref{re1th} and Lemma \ref{bef2l}, we have both:
\begin{align}
\notag\siz{\com}\rel{\cop{}}{\pv{1}}&\leq\siz{\com}\log\left(\frac{\nex}{\siz{\com}}\right)+\nL\siz{\com}\Bbine{\frac{1}{\siz{\com}}}\\
\notag&+(\gtu-\nL)\Bbine{\frac{\nsw{\com}}{\gtu-\nL}}+(\siz{\com}-1)(\gtu-\nL)\Bbine{\frac{\nsw{\com}}{(\siz{\com}-1)(\gtu-\nL)}}
\end{align}
\begin{align}
\notag\siz{\com}\rel{\cop{}}{\pv{1}}&\leq \siz{\com}\log\left(\frac{\nex}{\siz{\com}}\right)+\nL(\log(\siz{\com})+1)\\
\notag&+\left(\nsw{\com}\log\left(\frac{\gtu-\nL}{\nsw{\com}}\right)+\nsw{\com}\right)+\left(\nsw{\com}\log\left(\frac{(\siz{\com}-1)(\gtu-\nL)}{\nsw{\com}}\right)+\nsw{\com}\right)
\end{align}
from which the result follows.\hfill \qed
\subsection{Regret for Switching Multitask Allocation with Long-term Memory}\label{marss}
We are now ready to prove our regret bound. Theorem \ref{allospecth} gives us:
$$\sum_{\gti=1}^{T}\musgt\inprod{\arv{\gti}-\copgti}{\arc{\gti}}\leq\sqrt{2\rel{\cop{}}{\pv{1}}\sum_{\gti=1}^T\musgt}$$
for the reduction in Subsection \ref{Shcsec}. Substituting in the equalities in Lemmas \ref{musize} and \ref{musize2} gives us:
$$\sum_{\gti=1}^{T}\frac{1}{\siz{\com}}(\inprod{\arv{\gti}}{\arc{\gti}}-\lo{\gti}{\sh{\gti}{}})\leq\sqrt{2\rel{\cop{}}{\pv{1}}\sum_{\gti=1}^T\frac{1}{\siz{\com}}}.$$
Rearranging gives:
$$\sum_{\gti=1}^{T}(\inprod{\arv{\gti}}{\arc{\gti}}-\lo{\gti}{\sh{\gti}{}})\leq\sqrt{2\siz{\com}\rel{\cop{}}{\pv{1}}T}\,.$$
Theorem \ref{equivth} then gives us:
$$\sum_{\gti=1}^{T}(\lot{\gti}-\lo{\gti}{\sh{\gti}{}})\leq\sqrt{2\siz{\com}\rel{\cop{}}{\pv{1}}T}\,.$$
Finally, we substitute in Theorem \ref{tuneth} and set $\dis \ge \siz{\com} $ and $\swi := \nsw{\com}$ to obtain that the regret, $\sum_{\gti=1}^{T}(\lot{\gti}-\lo{\gti}{\sh{\gti}{}})$ that is bounded above by:
\begin{equation}\label{eq:reb}
\sqrt{2\gtu\left(\dis\log\left(\frac{\nex}{\dis}\right)+\nL(\log(\dis)+1)+\swi\left(\log(\dis-1)+2\log\left(\frac{\gtu-\nL}{\swi}\right)+2\right)\right)}
\end{equation}
and also, more tightly, bounded above by:
\begin{equation}\label{eq:hreb}
\sqrt{2T}
\sqrt{
\dis\log\left(\frac{\nex}{\dis}\right)+\nL\dis\Bbine{\frac{1}{\dis}}
+(\gtu-\nL)\Bbine{\frac{\swi}{\gtu-\nL}}+(\dis-1)(\gtu-\nL)\Bbine{\frac{\swi}{(\dis-1)(\gtu-\nL)}}.
}
\end{equation}
Note that if we don't know $\siz{\com}$ we can use an upper bound $m$ in the algorithm as, trivially, we can add arbitrary experts to $\ran$.

\subsection{Reducing to Multitask Allocation}\label{redss}
We now reduce the finite hypothesis class setting to the Multitask Allocation model, which will prove Theorem \ref{thm:fin}. First, we let $\nex:=|\hfin|$, define a bijection $\ehb:\na{\nex}\rightarrow\hfin$, and define $\com$ by $\sh{t}{i}=\ehb^{-1}(h_t^i)$.  On trial~$\gti$ the Learner randomly draws $\ei^{\gti}\in\na{\nex}$  with probability $\al{t}{\ei^{\gti}}$ and predicts with $\hat{y}^{\gti}:=[\ehb(\ei^{\gti})](x^{\gti})$. We then define $\blo{\gti}$ by $\lo{\gti}{\ei}:=\lzo(y^{\gti},[\ehb(\ei)](x^{\gti}))$. 
\begin{proposition}We have the following equivalence,
\[ R_T = \sum_{\li\in\na{\nL}}\sum_{t\in\na{T^{\li}}} \E [\lzo(y^{\li}_t,\hat{y}^{\li}_t)] - \lzo(y^i_t,h^{\li}_t(x^{\li}_t))=\sum_{\li\in\na{\nL}}\sum_{t\in\na{T^{\li}}}\tbc{t}{\li}-\lc{t}{\li}{\sh{\li}{t}}\,. \]
\end{proposition}
\begin{proof}
We have:
\begin{align*} 
\lot{\gti}&=\bal{\gti}\cdot\blo{\gti}\\
&=\sum_{\ei\in\na{n}}\al{\gti}{\ei}\lo{\gti}{\ei}\\
&=\sum_{\ei\in\na{n}}\mathbb{P}(\ei^{\gti}=\ei)\lo{\gti}{\ei}\\
&=\sum_{\ei\in\na{n}}\mathbb{P}(\ei^{\gti}=\ei)\lzo(y^{\gti},[\ehb(\ei)](x_t))\\
&=\mathbb{E}\left(\lzo(y^{\gti},[\ehb(\ei^{\gti})](x_t))\right)\\
&=\mathbb{E}\left(\lzo(y^{\gti},\hat{y}^{\gti})\right)\\
\end{align*} 
Noting that $\lzo(y^i_t,h^{\li}_t(x^{\li}_t))=\lzo(y^i_t,[\ehb(\sh{\li}{t})](x^i_t))=\lc{t}{\li}{\sh{\li}{t}}$ then gives us the result.
\end{proof}
Theorem~\ref{thm:fin} follows from this proposition and the regret bounds in Equations~\eqref{eq:reb} and~\eqref{eq:hreb}. \hfill $\blacksquare$

\section{Proofs for Section~\ref{sec:ker}}
\label{app:ker}
We prove Theorem~\ref{thm:ker} and give a proof sketch of Proposition~\ref{prop:lb} in this section.  We first provide a brief overview of the proof of Theorem~\ref{thm:ker} and a discussion of Theorem~\ref{thm:itdit}, which is a key result in the proof of the theorem.

\subsection*{Sketch of Theorem~\ref{thm:ker} and Proof of Theorem~\ref{thm:itdit}}
In the proof, we give a reduction of Algorithm~\ref{alg:ker} to~\cite[Alg. 2]{\mcside} (\IMCSI).  Two necessary additional results that we need include Theorem~\ref{thm:itdit} and Corollary~\ref{cor:PTKT}.  In Corollary~\ref{cor:PTKT}, we bound a normalized margin-like quantity of the  multitask-path-tree kernel used in the algorithm.  Then in Theorem~\ref{thm:itdit}, we 
bound the quasi-dimension which indicates how to set the parameters of Algorithm~\ref{alg:ker} as well as determines the value of $C(\gbh)$ in the main theorem.  As this bound of the quasi-dimension is a key element of our proof, we contrast it to a parallel result proved in~\cite[Thm. 3]{\mcside}.

We recall that the regret (see~\cite[Thm. 1]{\mcside}) of \IMCSI\ is $\cOT(\sqrt{({\upD}/{\gamma^{2}}){T}})$ where $1/\gamma^2\ge\maxnormsqr{\bU}$ and $\upD\ge\qD(\bU)$.   We will prove in our setting $1/\gamma^2 = \dis\ge\maxnormsqr{\bU}$ in the discussion following~\eqref{eq:ra}.

We now contrast our bound on the quasi-dimension (Theorem~\ref{thm:itdit}) to the bound of~\cite[Thm.~3]{\mcside}. 

 The quasi-dimension depends on $\gamma$ so that if $\gamma\le\gamma'$, then  $\qD(\bU)\le\qDA(\bU)$.  In~\cite[Thm.~3]{\mcside} the given bound on quasi-dimension is independent of the value of $\gamma$.  Thus to minimize the regret bound of $\cOT(\sqrt{({\upD}/{\gamma^{2}}){T}})$, it is sensible in~\cite{\mcside} to select the smallest  possible $1/\gamma^2 = \maxnormsqr{\bU}$.  The situation in this paper is essentially reversed.  In the following theorem, it is required that 
 $1/{\marh^2}=m\ge\maxnormsqr{\bU}$.  In fact, $m$ is the maximum possible value of the squared max norm in the case that $\dis = |\dis(\gbh)|$ with respect to all possible comparators $\gbh$ (see~\eqref{eq:ra}).
 Thus in contrast to~\cite[Thm.~3]{\mcside},  our result trade-offs a potentially larger value in $1/\gamma^2$ for a smaller possible~$\upD$.
If we were instead to use the bound of~\cite[Thm 3]{\mcside}, then the term in this paper $\sum_{h\in \dis(\gbh)} \norm{h}^2_K X^2_K$ would gain a leading multiplicative factor of $\dis^2$ (terrible!).

We introduce the following notation. We recall the class of $m \times d$ {\em row-normalized} as $\RNM^{m,d} := \{\hat{\bP} \subset \Re^{m \times d} : \norm{\hat{\bP}_{i}} =1, i\in [m]\}$ and 
denote the class of {\em block expansion} matrices as
$\BEM^{m,d} := \{\bR \subset \{0,1\}^{m \times d} : \norm{\bR_i} =1\ \text{for}\ i\in [m],\operatorname{rank}(\bR) = d\}$.  Block expansion matrices may be seen as a generalization of permutation matrices, additionally duplicating rows (columns) by left (right) multiplication.  
\togglefalse{shrink}
The class of $(k,\ell)$-binary-biclustered matrices is defined as
\iftoggle{shrink}{ 
$
\bikl = \{\bU = \bR \Ustar \bC^{\trans} \in \Uset :  \bU^* \in \{-1,1\}^{k\times\ell}\,, \bR\in\BEM^{m,k}\,,\bC\in\BEM^{n,\ell}  \} \,.
$
}{
\begin{equation*}
\bikl = \{\bU = \bR \Ustar \bC^{\trans} \in \Uset :  \bU^* \in \{-1,1\}^{k\times\ell}\,, \bR\in\BEM^{m,k}\,,\bC\in\BEM^{n,\ell}  \} \,.
\end{equation*} 
}
\begin{theorem}\label{thm:itdit}
If  $\U\in\biml$, ${\marh}={1}/\sqrt{\dis}$  and if
\begin{equation}\label{eq:dit}
\DITF(\U) := 
\marh^2 \trc( ({\Ustar})^{\trans} \RN\Ustar)\RRN + \trc( \bC^{\trans} \CN \bC)\RCN
\end{equation}
is defined as the minimum over all decompositions of $\U = \Ustar \bC^{\trans}$ for
 $\Ustar \in \{-1,1\}^{\ld \times \dis}$ and $\bC\in \BEM^{T,\dis}$ then   
\begin{equation*}
\qD(\U)  \le \DITF(\U)\quad\quad ({\marh}={1}/\sqrt{\dis})\,.
\end{equation*}
\end{theorem}
\begin{proof}
Recall by supposition $\marh = 1/\sqrt{\dis}$.  Set ${\nP'} := \marh \Ustar$ and $\nQ' := \bC$ hence $\nP' \in \RNM^{\ld,\dis}$, $\nQ' \in \RNM^{T,\dis}$ and $\nP' \nQ'^{\trans} = \marh \U$.

Recall~\eqref{eq:defpdim},
\begin{equation}\label{eq:recab}
\qD(\bU) := \min_{\nP \nQ^\trans = {\marh}\U} \trc\left(\nP^{\trans}\RN\nP  \right) \RRN
+
\trc\left(\nQ^{\trans}\CN\nQ \right)\RCN\,.
\end{equation}
Observe that $(\nP',\nQ')$ is in the feasible set of the above optimization.  Hence
\begin{align*}
\qD(\bU) & \le \trc\left(\nP'^{\trans}\RN\nP'  \right)\RRN 
+ \trc\left(\nQ'^{\trans}\CN\nQ' \right)\RCN\\
& = \marh^2 \trc( ({\Ustar})^{\trans} \RN\Ustar)\RRN + \trc( \bC^{\trans} \CN \bC)\RCN\,.
\end{align*}
\end{proof}
\subsection*{Proof of Corollary~\ref{cor:PTKT}}
In this section, we prove Corollary~\ref{cor:PTKT}, which is utilized in the proof of Theorem~\ref{thm:ker}.

We recall the notions of {\em effective resistance} between vertices in a graph and the {\em resistance diameter} of a graph.
A graph may naturally interpreted as an resistive network where each
edge in the graph is viewed as a unit resistor. Thus the {\em effective 
resistance} between two vertices is the potential difference needed to induce a unit current flow between them and the {\em resistance diameter} is the maximum effective resistance between all pairs of vertices.

To prove the corollary, we will need to bound the diagonal element of the Laplacian pseudo-inverse by the resistance diameter. In the following Lemma, we will improve upon~\cite[Eq. (9)]{herbster2006prediction} by a factor of $\onehalf$ for the special case of fully complete trees.  \begin{lemma}
For the graph Laplacian $\bL \in \Re^{N \times N}$ of a fully complete tree graph,
\[\RAD_{\bL}  = \max_{i\in[N]} L^+_{ii} \leq \frac{1}{2} \RAD_{diam}(\bL),\]
where $R_{diam}(\bL)$ is the resistance diameter of the graph described by $\bL$.
\label{lem:Rdim}
\end{lemma}
\begin{proof}
Before proving the result, we shall recall 4 general facts about graphs, trees and Laplacians. We also denote the set of vertices at a given depth a {\em level.}  The root is at level 0.
\begin{enumerate}
\item The effective resistance between vertices $i$ and $j$ is given by (see \cite{Klein1993}),
\begin{equation}
\label{eq:Klein_eq}
 \RAD(i,j) =  L^+_{ii} +L^+_{jj} - 2 L^+_{ij}.  \end{equation}
\item The diagonal element of $L^+$ is given by  (see eg. \cite{herbster2009fast})
\begin{equation}
\label{eq:pred_tree}
L^+_{ii} = \frac{\RAD(i)}{N} - \frac{\RAD_{tot}}{N^2},
\end{equation}
where $\RAD(i) = \sum_{j=1}^N \RAD(i,j)$ and $\RAD_{tot} = \sum_{i,j<i} \RAD(i,j)$.
\item For fully complete trees, we have that $\RAD(i)= \RAD(j)$ and $L^+_{ii} = L^+_{jj}$ if $i$ and $j$ are in the same level due to symmetry.
\item For trees, the effective resistance between vertices $i$ and $j$ is given by the geodesic distance (path length) between the two vertices.
\end{enumerate}

Next, we prove the following intermediate result.
\begin{quote}
{\bf Lemma:} For a given vertex $i$, the vertex $j$ that minimizes $L^+_{ij}$ is the leaf vertex with the largest geodesic distance from $i$.
\begin{proof}
Define $h$ to be the height of the tree. 
We take vertex $i'$ to be at level $k \in [h-1]$ and vertex $j'$ at level $k+1$. Recalling that  $\RAD(i) = \sum_{j=1}^N \RAD(i,j)$, we will consider the individual summands that compose $\RAD(j')$ and $\RAD(i')$, given by the geodesic distances between $i'$ and $j'$ respectively and the other vertices due to fact 4. From fact 3 (with respect to the summands), we can assume without loss of generality that vertex $j'$ is the child of $i'$. Going from the summation of $\RAD(j')$ to the summation of $\RAD(i')$, there are 3 possible changes to the geodesic distances in the summation: \begin{enumerate} \item the descendants of $j'$ will have a geodesic distance reduced by 1
\item the geodesic distance between $i'$ and $j'$ remains constant \item all the other vertices will have a geodesic distance increased by 1. \end{enumerate} Hence, defining $\mathcal{D}_{j'}$ to be the set of descendants of node $j'$,
\begin{align*}
\RAD(j') &= \RAD(i') - \sum_{i \in \mathcal{D}_{j'}} 1 + \sum_{i' \in [N] \backslash \mathcal{D}_{j'} \cup \{j'\}} 1\\
 &=  \RAD(i') - |\mathcal{D}_{j'}| + N-(|\mathcal{D}_{j'}| +1)\\
 &= \RAD(i') +N-2|\mathcal{D}_{j'}| -1.
\end{align*}
This gives that $\RAD(j') - \RAD(i')\leq N$, and 
\begin{equation}
\label{eq:diam_next}
\frac{\RAD(j')-\RAD(i')}{N}\leq 1.
\end{equation} 
We show that vertex $j$ that minimizes $L^+_{ij}$ must be a leaf vertex by contradiction. Suppose $j$ is not a leaf vertex then there exists a child of $j\ne i$. Call the child $j'$ which thus  satisfies $\RAD(i,j') - \RAD(i,j)=1$. Hence, Equations~\eqref{eq:Klein_eq} and~\eqref{eq:pred_tree} give
  \begin{align}
    L^+_{ij'} - L^+_{ij}  &= \frac{1}{2} \left( \frac{\RAD(j')}{N} - \frac{\RAD(j)}{N}- \RAD(i,j') + \RAD(i,j) \right)\\
    & \leq 0, 
    \end{align}
where the inequality is due to~\eqref{eq:diam_next} for which we let $i'=j$. Hence, we have that $L^+_{ij} \geq L^+_{ij'}$ which is a contradiction.

Then, using Equations~\eqref{eq:Klein_eq} and~\eqref{eq:pred_tree}, we have
 \[\argmin_j L^+_{ij} = \argmin_j \frac{\RAD(j)}{N} - \RAD(i,j).\]
Since all leaf vertices have the same $\RAD(i)$, the leaf vertex that minimizes must be the one with the largest geodesic distance from $i$.
\end{proof}
\end{quote}

Recall that for a tree, the resistance diameter is equal to its geodesic diameter, and hence the vertices that maximize the effective resistance are given by the two leaf vertices with the largest geodesic distance. We therefore proceed by considering $i$ and $j$ to be any of the vertices that maximize the effective resistance, giving the resistance diameter. Due to fact 3, we have $L^+_{ii} = L^+_{jj}$. Then, from~\eqref{eq:Klein_eq}, we obtain,
\begin{align}
 \frac{1}{2} \RAD_{diam}(\bL) &=L^+_{ii}- L^+_{ij} \notag \\
 &= L^+_{ii}- \min_{k}  L^+_{ik} \label{eq:intermediate_result}\\
 &\geq L^+_{ii} -  \frac{1}{N} \sum_{k=1}^N L^+_{ik}\notag \\
 &\geq L^+_{ii} \label{eq:final_resdis}
 \end{align}
 where~\eqref{eq:intermediate_result} comes from the intermediate lemma, and~\eqref{eq:final_resdis} comes from the fact that  $\sum_{j=1}^N L^+_{ij}=0$ for all $i \in [N]$ since $\bL\bone = \bzero$ for connected graphs.
\end{proof}

The following Lemma is essentially a simplification of the argument in~\cite[Section 6]{herbster2009online} for Laplacians,

\begin{lemma}(See \cite[Section 6]{herbster2009online}.)\label{lem:PTK}
If $f\in\cH_\PTK\cap \{0,1\}^T$ then
\begin{align}
\max_{\tau\in [T]} \PTK(t,t) & \le  2 \lceil \log_2 T \rceil \notag\\
\norm{f}^2_{\PTK} \max_{t\in [T]} \PTK(t,t) & \le  \swi(f) \lceil \log_2 T \rceil^2 +2\label{eq:ppropb},, 
\end{align}
where $k(f) := \sum_{t=1}^{T-1} [f(t) \ne f(t+1)]$.
\end{lemma}
\begin{proof}
First we recall the following standard fact about the graph Laplacian $\bL$ of an unweighted graph $\cG=(V,E)$,
\[
\bu^{\trans} \bL \bu = \sum_{(i,j)\in E} (u_i - u_j)^2\,,
\]
 where $V$ is the set of vertices and $E$ is the set of edges in the graph.  Call this quantity the {\em cut} of the labeling $\bu$.  Consider a fully complete binary tree with a depth of $\ceil{\log_2 T}+1$.  For simplicity now assume that there are exactly $T$ leaf nodes, i.e., $\log_2 T\in \N$.  Assume some natural linear ordering\footnote{Given every three vertices in ordering $(a,b,c)$ we have that $d(a,b)\le d(a,c)$ where $d(p,q)$ is the path length between $p$ and $q$.} of the leaves.   This ordering then defines our {\em path}.  We call each set of vertices at a given depth a ``level'' and they inherit a natural linear ordering from their children.   Suppose that there are $n$ vertices at a given level $\ell$, and define $w^\ell_i :=u_{v^{\ell}_i}$, where $v^{\ell}_i$ is the $i$th vertex on level $\ell$. The {\em path-cut} at this level is given by $\sum_{i=1}^{n-1} |w^{\ell}_i - w^{\ell}_{i+1}|$.  
 
We now proceed to argue that for a given binary labeling of a path with associated path-cut $k(f)$, we can identify a (real-numbered) labeling of the tree, such that: a. the labeling of the tree leaves is binary and consistent with that of the path and b. the tree has a cut of no more than $\frac{1}{2} \swi(f) \ceil{\log T}$.
The  construction is as follows: {\em each parent inherits the average of the labels                                                                                                                                                                                                                                                                                                                                                                                                                                                                                                                                                                                                                                                                                                                                                                                                                                                                                                                                                                                                                                                                                                                                                                                                                                                                                                                                                                                                                                                                                                                                                                                                                                                                                                                                                                                                                                                                                                                                                                                                                                                                                                                                                                                                                                                                                                                                                                                                                                                                                                                                                                                                                                                                                                                                                                                                                                                                                                                                                                                                                                                                       of its children.} We make two observations about the constructed labeling: 
\begin{enumerate}
\item The path-cut at a higher level cannot be more than the level below. Consider two adjacent levels with the lower level $\ell$ having $n$ vertices. Denote the set of odd numbers that is a subset of $[n-1 ]$ as $I_{odd}$, and the set of even number that is a subset of $[n-2]$ as $I_{even}$.  Recall that the path-cut of the lower level is $\sum_{i=1}^{n-1} |w^\ell_i- w^{\ell}_{i+1}|$. This can be upper bounded as follows:
\begin{align}
&\sum_{i=1}^{n-1} |w^\ell_i- w^\ell_{i+1}|\notag \\
& = \sum_{i \in I_{odd}}  \left|w^\ell_i- \frac{w^\ell_{i}+w^\ell_{i+1}}{2}\right|  + \left| \frac{w^\ell_{i}+w^\ell_{i+1}}{2} -  w^\ell_{i+1}\right| + \sum_{i \in I_{even}}  |w^\ell_i- w^\ell_{i+1}|\notag   \\
&= \sum_{i \in I_{odd}}  \left|w^\ell_i- \frac{w^\ell_{i}+w^\ell_{i+1}}{2}\right|  + \left| \frac{w^\ell_{i}+w^\ell_{i+1}}{2} -  w^\ell_{i+1}\right| + \sum_{i \in I_{odd}\backslash\{n -1\}}  |w^\ell_{i+1} - w^\ell_{i+2}|\notag  \\
&\geq \sum_{i \in I_{odd}}  \left|w^\ell_i- \frac{w^\ell_{i}+w^\ell_{i+1}}{2}\right|  + \sum_{i \in I_{odd}\backslash\{n -1\}} \left| \frac{w^\ell_{i}+w^\ell_{i+1}}{2} -  w^\ell_{i+1}\right| + |w^\ell_{i+1} - w^\ell_{i+2}|\notag  \\
&\geq \sum_{i\in I_{odd}}  \left|w^\ell_i- \frac{w^\ell_{i}+w^\ell_{i+1}}{2}\right| + \sum_{i \in I_{odd} \backslash\{n-1\}}\left| \frac{w^\ell_{i}+w^\ell_{i+1}}{2} -  w^\ell_{i+2}\right| \label{eq:triangle_ineq1}\\
&  \geq  \sum_{i\in I_{odd}\backslash \{1\}}  \left|u_{i} - \frac{u_{i}+w^\ell_{i+1}}{2}\right|  +\sum_{i \in I_{odd} \backslash\{n-1\}} \left| \frac{w^\ell_{i}+w^\ell_{i+1}}{2} - w^\ell_{i+2}\right| \notag \\
 &=  \sum_{i\in I_{odd}\backslash \{n -1 \}}  \left|w^\ell_{i+2} - \frac{w^\ell_{i+2}+w^\ell_{i+3}}{2}\right|  + \left| \frac{w^\ell_{i}+w^\ell_{i+1}}{2} -  w^\ell_{i+2}\right|\notag \\
& \geq \sum_{i\in I_{odd}\backslash \{n -1 \}} \left| \frac{w^\ell_{i}+w^\ell_{i+1}}{2} - \frac{w^\ell_{i+2}+w^\ell_{i+3}}{2}\right| \label{eq:triangle_ineq2}
\end{align}
 where~\eqref{eq:triangle_ineq1} and~\eqref{eq:triangle_ineq2} follow from $|a-b| + |b-c| \geq |a-c|$ (triangle inequality). Observing that the R.H.S. of~\eqref{eq:triangle_ineq2} is the path cut of the upper level, we are then done.
\item If we denote the set of edges between two adjacent levels by $\tilde{E}$, we have that $\sum_{(i,j)\in \tilde{E}} (u_i - u_j)^2$ is at most half the path-cut of the lower level. This can be seen by considering the edges between a given parent $i$ and its two children $j$ and $j'$. Let us define $x$ as half the path cut due to the children, i.e. $\frac{1}{2} |u_j - u_{j'}|$. Since all labelings are in $[0,1]$, we have that $x\in [0,1/2]$. The cut made due to the parent and the children, i.e. $(u_i - u_{j})^2 + (u_i-u_{j'})^2$ is then given by $2x^2$. Using the inequality $x-2x^2 \geq 0 $ for $x\in [0, 1/2]$, and applying this to all the parents on the same level as vertex $i$, we then prove the statement.
\end{enumerate}
Hence, combining the two above observations and recalling that there are $\ceil{\log_2 T}+1$ levels (and therefore $\ceil{\log_2 T}$ transitions between the levels), we have that  
\[ \sum_{(i,j)\in E}(u_i - u_j)^2 \leq \sum_{\ell=1}^{\ceil{\log_2 T}} \frac{1}{2} p(\ell) \leq \sum_{k=1}^{\ceil{\log_2 T}} \frac{1}{2}k(f),\] where $p(\ell)$ is the path-cut of the tree at level $\ell$, the first inequality is due to observation 2, and the second inequality is due to observation 1.  Hence  we have shown our premise that the cut is upper bounded by $ \frac{1}{2}k(f) \ceil{\log_2 T}$.  Observe that our premise still holds if there are more than $T$ leaf nodes, as we can treat any additional leaves on the bottom level as being labeled with the last label on their level; thus the cut will not increase.  Hence we have  shown the following inequality where $\bL$ is the Laplacian of a fully complete binary tree with $N$ vertices and a path of $T$ leaves labeled by an $f\in\{0,1\}^{[T]}$.
\begin{equation}\label{eq:ia}
\min_{\bu\in\R^{[N]} : u_{i} = f(i), i\in [T]} \bu^{\trans} \bL \bu \le \frac{1}{2} \swi(f) \ceil{\log_2 T}\,.
\end{equation}
We next observe that
\begin{equation}\label{eq:ib}
\RAD_{\bL} \le \ceil{\log_2 T}\,.
\end{equation}
This follows from Lemma~\ref{lem:Rdim}, where $\RAD_{\bL} = \max_{i\in [N]} L^+_{ii}$ is bounded by half the resistance diameter, which is then just bounded by half the geodesic diameter. 
Furthermore if $\bL$ is the Laplacian of a connected graph and
$\bLb := \la+\left(\frac{\one}{\m}\right)\left(\frac{\one}{\m}\right)^\trans\RAD_{\bL}^{-1}$ then if $\bu\in [-1,1]^m$ we have
\begin{align*}
\RAD_{\bL^{\circ}} & = 2 \RAD_{\bL}\,, \\
\bu^{\trans} \bLb \bu & \le \bu^{\trans} \bL \bu + \frac{1}{ \RAD_{\bL}}\,.
\end{align*}
Thus combining the above with~\eqref{eq:ia} and~\eqref{eq:ib}, we have, 
\begin{align*}
\RAD_{\bL^{\circ}} & \le 2\ceil{\log_2 T}\,, \\
\min_{\bu\in\R^{[N]} : u_{i} = f(i), i\in [T]} (\bu^{\trans} \bLb \bu)  \RAD_{\bL^{\circ}}& \le   \swi(f) \ceil{\log_2 T}^2+  2\,,
\end{align*}
which proves the Lemma.
\end{proof}
Observe that the left hand side in~\eqref{eq:ppropb} is up to constant factors, the normalized margin of $f$ in the sense of Novikoff's Theorem~\cite{\novikoff}.  The construction is somewhat counterintuitive as one may expect that one can use a path graph directly in the construction of the kernel. However, then $\max_{t\in [T]} \PTK(t,t) \in \Theta(T)$ which would lead to a vacuous regret bound.   Also one may wonder if one can reduce the term $(\log T)^2$ while maintaining a linear factor in $\swi(f)$.  In fact the term $(\log T)^2$ is known~\cite[Theorem 6.1]{Forster2001} to be required when $\swi(f) =1$.

As a straightforward corollary to Lemma~\ref{lem:PTK}, we have 
\begin{corollary}\label{cor:PTKT}
If $f\in\cH_{\PTKT}\cap \{0,1\}^T$ then
\begin{equation}\label{eq:pprop}
\norm{f}^2_{\PTKT} \max_{\tau\in [T]} \PTKT(\tau,\tau) \le  (\swi(f)+\nL(f))  \ceil{\log_2 T}^2 +2\,,
\end{equation}
where $\PTKT = \PTKTT$, $\swi(f) := \sum_{i=1}^{\nL} \sum_{t=1}^{T^i-1} [f(\tmit) \neq f(\tmits)]$ and $\nL(f) := \sum_{i=1}^{s-1}  [f({{\smash{{}^i_{T^i}}}})\neq f({{\smash{{}^{i+1}_{1}}}})]$.
\end{corollary}
\begin{proof}
Since each task is laid out contiguously along the bottom layer, we pay the path-cut for each task individually and we pay $\nL(f)$ for the intertask boundaries. 
\end{proof}
\subsection*{Proof of Theorem~\ref{thm:ker}}
We first recall some of the notation introduced earlier in the section. The {\em block expansion} matrices are defined as
$\BEM^{m,d} := \{\bR \subset \{0,1\}^{m \times d} : \norm{\bR_i} =1\ \text{for}\ i\in [m],\operatorname{rank}(\bR) = d\}$.  
The class of $(k,\ell)$-binary-biclustered matrices is defined as
$
\bikl = \{\bU = \bR \Ustar \bC^{\trans} \in \Uset :  \bU^* \in \{-1,1\}^{k\times\ell}\,, \bR\in\BEM^{m,k}\,,\bC\in\BEM^{n,\ell}  \} \,.
$
Next, we recall Theorem~\ref{thm:ker} and provide a proof.

{\bf Theorem~\ref{thm:ker}.}
{\em 
The expected regret of Algorithm~\ref{alg:ker} with upper estimates, $\swi \ge \swi(\gbh)$, $\dis \ge |\dis(\gbh)|$,
\[
\CIH\ge \CI(\gbh) := \left(\sum_{h\in \dis(\gbh)} \norm{h}^2_K X^2_K + 2(\nL+{\swi}-1)\dis \lceil\log_2 T\rceil^2+2m^2\right)\,,
\]
$\KRH\ge\max_{\gti\in[T]} K(\xtau,\xtau)$, and learning rate $\lr = \sqrt{\frac{\CIH \log(2T) }{2T\dis}}$ is bounded  by
\begin{equation}
\sum_{i=1}^\nL \sum_{t=1}^{T^i} \Exp[\lzo(\yl,\yhl)] - \lzo(\yl,h^i_t(x^i_t)) \le 4\sqrt{2\CIH\, T\log(2T)}
\end{equation}
with received instance sequence $\bxx \in\cX^T$ for any $\gbh \in{\hKbin}^T$.
}
\begin{proof}
Algorithm~\ref{alg:ker} is the same as ~\cite[Algorithm 2]{\mcside} (which we call \IMCSI) except for some redefinitions in the notation.  For convenience we recall \IMCSI\ below\footnote{Since we are only concerned with regret bound we have set the parameter $\mbox{\bf NON-CONSERVATIVE} =1$ in our restating of the algorithm.}

\begin{quote}
\begin{algorithmic} 
\renewcommand{\algorithmicrequire}{\textbf{Algorithm 2 (\cite{\mcside}).}}
\REQUIRE
\renewcommand{\algorithmicrequire}{\textbf{Parameters:}} 
\REQUIRE Learning rate: $0<\lr$ \, quasi-dimension estimate: $1\le \scp$, margin estimate: $0 <\marh\leq 1$ and side-information kernels $\RNfunc:\mathcal{I}\times\mathcal{I}\rightarrow\Re$, $\CNfunc:\mathcal{J}\times\mathcal{J}\rightarrow\Re$,  
with $\IRRN := \max_{i\in\mathcal{I}} \RNfunc(i,i)$ and $\IRCN := \max_{j\in\mathcal{J}} \CNfunc(j,j)$,
and maximum distinct rows $m$ and columns $n$, where $m+n \geq 3$.
\\  \vspace{.1truecm}
\renewcommand{\algorithmicrequire}{\textbf{Initialization:}}   \vspace{.1truecm} \vspace{.1truecm}
\REQUIRE $\nmset \leftarrow \emptyset\,,\uset \leftarrow \emptyset\,, \cI^1 \leftarrow \emptyset\,, \,\cJ^1 \leftarrow \emptyset\, \,.$   \vspace{.1truecm}
\renewcommand{\algorithmicrequire}{\textbf{For}}
\REQUIRE $t =1,\dots,T$  \vspace{.1truecm}
\STATE $\bullet$ Receive pair $(i_t,j_t) \in \mathcal{I} \times \mathcal{J}.$ \\ \vspace{.1truecm}
\STATE $\bullet$ Define 
\begin{align*}
&\quad(\RN^t)^+   := (\RNfunc(i_r,i_s))_{r,s\in \cI^t\cup \{i_t\}}\,; \quad (\CN^t)^+  := (\CNfunc(j_r,j_s))_{r,s\in \cJ^t\cup \{j_t\}}\,,\\
& \quad \XT(s) := \con{\frac{(\sqrt{(\RN^t)^+})\be^{i_s}}{\sqrt{2\IRRN}}}{\frac{(\sqrt{(\CN^t)^+})\be^{j_s}}{\sqrt{2\IRCN}}} \con{\frac{(\sqrt{(\RN^t)^+})\be^{i_s}}{\sqrt{2\IRRN}}}{\frac{(\sqrt{(\CN^t)^+})\be^{j_s}}{\sqrt{2\IRCN}}}^{\trans}\,, \\
&\quad \wem{t} \leftarrow \exp \left(\log\left(\frac{\scp}{m+n}\right) \id^{|\cI^t|+|\cJ^t| +2} + \sum_{s\in\uset}  \lr\y{s}  \XT(s)\right)\,.\vspace{-.17in}
\end{align*}

\STATE $\bullet$  Predict 
\begin{equation*} \Yrv \sim \mbox{\sc Uniform}(-\marh,\marh) \,;\ybt \leftarrow \tr{\wem{t}\XT}-1 \,;\quad\yht \leftarrow \sign(\ybt-\Yrv)\,.\vspace{-.17in}
\end{equation*}
\STATE $\bullet$ Receive label $\y{t} \in \{-1,1\}$\,.\vspace{.1truecm}
\STATE $\bullet$ If $y_t\ybt \leq  \marh$ then 
\begin{equation*}
\uset \leftarrow \uset \cup \{t\}\,,\ \ \cI^{t+1} \leftarrow \cI^{t} \cup \{i_t\}, \text{ and } \cJ^{t+1} \leftarrow \cJ^{t} \cup \{j_t\}\,.\vspace{-.17in}
\end{equation*}
\STATE $\bullet$  Else $\cI^{t+1} \leftarrow \cI^{t}$ and $\cJ^{t+1} \leftarrow \cJ^{t}$\,.
\end{algorithmic}
\end{quote}
The following table summarizes the notational changes between the two algorithms.
\begin{center}
\begin{tabular}{l|c|c}
Description  & \IMCSI\  & Algorithm 2 \\ \hline
Row space & $\cI$  & $\cX$ \\
Column space & $\cJ$  & $[T]$  \\
Row kernel & $\RNfunc$  & $K$ \\
Column kernel & $\CNfunc$  & $\PTK := \PTKT^{\tv,T^1,\ldots,T^{\nL}}$  \\
Row squared radius & $\IRRN$  & $\KRH$ \\
Column squared  radius& $\IRCN$  & $\KPH$ \\
Margin estimate&  $\gamma^{-2}$ & $\dis$ \\
Complexity Estimate & $\scp \gamma^{-2}$ & $\CIH$ \\
Dimensions\footnote{Note $T$ is an upper bound known in advance for numbers rows.  We will use $\ld$ to denote the number of distinct $x$ values seen over the algorithm.} & $m,n$ & $T,T$ \\
Time & $t$ & $\tau$ \\
Instance&  $(i_t,j_t)$  & $(\xtau,\tau)$
\end{tabular}
\end{center}
We now recall the following regret bound for \IMCSI.
\begin{quote}
{\bf Theorem 1} (~\cite[Theorem 1/Proposition 4]{\mcside})
{\em The expected regret of~\cite[Algorithm~2]{\mcside}  with  parameters $\mar \in (0,1]$, $\upD \geq \qD(\bU)$ ,
$\lr = \sqrt{\frac{\upD \log(m+n) }{2 T}}$, p.d. matrices $\RN\in \SPDM^m$ and $\CN\in \SPDM^n$  is bounded  by
\begin{equation}\label{eq:baseregret}
\sum_{t\in [T]} \Exp[\lzo(y_t,\yht)] - \lzo(y_t,U_{i_t j_t}) \le  4 \sqrt{2 \frac{\upD}{\mar^2}\log(m+n) T}
\end{equation}
for all $\bU\in \sm$ with $\maxnorm{\bU} \le 1/\mar$.
}
\end{quote}
We introduce the following notation: the matrix $\bH := (h^{\tau}(x): {x\in\cXfin,\tau\in [T])}$, the set $\cXfin := \cup_{\tau\in [T]} \{\xtau\}$, the matrices $\RKD = [(K(x,x') : {x,x'\in\cXfin})]^{-1}$, and $\CKD = [(\PTKT(\tau,\upsilon) :{\tau,\upsilon\in [T]})]^{-1}$.

Initially we note that we very minorly extend the algorithm and thus its analysis~\cite[Theorem 1]{\mcside} in so far as  we use the upper bounds and $\KRH\ge \RAD_{\RKD}$ and $\KPH\ge \RAD_{\CKD}$.

It now remains that in order to complete the reduction of Theorem~\ref{thm:ker} to~\cite[Theorem 1]{\mcside} we need to demonstrate the following two inequalities,
\begin{align}
\maxnorm{\bH} & \le \sqrt{\dis} \label{eq:ra} \\
\qDH(\bH) & \le \frac{1}{\dis} \CI(\gbh)\label{eq:rb}\,. 
\end{align}

First we show~\eqref{eq:ra}.  
Initially we derive the following simple inequality,
\begin{equation}\label{eq:maxdb} 
\maxnorm{\U}\le \min(\sqrt{m},\sqrt{n})\,, 
\end{equation}
which follows since we may decompose $\U = \U \bI^n$ or as $\U = \bI^m \U$.  
Let $\ld=|\cXfin|$.  Recall by definition $\dis \ge |\dis(\gbh)|$ and thus there are only at most $m$ distinct columns which implies $\bH = \bI^p  \Hstar  \bC^{\trans}$ where $\Hstar \in \{-1,1\}^{\ld \times \dis}$ and $\bC\in \BEM^{T,\dis}$ hence $\bH\in\biml$.  We now show, 
\begin{equation}\label{eq:iia}
\maxnorm{\Hstar} \ge \maxnorm{\bH}\,.
\end{equation}
For every factorization $\Hstar = \bP^* {\bQ^*}^\trans$ there exists a factorization $\bH = \bP^* ({\bQ^*}^\trans\bC^{\trans})$ for some $\bC\in \BEM^{T,\dis}$ such that 
\[
\max_{1\leq i\leq \ld} \vn{\bP^*_{i}} ~\max_{1\leq j\leq \dis}\vn{\bQ^*_{j}} =  
\max_{1\leq i'\leq \ld} \vn{\bP^*_{i'}} ~\max_{1\leq j'\leq T}\vn{(\bC\bQ^*)_{j'}}
\,,
\]
since for every row vector  $\bQ^*_{j}$ there exists a row vector $(\bC\bQ^*)_{j'}$ such that $\bQ^*_{j}=(\bC\bQ^*)_{j'}$ and vice versa.  Therefore since the max norm (see~\eqref{eq:maxnorm}) is the minimum over all factorizations we have shown~\eqref{eq:iia}.  Since $\Hstar\in \{-1,1\}^{\ld \times \dis}$ we have $\maxnorm{\bH} \le \maxnorm{\Hstar}\le \min(\sqrt{\ld},\sqrt{\dis})$ by~\eqref{eq:maxdb} and thus we have demonstrated~\eqref{eq:ra}.

We now show~\eqref{eq:rb}.  
We  recall the following useful equality,
\begin{equation}\label{eq:qd}
\bu^{\trans} \bK^{-1} \bu = \argmin_{f\in\cH_K : f(x) = u_x : x\in X} \norm{f}^2_K\,.
\end{equation}
where $\bK = (K(x,x'))_{x,x'\in X}$, $\bu \in \Re^X$ and $\bK$ is invertible and $K$ is a kernel.
By Theorem~\ref{thm:itdit} we have
\[
\qDH(\bH) \le \frac{1}{\dis}  \trc( ({\Hstar})^{\trans} \RKD \Hstar)\KRH + \trc( \bC^{\trans} \CKD  \bC)\KPH
\]
where $\bH = \Hstar  \bC^{\trans}$ with $\Hstar:=(h(x))_{x\in \cXfin,h\in\dis(\gbh)}$ and $\bC:=([{h}^\tau=h])_{\tau\in [T],h\in\dis(\gbh)}$ (note $\bC\in \BEM^{T,\dis}$).  

Simplifying and using~\eqref{eq:qd} we have,
\begin{equation}
\qDH(\bH) \le \frac{1}{\dis}  \sum_{h\in \dis(\gbh)} \norm{h}^2_K  \KRH + \trc( \bC^{\trans} \CKD  \bC)\KPH\,.
\label{eq:before_twice}
\end{equation}
From~\eqref{eq:qd} we have,
\begin{equation}\label{eq:iaa}
\trc( \bC^{\trans} \CKD\  \bC)
= \sum_{h\in\dis(\gbh)} \bc_h^\trans \CKD \bc_h = \sum_{h\in\dis(\gbh)} \norm{f_h}^2_{\PTKT} 
\end{equation}
where  $\bc_h$ is the column vector formed by taking the $h^{th}$ column of C. The vector $\bc_h \in \{0,1\}^T$ indicates if hypothesis $h$ is ``active'' on trial $\tau$, i.e., $c^\tau_h = [\htau = h]$.
Next we define $f_h(\tau):= c^\tau_h$ for $\tau=1,\cdots,T$. Recalling $\tau \equiv {{\smash{{}^{\ell^\tau}_{\sigma(\tau)}}}}$, we also have $f_h(\tau)\equiv f_h \left({{\smash{{}^{\ell^\tau}_{\sigma(\tau)}}}}\right)$.

From~\eqref{eq:iaa} and Corollary~\ref{cor:PTKT} we have,
\begin{align}
\trc( \bC^{\trans} \CKD  \bC) \KPH
& = \sum_{h\in \dis(\gbh)} \norm{f_h}_{\PTKT} \KPH \\
 &\le\sum_{h\in \dis(\gbh)} \left(\swi(f_h)+s(f_h)) \ceil{\log_2 T}^2 +2\right)  \notag \\
&\le\sum_{h\in \dis(\gbh)} ({k}(f_h)+{s}(f_h)) \ceil{\log_2 T}^2 +2 \dis(\gbh) \notag\\
&\leq\sum_{h\in \dis(\gbh)} ({k}(f_h)+{s}(f_h)) \ceil{\log_2 T}^2 +2m \notag \\
&\le 2(s+k-1) \ceil{\log_2T}^2 + 2m \label{eq:twice}
\end{align}
where 
\[ {k}(f)= \sum_{i=1}^s \sum_{t=1}^{T^i-1} [f({{\smash{{}^i_t}}}) \neq f({{\smash{{}^i_{t+1}}}})]\, , \,
 {s}(f)=\sum_{i=1}^{s-1}  [f({{\smash{{}^i_{T^i}}}})\neq f({{\smash{{}^{i+1}_{1}}}})]\,,\]
and where \eqref{eq:twice} comes from using $\sum_{h\in\dis(\gbh)} {k}(f_h) = k(\gbh) \le 2 k$ and  
 $\sum_{h\in\dis(\gbh)} {s}(h)  \le 2(s-1)$, where the factors of two are due to each switch of $f_h$ on successive time steps as well as intertask boundaries being counted twice.

Substituting \eqref{eq:twice} into \eqref{eq:before_twice}, we have
\[
\qDH(\bH) \le \frac{1}{\dis}  \sum_{h\in \dis(\gbh)} \norm{h}^2_K  \KRH  + 2 (\nL+{\swi}-1) \ceil{\log_2 T}^2 + 2 m\,, 
\]  This demonstrates~\eqref{eq:rb} thus completing the reduction. 
\end{proof}
\subsection*{Proof Sketch of Proposition~\ref{prop:lb}}
First we recall and then give a proof sketch of Proposition~\ref{prop:lb}.

{\bf Proposition~\ref{prop:lb}.}\ 
{\em  For any (randomized) algorithm and any $\nL,\swi,\dis,\Gamma\in \N$, with $\swi +\nL \ge \dis> 1$ and $\Gamma \ge \dis \log_2 \dis$, there exists a kernel $K$ and a $T_0\in\mathbb{N}$ such that for every $T\ge T_0$:
\[
\sum_{\tau=1}^{T} \Exp[\lzo(\ytau,\yhtau)] -\lzo(\ytau,\htau(\xtau)) \in  \Omega\left(\sqrt{ \left(\Gamma +\nL \log\dis + \swi \log \dis \right)T}\right)\,,
\]
for some multitask sequence $(x^1,y^1),\ldots,(x^T,y^T)\in(\cX\times\{-1,1\})^T$ and some $\gbh\in[{\hKbin}]^T$ such that $\dis \ge |\dis(\gbh)|$, $\swi \ge \swi(\gbh)$, $\sum_{h\in \dis(\gbh)} \norm{h}^2_K X^2_K\ge |\dis(\gbh)| \log_2 \dis$, where $X^2_K=\max_{\tau\in [T]} K(\xtau,\xtau)$.}

{\em Proof Sketch.}
We recall the following online learning terminology.   A sequence of examples $(x_1,y_1),\ldots,(x_T,y_T)$ is {\em realizable} with respect to a hypothesis class $\cH$ if there exists an $h\in\cH$, such that $\sum_{t=1}^T \lzo(y_t,h(x_t)) =0$.
The {\em optimal mistake bound} ($\text{Ldim}(\cH)$) with respect to a hypothesis class~$\cH$ also known as the {\em Littlestone dimension}~\cite{\littleopt,\msktoregret} is, informally speaking, the minimum over all deterministic learning algorithms, of the maximum over all realizable example sequences  of the number of mistaken predictions.

We will apply the following useful result~\cite[Lemma 14]{\msktoregret} which we quote below for convenience,
\begin{quote}
{\bf Lemma 14} (Lower Bound). {\em Let $\cH$ be any hypothesis class with a finite $\text{Ldim}(\cH)$. For any (possibly randomized) algorithm, exists a sequence $(x_1, y_1), \ldots , (x_T , y_T )$ such that }
\[ 
\Exp \sum_{t=1}^T \lzo(y_t,\hat{y}_t) - \min_{h\in\cH} \lzo(y_t,h(x_t)) \ge \sqrt{\frac{\text{Ldim}(\cH) T}{8}}\,.
\]
\end{quote}
In essence, this allows one to go from a lower bound on mistakes in the {\em realizable} case  to a lower bound in the non-realizable case.  However, Lemma~14 only applies directly to the standard  single-task model. To circumvent this, we recall as discussed in Section~\ref{sec:formalmodel}, that the switching multitask model may be reduced to the single-task model with a domain $\cX' = \cX\times [T]\times [s]$ and hypothesis class $\cH'$.  Therefore a lower bound in the switching multitask model with respect to $\cH$ implies a lower bound in the single-task non-switching case for $\cH'$ via the reduction.  
There are some slight technical issues over the fact that ``time'' is now part of the domain $\cX'$ and thus e.g., valid example sequences cannot be permuted.  We gloss over these issues in this proof sketch noting that they do not in fact impact our arguments.   The argument proceeds by demonstrating that there exists for any $\nL,\swi,\dis,\Gamma\in \N$   a kernel~$K$ and a realizable multitask sequence $(x^1,y^1),\ldots,(x^T,y^T)$ for which
\begin{equation}\label{eq:lbm}
\sum_{\tau=1}^{T} \lzo(\ytau,\yhtau) \in  \Omega\left(\Gamma +\nL \log\dis + \swi \log \dis \right)\,,
\end{equation}
where $\bxx \in\cX^T$, $X^2_K=\max_{\tau\in [T]} K(\xtau,\xtau)$, $\Gamma \ge \sum_{h\in \dis(\gbh)} \norm{h}^2_K X^2_K\ge \dis \log_2 \dis$, $\swi \ge \swi(\gbh)$, $\dis \ge |\dis(\gbh)|$ and $\swi+\nL  \ge \dis>1$.  After demonstrating that there exists such an example sequence we can apply~\cite[Lemma 14]{\msktoregret} to demonstrate the proposition.  Since the lower bound is in the form $\Omega(P+Q+R)$ which is equivalent to $\Omega(\max(P,Q,R))$, we may treat $P$, $Q$ and $R$, independently to prove the bound. Before we treat the individual cases, we give a straightforward result for a simplistic hypothesis class.

Define $\cX_d := [d]$ and $\cH_d := \{-1,1\}^d$ (i.e., the set of functions that map $[d]\into\bset$).  Observe that $\text{Ldim}(\cH_d) =d$, as an algorithm can force a mistake for every component and then no more.  Also, observe that if we define a kernel $K_d(x,x') := 2[x=x']-1$ over the domain $\cX_d$ that $\cH_d = \cH^{([d])}_{K_d}$,  $\max_{x\in[d]} K_d(x,x)=1$ and that $\norm{h}^2_{K_d} =d$ for all $h\in\cH_d$.  Finally, note that if $m=|\cH_d|$ then $\sum_{h\in\cH_d} \norm{h}^2_{K_d} X^2_{K_d} = m \log_2 m $.

We proceed by sketching an adversary for each of the three cases.

Case $\Gamma$ is the max.

To force $\Gamma$ mistakes, we choose $K=K_d$ and set $d=\Gamma/\dis$  and without loss of generality assume that~$d$ is an integer and recall that $\swi +\nL \ge \dis$.   Since $\ldim(\cH_d)=d$, an adversary may force $d$ mistakes within a single task in the first $d$ trials.  This strategy may repeated $\swi$ more times within a single task thus forcing $(\swi+1)d$ mistakes.  If $\swi+1\ge \dis$, we are done.  Otherwise, the constraint $\swi+\nL\ge \dis$ implies that we may force $d$ mistakes per task in $\dis-(\swi+1)$ other tasks.
Thus after $md$ trials, $md=\Gamma$ mistakes have been forced while maintaining the condition $m\ge |\dis(\gbh)|$.

Case $\swi \log_2\dis$ is the max.

Set $d=\log_2\dis$  and without loss of generality assume $d$ is positive integer.  Using $\cH_d$ we force $k d$ mistakes by first forcing $d$ mistakes within a single task then ``switching'' $\swi-1 $ times forcing $k d = k \log_2 m$ mistakes, while maintaining the conditions $m\ge |\dis(\gbh)|$ and $k\ge \swi(\gbh)$.

Case $\nL \log_2\dis$ is the max.   Same instance as the above case, except we force $d$ mistakes per task.\hfill $\qed$

\section{Proofs and Details for Section~\ref{sec:formalmodel}}
For the reader's convenience, we collect some standard well-known online learning results or minor extensions thereof in this appendix.
\subsection{Proof the \MW\ Bound}
The algorithm and analysis corresponds essentially to the classic weighted majority algorithm introduced in~\cite{\wmref}.
In the following, we will denote $|\hfin|$ as $n$. We introduce the MW algorithm and give the corresponding regret. \begin{algorithm}\caption{MW Algorithm}
\begin{algorithmic}
\renewcommand{\algorithmicrequire}{\textbf{Parameters:}} 
\REQUIRE Learning rate $\lr$; finite hypothesis set $\{h^1,\ldots,h^n\}=\hfin \subset \{-1,1\}^\cX$ \vspace{.1truecm}
\renewcommand{\algorithmicrequire}{\textbf{Initialization:}}
\REQUIRE Initialize $\bmv_1 = \frac{1}{n}\bm{1}^n$  \vspace{.1truecm}
\renewcommand{\algorithmicrequire}{\textbf{For}}
\REQUIRE $t =1,\dots,T$  \vspace{.1truecm}
\STATE $\bullet$ Receive instance $\bxx_t \in \mathcal{X}$.
\STATE $\bullet$ Set $\bh_t = (h^1(\bxx_t) \ldots h^n(\bxx_t)) \in \{-1,1\}^n$.
\STATE  $\bullet$ Predict 
\begin{equation*} i_t \sim v_t \,;\  
  \yht\! \leftarrow\!  h^{i_t}_t \,.\vspace{-.2in}
\end{equation*}
\STATE $\bullet$ Receive label $y_{t} \in \{-1,1\}$\,.\vspace{.1truecm}
\STATE $\bullet$ Update 
\begin{align*}
 &\lv_t  \leftarrow \onehalf | \bh_t - \yht \bm{1}| \\
 &\bw_{t+1}  \leftarrow  \bw_{t} \odot \exp\left(-\lr \lv_t \right) \\
 &\bmv_{t+1}  \leftarrow  \frac{\bw_{t+1}}{\sum^n_{i=1} w_{t+1,i }} 
\end{align*}
\end{algorithmic}
\label{alg:mw}
\end{algorithm}

\begin{theorem}
For Algorithm~\ref{alg:mw}, setting $\lr= \sqrt{(2 {\log n})/{T}}$ 
\begin{equation}
\sum_{t=1}^T \E [\lzo(y_t,\yht)] - \lzo(y_t,h(x_t)) \leq \sqrt{2 \log(n) T} 
\end{equation}
for any $h\in\hfin$.
\end{theorem}
\begin{proof}
Recalling that $\lv_t = \frac{| \bh_t - \yht \bm{1}|}{2}$, we have that $\bmv_t \cdot \lv_t = \E [\lzo(y_t,\yht)]$ and that $\be^i \cdot \lv_t = \lzo(y_t,h^i(x_t))$. In what follows, we will therefore bound $\bmv_t \cdot \lv_t -\be^i \cdot \lv_t $.
We first prove the following ``progress versus regret'' inequality.
\begin{equation}\label{eq:penhed}
\bmv_t \cdot \lv_t - \be^i \cdot \lv_t \le \frac{1}{\eta} \left(d(\be^i,\bmv_t) - d(\be^i,\bmv_{t+1})\right) + \frac{\eta}{2}  \sum_{i=1}^n v_{t,i} \ell^2_{t,i} .
\end{equation}
Let $Z_t := \sum_{i=1}^n v_{t,i} \exp(-\eta \ell_{t,i})$.  Defining $d(\bu,\bmv)$ as the relative entropy between $\bu$ and $\bmv \in \Delta_n$, observe that from the algorithm 
\begin{align}
d(\be^i,\bmv_t) - d(\be^i,\bmv_{t+1}) & = \sum_{j=1}^n e^i_j \log \frac{v_{t+1,j}}{v_{t,j}} \notag\\
& = -\eta \sum_{j=1}^n e^i_j  \ell_{t,j} -  \log Z_t \notag\\
& = -\eta  \be^i  \cdot \lv_t - \log \sum_{i=1}^n v_{t,i} \exp(-\eta \ell_{t,i}) \notag\\
& \ge -\eta  \be^i  \cdot \lv_t -  \log \sum_{i=1}^n v_{t,i} (1 - \eta \ell_{t,i} +\frac{1}{2}\eta^2 \ell^2_{t,i})\label{eq:paa} \\
& = -\eta  \be^i \cdot \lv_t -  \log (1 - \eta \bmv_t \cdot \lv_t +\frac{1}{2}\eta^2\sum_{i=1}^n v_{t,i} \ell^2_{t,i})  \notag\\
& \ge \eta ( \bmv_t \cdot \lv_t -  \be^i  \cdot \lv_t ) -\frac{1}{2}\eta^2\sum_{i=1}^n v_{t,i} \ell^2_{t,i} \label{eq:pbb}
\end{align}
using inequalities $e^{-x} \le 1 - x + \frac{x^2}{2} \text{ for } x\ge 0$ and $\log(1+x) \le x$ for~\eqref{eq:paa} and~\eqref{eq:pbb} respectively.

Summing over $t$ and rearranging we have 
\begin{align}
\sum_{t=1}^m   \left( \bmv_t \cdot \lv_t -  \be^i \cdot \lv_t\right) & \le \frac{1}{\eta} \left(d( \be^i,\bmv_1) - d( \be^i,\bmv_{m+1})\right)+\frac{\eta}{2} \sum_{t=1}^T\sum_{i=1}^n v_{t,i} \ell^2_{t,i} \notag\\
& \le \frac{\log n}{\eta}  +\frac{\eta}{2}\sum_{t=1}^T\sum_{i=1}^n v_{t,i} \ell^2_{t,i} \label{eq:hedgefull} 
\end{align}
where \eqref{eq:hedgefull} comes from noting that $d(\bu,\bmv_1) \le \log n$, $-d(\bu,\bmv_{m+1}) \le 0$, and  $\sum_{t=1}^T\sum_{i=1}^n v_{t,i} \ell^2_{t,i} \le T$.  Finally we substitute the value of $\eta$ and obtain the theorem.
\end{proof}
\subsection{Review of Reproducing Kernel Hilbert Spaces}\label{app:rkhsrev}
For convenience we provide a minimal review of RKHS see~\cite{\RKHSaron,Cristianini2000} for more details.

A real RKHS $\hK$ is induced by a kernel $K:\cX\times\cX\into\Re$.  Where $K$ is a symmetric and positive definite function.  A function $K$ is (strictly) positive definite iff the matrix $(K(x',x''))_{x',x'' \in X}$ is (strictly) positive definite for every finite cardinality $X\subseteq \cX$.  In this paper we are only concerned with strictly positive definite kernels.
 The pre-Hilbert space induced
by kernel $K$ is the set
$H_K: = \mbox{span}(\{K(x,\cdot)\}_{\forall x\in \cX})$ with the
inner product of $f = \sum_{i=1}^m \alpha_i K(x_i,\cdot)$ and $g
= \sum_{j=1}^n \alpha'_j k(x'_j,\cdot)$ defined as $\dotp{f,g}_K := \sum_{i=1}^m
\sum_{j=1}^n \alpha_i \alpha'_j K(x_i,x'_j)$.  The completion of $H_K$ is denoted $\hK$.  Finally the fact that $K$ is positive definite implies the {\em reproducing property:} if $f\in\hK$ and $x\in\cX$ then $f(x) = \dotp{f(\cdot),K(x,\cdot)}_K$.

\subsection{Proof of the Online Gradient Descent Regret Bound}\label{app:perc}
In this section, we will prove expected regret bounds for Online Gradient Descent for both the switching and non-switching case. The proofs are adapted from the material in~\cite{Shalev-Shwartz2011, herbster2001tracking, Zinkevich} (see~\cite{CLW96} for the seminal work on worst case bounds for online gradient descent with the square loss).
Recall that we wish to proof the following for the non-switching case:
\begin{equation}\label{eq:ogd1}
\sum_{t=1}^T \E [\lzo(y_t,\yht)] - \lzo(y_t,h(x_t)) \in \cO\left(\sqrt{\norm{h}^2_K X_K^2 T}\right) \quad (\forall h\in\hKbin)
\end{equation}
where $X_K^2:= \max_{t\in [T]} K(x_t,x_t)$.
For the switching case, we wish to prove
\begin{equation}\label{eq:ogd_switch}
\sum_{t=1}^T \E [\lzo(y_t,\yht)] - \lzo(y_t,h(x_t)) \in \cO\left(\sqrt{\swi \max_t \norm{h_t}^2_K X_K^2 T}\right).
\end{equation}
For simplicity, we prove for an arbitrary inner product $\langle\cdot,\cdot\rangle$ space with induced norm $\norm{\cdot}$. 
The RKHS setting reduces to this setting by identifying $\bxx := K(x,\cdot)$, $\bu := h$, and $\dotp{\bu,\bxx} := h(x)$.
\begin{algorithm}\caption{Randomized Constrained Online Gradient Descent Algorithm}
\begin{algorithmic}
\renewcommand{\algorithmicrequire}{\textbf{Parameters:}} 
\REQUIRE Learning rate $\lr$, radius $\msize$  \vspace{.1truecm}
\renewcommand{\algorithmicrequire}{\textbf{Initialization:}} 
\REQUIRE Initialize $\bw_1 = \bm{0}$  \vspace{.1truecm}
\renewcommand{\algorithmicrequire}{\textbf{For}}
\REQUIRE $t =1,\dots,T$  \vspace{.1truecm}
\STATE $\bullet$ Receive vector $\bxx_t$.
\STATE $\bullet$ Predict 
\begin{equation*} \Yrv \sim \mbox{\sc Uniform}(-1,1) \,;\  
\ybt \!\leftarrow\!\langle \bw_{t}, \bxx_t \rangle \,;\  \yht\! \leftarrow\! \sign(\ybt - \Yrv)\,.\vspace{-.2in}
\end{equation*}
\STATE $\bullet$ Receive label $y_{t} \in \{-1,1\}$\,.\vspace{.1truecm}
\STATE $\bullet$ If {$\ybt y_t \leq 1$} then
 \[\bw_{t}^m \leftarrow  \bw_{t} + \lr y_t \bxx_t\]
\[\bw_{t+1} \leftarrow  P_\msize(\bw_{t}^m) \]
\STATE $\bullet$ Else $\bw_{t}^m \leftarrow  \bw_{t}$\,; \,
$\bw_{t+1} \leftarrow \bw_{t} $

\end{algorithmic}
\label{alg:ogd}
\end{algorithm}

In the following, we define the hinge loss $\lhi(y',y'')= [1-y'y'']_+$ for $y',y''\in \Re$. We define $\bz_t :=  -y_t \bxx_t  [1-y_t \langle \bw_t, \bxx_t \rangle \geq 0] \in \nabla_\bw \lhi(y_t, \langle \bw, \bxx_t \rangle ) $, where $\bw_t$, $\bxx_t$ and $y_t$ are as defined in Algorithm~\ref{alg:ogd}. We denote  $P_\msize(\bw)$ to be the projection into the closed origin-centered ball with radius $\msize$, so that \[ P_\msize (\bw) = \begin{cases}
\bw & \text{ if } \norm{\bw} \leq \msize \\
\msize \frac{\bw}{\norm{\bw}} & \text{ otherwise }.
 \end{cases}
\] We also present a lemma, used as a starting point for both the switching and non-switching proofs.

\begin{lemma} \label{lem:found_ogdp}
For Algorithm~\ref{alg:ogd} and any $\bu$ lying in the convex set $\{\bw:\norm{\bw} \leq \msize\}$,
\[\langle \bw_t - \bu, \bz_t \rangle \leq \frac{1}{2\lr} \left( \norm{\bw_t- \bu}^2 - \norm{\bw_{t+1}- \bu}^2 + \lr^2 \norm{\bz_t}^2 \right)\]
\end{lemma}
\begin{proof}
Using the update rule of the algorithm, we have
\begin{align*} 
\norm{\bw_t^m - \bu}^2 &=  \norm{\bw_{t} -\lr \bz_t - \bu}^2 \\
&= \norm{\bw_t - \bu}^2 - 2\lr \langle \bw_t - \bu, \bz_t \rangle + \lr^2 \norm{\bz_t}^2
\end{align*}
Next note that \[ \norm{\bw_{t+1} - \bu}^2\le \norm{\bw_{t+1} - \bw_t^m}^2+\norm{\bw_{t+1} - \bu}^2 \leq \norm{\bw_t^m - \bu}^2\] 
where the rightmost inequality is the Pythogorean inequality for projection onto a convex set where  $\bw_{t+1}$  is the projection of  $\bw_t^m$ on to the convex set $\{\bw:\norm{\bw} \leq \msize\}$ which  contains $\bu$. Thus,
\[
\norm {\bw_{t+1} - \bu }^2 \leq  \norm{\bw_t - \bu}^2 - 2\lr \langle \bw_t - \bu, \bz_t \rangle + \lr^2 \norm{\bz_t}^2.\]
Rearranging then results in the lemma.
\end{proof}

We will use the following lemma to upper bound the zero-one loss  of our randomized prediction by the hinge loss.
\begin{lemma}\label{lem:shinge_bound}
For $y \in \{-1,1\}$, $\yb \in \Re$, $Y\sim \mbox{\sc Uniform}(-1,1)$,  and $\yh := \sign(\yb - Y)$,
\[2\mathbb{E}[\lzo(y,\yh)] \leq \lhi(y,\yb). \]
\end{lemma}
\begin{proof}
We have
\[
p(\yh=1) = 
\begin{cases}
0 & \text{ if } \yb \leq -1 \\
\frac{1}{2}(1+{\yb}) & \text{ if } -1 < \yb \leq 1\\
1 & \text{ if } \ybt> 1
\end{cases}
\]
and
\[
p(\yh=-1) = 
\begin{cases}
1 & \text{ if } \yb \leq -1 \\
\frac{1}{2}(1 - \yb) & \text{ if } -1 < \yb \leq 1\\
0 & \text{ if } \yb > 1.
\end{cases}
\]
The possible cases are as follows.
\begin{enumerate}
\item If $|\yb| < 1$ then $2 \mathbb{E}[\lzo(y,\yh)] = \lhi(y,\yb)$. This is since if $y=1$ then $\mathbb{E}[\lzo(y,\yh)] = \frac{1}{2}(1 - {\yb})$ and $\lhi(y,\yb) = 1 - \yb$. Similarly if $y=-1$ then $\mathbb{E}[\lzo(y,\yh)] = \frac{1}{2}(1+\yb)$ and $\lhi(y,\yb) =(1 +\yb	) $.
\item If $|\yb| \geq 1$ and
$\mathbb{E}[\lzo(y,\yh)]=0$, then $\lhi( y,\yb)=  [1 - |\yb|]_+ = 0$. 
\item If $|\yb| \geq 1$ and $\mathbb{E}[ \lzo(y,\yh)]=1$ then,
 $\lhi( y,\yb) = [1 + |\yb|]_+ \geq 2 = 2 \mathbb{E}[\lzo(y,\yh)]. $
\end{enumerate}
\end{proof}

\subsubsection{Non-switching bound}

\begin{lemma}\label{lem:hl_bound-ns}
For Algorithm~\ref{alg:ogd}, 
given $X= \max_t \norm{ \bxx_t}$, $\norm{\bu} \leq U$ and $\lr = \frac{U }{X \sqrt{T}}$
we have that
\begin{equation}\label{eq:bdhi}
\sum_{t=1}^T \lhi(y_t,\ybt) - \lhi(y_t,\langle \bu, \bxx_t \rangle) \leq \sqrt{U^2 X^2 T}\,,
\end{equation}
for any vector $\bu$.
\end{lemma}
\begin{proof}
Using the convexity of the hinge loss (with respect to its second argument), we have
\[ 
 \lhi(y_t,\ybt) - \lhi(y_t,\langle \bu, \bxx_t \rangle) \leq \langle \bw_t - \bu, \bz_t \rangle.
\]
We may therefore proceed by bounding $ \sum_{t=1}^T \langle \bw_t - \bu, \bz_t \rangle$.
Starting with Lemma~\ref{lem:found_ogdp} and summing over $t$, we have 
\begin{align}
\sum_{t=1}^T \langle \bw_t - \bu, \bz_t \rangle &\leq \frac{1}{2\lr} \left( \norm{\bw_1- \bu}^2 - \norm{\bw_{T+1}- \bu}^2 + \lr^2 \sum_{t=1}^T \norm{\bz_t}^2 \right) \notag \\
& \leq \frac{1}{2\lr} \left( \norm{\bu}^2 + \lr^2 \sum_{t=1}^T \norm{\bz_t}^2 \right) \label{eq:w0-def}\\
 &=  \frac{1}{2\lr} \norm{\bu}^2 + \frac{\lr}{2 } \sum_{t=1}^T \norm{\bxx_t}^2 [1-y_t \langle \bw_t, \bxx_t \rangle \geq 0] \notag\\
 &\leq  \frac{1}{2\lr} \norm{\bu}^2 +  \frac{\lr}{2 }\sum_{t=1}^T \norm{\bxx_t}^2  \notag\\
 &\leq \frac{1}{2\lr} U^2 + \frac{\lr}{2 } X^2 T \notag\\
 &=   \sqrt{U^2X^2T} \notag
\end{align}
where Equation~\eqref{eq:w0-def} results from the fact that $\bw_1 = 0$.
\end{proof}

 \begin{theorem}\label{thm:ogd_static_fin}
 For Algorithm~\ref{alg:ogd}, given $X= \max_t \norm{ \bxx_t}$, $  \norm{\bu} \leq U$, $\lr = \frac{U}{X \sqrt{T}}$,  \[ \sum_{t=1}^T \Exp[\lzo(y_t,\yht)] -  \lzo (y_t ,\langle \bu, \bxx_t \rangle )\leq \frac{1}{2}  \sqrt{U^2 X^2 T}\,,\]
 \end{theorem}
 for any vector $\bu$ such that $|\langle \bu, \bxx_t \rangle| = 1$ for $t= 1,\ldots, T$.
\begin{proof}
Applying the lower bound on the hinge loss from Lemma~\ref{lem:shinge_bound} to~\eqref{eq:bdhi} gives
\begin{equation*}
2  \sum_{t=1}^T \Exp[\lzo(y_t,\yht)]  \leq \sum_{t=1}^T \lhi(y_t,\langle \bu, \bxx_t \rangle) + \sqrt{U^2 X^2 T}\,,
\end{equation*}
observe that $\lhi(y_t,\langle \bu, \bxx_t \rangle) = 2 \lzo (y_t ,\langle \bu, \bxx_t \rangle )$ since we have 
the  condition $|\langle \bu, \bxx_t \rangle| = 1$ for $t= 1,\ldots, T$ dividing both sides by 2 proves the theorem.
\end{proof}
The bound for the non-switching case in~\eqref{eq:ogd1} then follows by setting $U= ||\bu||$.
\subsubsection{Switching bound}

\begin{lemma}\label{lem:hl_ogd_switch}
For Algorithm~\ref{alg:ogd}, given $X= \max_t \norm{ \bxx_t}$, $\{\bu_1, \ldots \bu_T \} \subset \{\bu:\norm{\bu} \leq \msize\}$,  $\lr = \frac{U}{X \sqrt{T}}$ 
and \, $  \sqrt{   \norm{\bu_{T}}^2 + 2 \msize\sum_{t=1}^{T-1}  \norm{\bu_{t+1} - \bu_t} } \leq U$, we have that
\[ \sum_{t=1}^T \lhi(y_t,\ybt) - \lhi(y_t,\langle \bu_t, \bxx_t \rangle) \leq  \sqrt{U^2 X^2T}.\]
\end{lemma}
\begin{proof}
Using the convexity of the hinge loss (with respect to its second argument), we have
\[ 
 \lhi(y_t,\ybt) - \lhi(y_t,\langle \bu_t, \bxx_t \rangle) \leq \langle \bw_t - \bu_t, \bz_t \rangle.
\]
We may therefore proceed by bounding $ \sum_{t=1}^T \langle \bw_t - \bu_t, \bz_t \rangle$.
Starting with Lemma~\ref{lem:found_ogdp} and summing over $t$, we have 
\begin{align}
\sum_{t=1}^T \langle \bw_t - \bu_t, \bz_t \rangle & \leq \frac{1}{2\lr}  \sum_{t=1}^T  \left( \norm{\bw_t- \bu_{t}}^2 - \norm{\bw_{t+1}- \bu_t}^2 + \lr^2  \norm{\bz_t}^2 \right) \label{eq:start_ogdsw}
\end{align}
To transform the right hand side of the above equation into a telescoping sum, we add and subtract the term $A_t =\norm{\bw_{t+1}- \bu_{t}}^2 - \norm{\bw_{t+1}- \bu_{t+1}}^2 $, giving
\begin{align}
\sum_{t=1}^T \norm{\bw_t- \bu_{t}}^2 - \norm{\bw_{t+1}- \bu_t}^2
&= \sum_{t=1}^T \norm{\bw_t- \bu_{t}}^2 - \norm{\bw_{t+1}- \bu_{t+1}}^2- A_t \notag\\
& =  \norm{ \bu_{1}}^2 - \norm{\bw_{T+1}- \bu_{T+1}}^2 -  \sum_{t=1}^T (\norm{\bw_{t+1}- \bu_{t}}^2 - \norm{\bw_{t+1}- \bu_{t+1}}^2 ) \notag \\
&= \norm{ \bu_{1}}^2 - \norm{\bw_{T+1}- \bu_{T}}^2 -  \sum_{t=1}^{T-1} (\norm{\bw_{t+1}- \bu_{t}}^2 - \norm{\bw_{t+1}- \bu_{t+1}}^2 )\label{eq:eval_T} \\
&\leq \norm{ \bu_{1}}^2  -  \sum_{t=1}^{T-1} (\norm{\bw_{t+1}- \bu_{t}}^2 - \norm{\bw_{t+1}- \bu_{t+1}}^2 ) \label{eq:before_subs},
\end{align}
where Equation~\eqref{eq:eval_T} comes from evaluating $t=T$ in the summation.

Computing the sum, we obtain
\begin{align}
 \sum_{t=1}^{T-1} \norm{\bw_{t+1}- \bu_{t}}^2 - \norm{\bw_{t+1}- \bu_{t+1}}^2  & =  \sum_{t=1}^{T-1} \norm{\bu_t}^2  - \norm{\bu_{t+1}}^2 - 2 \langle \bw_{t+1}, (\bu_{t} - \bu_{t+1}) \rangle \notag  \\ 
& \geq \sum_{t=1}^{T-1} \norm{\bu_t}^2 - \norm{\bu_{t+1}}^2 -   2 \norm{\bw_{t+1}} \norm{\bu_{t} - \bu_{t+1}} \notag  \\
&\geq \norm{\bu_1}^2 -\norm{\bu_{T}}^2 -2 \msize \sum_{t=1}^{T-1} \norm{\bu_{t} - \bu_{t+1}} \label{eq:wtpbound}
\end{align}
where Equation~\eqref{eq:wtpbound} comes from $\norm{\bw_{t+1}} \leq \msize$, a consequence of the projection step.
Substituting this back into Equations~\eqref{eq:start_ogdsw} and~\eqref{eq:before_subs}, we then obtain
\begin{align*}
\sum_{t=1}^T \langle \bw_t - \bu_t, \bz_t \rangle & \leq \frac{1}{2\lr}   \left( \norm{\bu_{T}}^2+  2\msize \sum_{t=1}^{T-1} \norm{\bu_{t} - \bu_{t+1}} + \sum_{t=1}^{T}  \lr^2  \norm{\bz_t}^2 \right) \\
  &\leq \frac{1}{2\lr} U^2 + \frac{\lr}{2} X^2 T.\\
 &=   \sqrt{U^2X^2T},
\end{align*}
where the second inequality comes from the definitions of $\bz_t$, $U$ and  $X$, and the equality comes from the definition of $\lr$.
\end{proof}

 \begin{theorem}\label{thm:ogd_switch_fin}
For Algorithm~\ref{alg:ogd}, given $X= \max_t \norm{ \bxx_t}$, $\{\bu_1, \ldots \bu_T \} \subset \{\bu:\norm{\bu} \leq \msize\}$, and $  \sqrt{ \norm{\bu_{T}}^2+  2 \msize \sum_{t=1}^{T-1}  \norm{\bu_{t+1} - \bu_t}}  \leq U$, and $\lr = \frac{U}{X\sqrt{T}}$ 
we have that
 \[ \sum_{t=1}^T \Exp[\lzo(y_t,\yht)] -  \lzo (y_t ,\langle \bu_t, \bxx_t \rangle )\leq  \frac{1}{2} \sqrt{U^2 X^2T}\,,\]
 \end{theorem}
 for any sequence of vectors $\bu_1, \ldots \bu_T$ such that $|\langle \bu_t, \bxx_t \rangle| = 1$ for $t= 1,\ldots, T$.
 \begin{proof}
The proof follows the structure of the proof of Theorem~\ref{thm:ogd_static_fin} except that Lemma~\ref{lem:hl_ogd_switch}
is the base inequality.
\end{proof}

The bound for the switching case then follows from Theorem~\ref{thm:ogd_switch_fin} by setting $\msize = \max_t ||\bu_t||$, and $ U=\sqrt{(4 k+1) \max_t ||\bu_t||^2} $, noting that
\begin{align*}
\norm{\bu_{T}}^2+  2 \msize \sum_{t=1}^{T-1}  \norm{\bu_{t+1} - \bu_t}  &\leq \norm{\bu_{T}}^2+2 \max_t ||\bu_t|| \,(2 k \, \max_t|| \bu_t||)\\
&= \norm{\bu_{T}}^2+4 \swi \max_t ||\bu_t||^2 \\
&\leq (4 \swi + 1)\max_t ||\bu_t||^2 \\
&= U^2.
\end{align*}
This gives us a regret bound of 
 \[ \sum_{t=1}^T \Exp[\lzo(y_t,\yht)] -  \lzo (y_t ,\langle \bu_t, \bxx_t \rangle )\leq  \frac{1}{2} \sqrt{(4 k+1) \max_t ||\bu_t||^2 X^2T}\,,\]
 as desired.

\end{document}